\documentclass[letterpaper,11pt,onecolumn,final]{report}              
\usepackage[left=1.5in,right=1in,top=1in,bottom=1in]{geometry}    
\usepackage{setup/ConcordiaU}                         
\usepackage{setspace}                             
\usepackage[natbibapa]{apacite}                  
\usepackage{amsmath, amssymb, mathtools, amsthm}                          
\usepackage{amsfonts}
\usepackage[dvips]{graphicx}                    
\usepackage{stmaryrd}
\usepackage{tabularx}
\graphicspath{{figures/}}                           
\usepackage{multirow}    
\usepackage{color, colortbl}
\usepackage{multicol}
\usepackage{booktabs}
\usepackage{wrapfig}
\usepackage{indentfirst}  
\usepackage{caption} 
\usepackage{algorithm, algorithmic}                   
\usepackage[colorlinks,linkcolor=blue,anchorcolor=blue,citecolor=blue,urlcolor=blue]{hyperref}    
\usepackage{enumitem}
\usepackage{subcaption}
\usepackage{titlesec}
\usepackage{tcolorbox}
\usepackage[titletoc,title]{appendix}
\newtheorem{proposition}{Proposition}[section]
\newtheorem{definition}{Definition}[section]

\definecolor{PaleGreen}{rgb}{0.76, 1.0, 0.76}
\definecolor{Orangelight}{rgb}{1,0.85,0.6}
\definecolor{LightCyan}{rgb}{0.90,1,1}
\definecolor{LightGreen}{rgb}{0.88, 0.95, 0.88}

\tcbuselibrary{skins, breakable, listingsutf8}
\definecolor{lightgray}{RGB}{250, 250, 250} 
\definecolor{lightblue}{RGB}{15, 71, 127}  \definecolor{darkorange}{RGB}{37, 105, 16}  
\newtcolorbox[auto counter, number within=section]{blueequationbox}[2][]{colback=lightgray, colframe=lightblue,
    fonttitle=\bfseries, coltitle=white, colbacktitle=lightblue,
    title= #2, boxrule=0.8pt, arc=3pt, boxsep=5pt, left=2pt, right=2pt, top=2pt, bottom=2pt}

\newtcolorbox[auto counter, number within=section]{orangeequationbox}[2][]{colback=lightgray, colframe=darkorange,
    fonttitle=\bfseries, coltitle=white, colbacktitle=darkorange,
    title= #2, boxrule=0.8pt, arc=3pt, boxsep=5pt, left=2pt, right=2pt, top=2pt, bottom=2pt}
    
\begin{document}

\author{Damien Martins Gomes}
\title{Towards Practical Second-Order Optimizers in Deep Learning: Insights from Fisher Information Analysis}

\mastersDegree{Master of Applied Science}

\titleOfPhDAuthor{Mr.}

\program{Computer Science}

\dept{Computer Science and Software Engineering}

\chairOfDept{Joey Paquet}

\deanOfENCS{Mourad Debbabi}

\chairOfCommittee{Mirco Ravanelli}
\examinerFirstOfCommittee{Mirco Ravanelli}
\examinerExternalOfCommittee{Elvis Dohmatob} 

\principaladvisor{Mahdi S. Hosseini}



\begin{abstract}
First-order optimization methods remain the standard for training deep neural networks (DNNs). Optimizers like Adam incorporate limited curvature information by preconditioning the stochastic gradient with a diagonal matrix. Despite the widespread adoption of first-order methods, second-order optimization algorithms often exhibit superior convergence compared to methods like Adam and SGD. However, their practicality in training DNNs is still limited by a significantly higher per-iteration computational cost compared to first-order methods. In this thesis, we present \emph{AdaFisher}, a novel adaptive second-order optimizer that leverages a \emph{diagonal block-Kronecker approximation} of the Fisher information matrix to adaptively precondition gradients. AdaFisher aims to bridge the gap between the improved \emph{convergence} and \emph{generalization} of second-order methods and the computational efficiency needed for training DNNs. Despite the traditionally slower speed of second-order optimizers, AdaFisher is effective for tasks such as image classification and language modeling, exhibiting remarkable stability and robustness during hyperparameter tuning. We demonstrate that AdaFisher \textbf{outperforms state-of-the-art optimizers} in both accuracy and convergence speed. The Code is available from \href{https://github.com/AtlasAnalyticsLab/AdaFisher}{https://github.com/AtlasAnalyticsLab/AdaFisher}.
\end{abstract}
\begin{acknowledgments}
\vspace{-5mm}
I would like to express my deepest gratitude to Dr. Mahdi S. Hosseini, my supervisor, for his invaluable guidance, extensive support, and the countless hours he dedicated to helping me prepare and submit the AdaFisher paper to the ICLR 2025 conference. His mentorship has been instrumental in this journey.

A special and heartfelt thanks to PhD. Yanley (Kaly) Zhang, the second author of the AdaFisher paper and co-author for the BDN paper, for his relentless efforts, time, and dedication. His contributions to the experiments and countless meetings, even during late nights, have been crucial to the success of this work.

I am also immensely grateful to Atlas Analytics Lab for their support throughout this journey. I benefited greatly from the advice and assistance of my colleagues Cassandre, Ali, Hailey, Amirhossein, Thomas, Denisha, Vasudev, and Sina. Your support has been invaluable.

Merci aux Fonds de Recherche du Québec - Nature et Technologies, pour leur soutien et leur volonté de soutenir la recherche académique québécoise.

I extend my deepest gratitude to my family and friends. To Alexis, my roommate, for the engaging discussions about research, fieldwork, and One Piece. Your intellectual companionship has been invaluable. Thank you to my mother and sibling for their unwavering encouragement and support across the Atlantic.

Merci à mon école en France, IPSA Toulouse, de m’avoir permis de compléter mon parcours universitaire par une expérience de recherche inoubliable. Special thanks to Dr. Lorenzo Ortega for accompanying me as my supervisor in France.

And finally, to all the people who have accompanied me on this research journey, both from France and Canada, merci infiniment pour votre soutien.
\end{acknowledgments}

\begin{contributions}
Damien Martins Gomes originated and developed the AdaFisher optimizer, addressing the critical challenge of achieving rapid convergence and enhanced generalization while preserving the computational efficiency typical of first-order methods. Drawing on an extensive literature review and extensive discussions with his supervisor, Dr. Mahdi S. Hosseini, Damien designed the theoretical framework for AdaFisher and validated its performance through large-scale experiments. His work was pivotal not only in the conceptual development and empirical evaluation of AdaFisher but also in the composition of the conference paper and the rigorous rebuttal process for ICLR 2025 submission.

Dr. Yanlei Zhang, a Post-Doctoral researcher at Mila and Université de Montréal, joined the project to further strengthen the convergence analysis and support the experimental investigations of AdaFisher. His contributions were crucial in refining the theoretical analysis and ensuring robust experimental outcomes. Dr. Zhang also played an active role in drafting the manuscript and participated in numerous collaborative meetings, thereby enhancing the overall quality of the work.

All authors have reviewed and approved the final manuscript.
\end{contributions}
\pagenumbering{arabic}
\begin{glossaryy}
\begin{multicols}{2}
\noindent
\textbf{AI}: Artificial Intelligence \newline
\textbf{BN}: BatchNorm Layer \newline
\textbf{CNNs}: Convolutional Neural Networks \newline
\textbf{CV}: Computer Vision \newline
\textbf{DL}: Deep Learning \newline
\textbf{DNNs}: Deep Neural Networks \newline
\textbf{EFIM}: Empirical Fisher Information Matrix \newline
\textbf{EMA}: Exponential Moving Average \newline
\textbf{FFT}: Fast Fourier Transform \newline
\textbf{FIM}: Fisher Information Matrix \newline
\textbf{HP}: Hyper-Parameter \newline
\textbf{KF}: Kronecker Factor \newline
\textbf{LLM}: Large Language Model \newline
\textbf{MLP}: Multi-Layer Perceptron \newline
\textbf{NGD}: Natural Gradient Descent \newline
\textbf{NLP}: Natural Language Processing \newline
\textbf{NN}: Neural Network \newline
\textbf{PCA}: Principal Component Analysis \newline
\textbf{PPL}: Perplexity \newline
\textbf{PTB}: Penn TreeBank \newline
\textbf{SGD}: Stochastic Gradient Descent \newline
\textbf{SNRs}: Signal-to-Noise Ratios \newline
\textbf{SOTA}: State-Of-The-Art \newline
\textbf{ViTs}: Vision Transformers \newline
\textbf{VRAM}: Video RAM \newline
\textbf{WCT}: Wall-Clock-Time \newline
\end{multicols}
\end{glossaryy}

\chapter{Introduction}
\label{chap:introduction}

Deep Neural Networks (DNNs) have revolutionized a wide array of applications, from Computer Vision (CV) and Natural Language Processing (NLP) to reinforcement learning and beyond. Despite their success, training these models remains a challenging task due to the complexity of the loss landscapes they must navigate. In particular, DNN optimization often struggles with the dual challenges of achieving fast convergence and robust generalization across diverse architectures and complex data distributions.

DNNs have demonstrated remarkable success across a wide range of applications, yet their training remains a computationally intensive, time-consuming process that often requires massive amounts of data \citep{brown2020language}. This raises a fundamental motivating question: \textbf{how can we train Neural Networks (NN) more effectively}? Efforts to address this challenge have emerged from many directions, including improved optimization algorithms \citep{martens2010deep, martens2015optimizing}, specialized hardware designs \citep{misra2010artificial}, more data-efficient NN architectures \citep{snell2017prototypical}, and dedicated Deep Learning (DL) compilers \citep{chen2018tvm, rotem2018glow}. However, each of these approaches underscores the need for a deeper understanding of the factors that govern NN performance. Modifications to network architectures or optimization strategies can have profound impacts, as can more subtle changes such as reduced precision training \citep{gupta2015deep}. Yet, finding a cohesive set of tools to comprehend and harness these effects remains a long-standing challenge in DL optimization.

This thesis promotes a holistic approach to DL optimization, exploring both practical implementations and theoretical insights to illuminate the structure of loss landscapes and the dynamics of training. Central to our perspective is the observation that DL models exhibit a surprisingly rich structure in their loss landscapes, a property that not only facilitates acceleration in optimization but also helps explain the overall success of these models. By adopting loss landscape geometry as a framework, we reveal how various components—from curvature-aware updates to adaptive optimization techniques—contribute to effective neural network training. Connecting these pieces provides a clearer perspective on how the interplay between network architecture, optimization dynamics, and loss landscape geometry can be harnessed to design optimizers that balance convergence speed, generalization, and computational efficiency \citep{nocedal1999numerical}.

In the following chapters, we build on these ideas to introduce \emph{AdaFisher}, an adaptive second-order optimizer that leverages a refined \textbf{diagonal block-Kronecker approximation} of the Fisher Information Matrix. Through theoretical development, extensive empirical validation, and comprehensive ablation studies, our work aims to advance the state of DL optimization, providing both practical solutions and new insights for the research community.

\section{Problem Statement}
At the core of DNN training is the minimization of a highly non-convex loss function \(\mathcal{L}(\boldsymbol{\theta})\) by updating model parameters \(\theta\) according to an iterative scheme
\[
\boldsymbol{\theta}^{(t+1)} = \boldsymbol{\theta}^{(t)} - \alpha\, (\mathcal{G}^{(t)})^{-1} \nabla \mathcal{L}(\boldsymbol{\theta}^{(t)}),
\]
where \(\alpha\) is the learning rate and \(\mathcal{G}^{(t)}\) encapsulates curvature information. For first-order methods like Stochastic Gradient Descent (SGD), \(\mathcal{G}^{(t)}\) is simply the identity matrix, which makes these methods computationally efficient but often blind to the underlying geometry of the loss surface. In contrast, second-order methods employ richer curvature information—via the Hessian or the Fisher Information Matrix (FIM)—to rescale and orient gradients in a manner that can accelerate convergence and guide the optimizer toward flatter, more generalizable minima. However, these benefits come at a steep computational cost, as calculating and inverting full curvature matrices is often intractable for large-scale networks. This thesis addresses the need for an optimizer that achieves a balance between rapid convergence, strong generalization, and computational efficiency.

\section{Challenges Associated with the Problem}
The existing literature reveals a fundamental trade-off in DNN optimization. On one hand, methods such as Adam \citep{kingma2017adam} and its variants (e.g., AdamP \citep{heo2021adamp}, AdaInject \citep{dubey2022adainject}, and AdaBelief \citep{zhuang2020adabelief}) rely on diagonal approximations of the FIM to remain computationally efficient. While effective in many settings, these approximations can lead to suboptimal convergence and poorer generalization because they fail to capture the full curvature of the loss landscape. On the other hand, advanced second-order approaches like AdaHessian \citep{yao2021adahessian}, Shampoo \citep{gupta2018shampoo}, and K-FAC \citep{grosse2016kroneckerfactored} attempt to incorporate richer curvature information to improve optimization performance, but they are often hampered by high computational overhead, extensive hyper-parameter tuning requirements, and scalability issues on commodity hardware \citep{martens2020new}. This imbalance between the richness of curvature information and computational feasibility represents a critical challenge for the field, as an effective solution would greatly enhance training dynamics and yield better-performing models across various domains. Hence, we list three research questions that we will try to talk throughout this thesis: How can the Fisher Information Matrix’s structure be leveraged to design a second-order optimizer that is computationally feasible for deep networks?; Does using a refined FIM-based preconditioning improve convergence speed and generalization compared to first-order methods?; What are the impacts of various design choices (e.g., using square-root of adapting optimization scheme, including Fisher computation for normalization layers) on the optimizer’s performance?

\section{Proposed Methodology}
To address this challenge, we introduce \emph{AdaFisher}, a novel adaptive second-order optimization method that leverages a refined diagonal block-Kronecker approximation of the FIM. Inspired by Natural Gradient Descent (NGD) \citep{Amari2000MethodsOI}—which uses the FIM as a \emph{Riemannian metric} to precondition gradients—AdaFisher replaces the second moment in Adaptive framework with our enhanced FIM approximation. This novel approach effectively combines the computational efficiency of first-order methods with the curvature-awareness of second-order techniques. By capturing the essential curvature information through a diagonal block-Kronecker factorization, AdaFisher accelerates convergence and guides the optimization process toward flatter local minima, thereby improving generalization. As illustrated in Figure~\ref{fig:heatloss}, AdaFisher not only converges more rapidly but also reaches a superior local minimum by effectively navigating through saddle points compared to its counterparts. Further details regarding the visualization can be found in Appendix~\ref{chap:Appendixvisualization}. 
\begin{figure}[!h]
    \centering
    \includegraphics[width=0.6\linewidth]{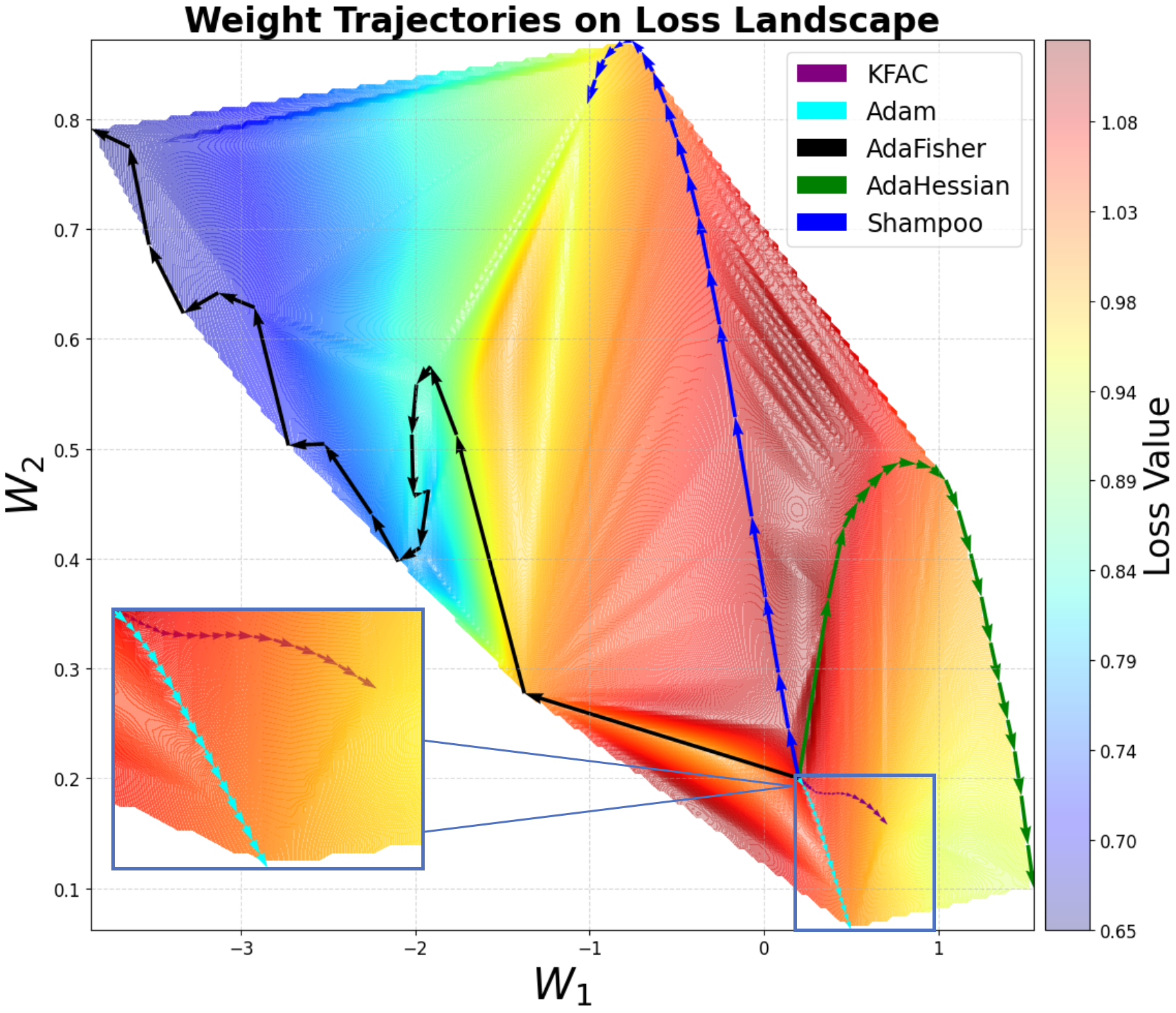}
    \caption{Visualizing optimization trajectories for various optimizers overlaid a loss landscape.}
    \label{fig:heatloss}
\end{figure}
\section{Contributions}
The contributions of this thesis are multifaceted and build upon extensive theoretical and empirical insights in DNN optimization. First, we provide a comprehensive literature review that situates our work within the broader context of first-order, second-order, and emerging optimization methods, highlighting the limitations of current approaches. Second, we empirically demonstrate that the energy of the Kronecker Factors (KF) is predominantly concentrated along the diagonal, offering fresh insights into the structure of the FIM in deep learning. Third, we introduce a novel diagonal block-Kronecker approximation of the FIM, which is applicable across various network layers, including normalization layers, thereby enhancing model adaptability while maintaining computational efficiency. Fourth, we establish the robustness and stability of AdaFisher through extensive experiments across diverse training settings, proving its superior convergence and generalization performance. Fifth, we showcase AdaFisher's empirical performance against state-of-the-art optimizers in both image classification and language modeling tasks. Lastly, we develop innovative visualization techniques that map optimizer trajectories in the loss landscape and introduce an explainable FIM measure, facilitating comparative analysis of optimizer behavior.

\section{Outline}
This thesis is organized to provide a cohesive narrative that bridges theoretical developments with practical implementations. Chapter~\ref{chap:literaturereview} offers a detailed literature review, surveying existing optimization techniques and identifying the challenges that motivate our approach. Chapter~\ref{chap:efficientfisherapprox} delves into a new approximation of the FIM which enables its computation for large scale DNNs while preserving the main curvature information. Then in Chapter~\ref{chap:adafisher} we introduce AdaFisher optimizer, explaining its theoretical foundations, algorithmic structure, and distributed implementation. In Chapter~\ref{chap:experiments}, we present extensive experiments in both computer vision and natural language processing, demonstrating the effectiveness of AdaFisher relative to other state-of-the-art optimizers. Chapter~\ref{chap:ablativestudies} provides an in-depth ablation and stability analysis, dissecting the contributions of individual components within AdaFisher. Finally, Chapter~\ref{chap:conclusion} concludes the thesis with a discussion of the implications of our findings and potential directions for future research.

By addressing the inherent trade-offs in DNN optimization, this thesis contributes a new, scalable approach that combines the efficiency of first-order methods with the precision of second-order curvature information, paving the way for more robust and generalizable deep learning models.
\chapter{Literature Review}
\label{chap:literaturereview}
Deep Learning has transformed numerous fields by achieving breakthrough results in computer vision \citep{he2016deep, liu2021swin, krizhevsky2012imagenet}, natural language processing \citep{NIPS2017_3f5ee243, devlin2018bert, brown2020language}, and beyond. The extraordinary performance of deep networks is not solely a consequence of their expressive architectures but also of the intricate dynamics that govern their training. Understanding these dynamics is crucial for several reasons. First, it sheds light on the non-convex optimization challenges inherent to deep models. Second, it helps to reveal why and how over-parameterized networks generalize well despite their capacity to fit highly complex datasets. Third, it informs the design of optimization algorithms that balance rapid convergence with robust generalization. While empirical successes abound, theoretical understanding lags behind practice, with critical open questions persisting about: 1) how optimization trajectories navigate high-dimensional non-convex landscapes \citep{zaheer2018adaptive}, 2) why certain algorithms generalize better despite equivalent training performance \citep{allen2019learning}, and 3) how computational constraints shape method design \citep{hu2023optimization}. This chapter synthesizes recent advances across three interconnected themes: 

First, we will introduce the preliminaries notations and concepts in  Section~\ref{sec:preliminaries}, then Section~\ref{sec:dynamics} analyzes DL dynamics through the lens of phase transitions, loss geometry, and implicit regularization. Section~\ref{sec:challenges} then delineates fundamental obstacles in NN optimization, from saddle point proliferation to memory-time tradeoffs. Section~\ref{sec:methods} systematically evaluates optimization paradigms, contrasting zeroth-, first-, and second-order methods through their theoretical guarantees and empirical behavior. Finally,  Section~\ref{sec:emerging} discusses emerging directions that contextualize the design space for modern optimizers.
\section{Preliminaries}\label{sec:preliminaries}
\subsection{Standard Notation}
We include an overview of our notation within this section in addition to covering some fundamental material that we use throughout the thesis, and additional notations will be defined as needed. We write $\mathbb{R}$ for the space of real numbers and $\mathbb{N}$ for the space of natural numbers. Scalars are
typically represented in lower-case, e.g. $x \in \mathbb{R}$. We present vectors in lower-case boldface font, e.g. $\mathbf{x} \in \mathbb{R}^d$. Matrices are written in upper-case font, e.g. $A \in \mathbb{R}^{d\times d}$. We use subscript to denote the entries of vectors and matrices. For example, $v_i$ denotes the $i^{\text{th}}$ entry of the vector $\mathbf{v}$, and $A_{ij}$ denotes the entry of $A$ at position $(i,j)$.
\subsection{Probabilities}
Throughout this thesis, we rely on fundamental tools from probability theory. While a careful, thorough review is out of scope, we present here some key ideas that we need later. To denote the probability of an event $A$ occurring, we write $\mathbb{P}[A]$. We frequently work with random variables, which are defined as a function from a sample space to a measurable space. We focus primarily on real-valued random variables, such that the measurable space is a subset of $\mathbb{R}$.
Our random variables can typically be described via a \emph{probability density function} (PDF), $p(x)$. That is, given a random variable $X$,
\begin{align}
    \mathbb{P}[X \in R] = \int_R p(x)dx, \notag
\end{align}
where $R \in \mathbb{R}$. Given two random variables $X$ and $Y$, we define the conditional distribution of $X$ given $Y$ by
the PDF,
\begin{align}
    p(x|y) = \frac{p(x,y)}{p(y)} \notag
\end{align}
where $p(x,y)$ denotes the joint PDF.
\paragraph{Expected value.} We define the expected value of a real-valued random variable with PDF $p(x)$ as,
\begin{align}
    \mathbb{E}[X] = \int_{\mathbb{R}} xp(x)dx \notag
\end{align}
This definition can be readily extended to vector-valued random variables by integrating element-wise. In the case where we have several samples of a random variable, $x_1,\dots,x_n$, we write the empirical mean as,
\begin{align}
    \bar{x} = \frac{1}{n} \sum_{i=1}^{n} x_i \notag
\end{align}
\paragraph{Covariance Matrices.} Given a vector-valued random variable $x \in \mathbb{R}^d$, its covariance is defined by,
\begin{align}
    \text{Cov}(\mathbf{x}) = \mathbb{E}\left[(\mathbf{x} - \mathbb{E}[\mathbf{x}])(\mathbf{x} - \mathbb{E}[\mathbf{x}])^\top\right]. \notag
\end{align}
Given, observations $\mathbf{x}_i \in \mathbb{R}^{d \times n}$, for $i= 1,\dots,n$ of the r.v $\mathbf{x}$, we define the empirical covariance matrix as,
\begin{align}
    \hat{\Sigma} = \frac{1}{n} \sum_{i=1}^{n} (\mathbf{x}_i - \bar{\mathbf{x}})^\top = \frac{1}{n} (X - \bar{X})(X - \bar{X})^\top, \notag
\end{align}
where the design matrix $X \in \mathbb{R}^{d\times n}$ contains $\mathbf{x}_i$ in its $i^{\text{th}}$ column, and $\bar{X} \in \mathbb{R}^{d\times n}$has $\bar{\mathbf{x}}$ in each column. From this, we can see that the empirical covariance matrix is symmetric and positive semidefinite--the latter meaning that all eigenvalues of $\hat{\Sigma}$ are non-negative.
\paragraph{Latent Variable Models.}
In Chapter~\ref{chap:efficientfisherapprox}, we present a new, efficient approximation of the Fisher Information Matrix—defined below—which naturally arises in the context of latent variable models that require an underlying statistical framework. These are probabilistic graphical models \citep{koller2009probabilistic} where latent variables\footnote{Often referred to as \emph{unobserved} or \emph{hidden} variables} capture features of observed variables. In the simplest setting, the latent variables, $\mathbf{z} \in \mathbb{R}^k$, and the observed variables $\mathbf{x}\in \mathbb{R}^d$ have their joint PDF defined in the following way,
\begin{align}
    p(x,z) = p_{\boldsymbol{\theta}}(\mathbf{x}|\mathbf{z})p(\mathbf{z}), \notag
\end{align}
where $p(\mathbf{z})$ represents a \emph{prior distribution} over the latent variables, and $p_{\boldsymbol{\theta}}(\mathbf{x} |\mathbf{z})$ represents the conditional distribution parameterized by some $\boldsymbol{\theta}$ to be learned from observed data. 
Given some i.i.d. observed data, $\{\mathbf{x}_i\}_{i=1}^n$, we seek to maximize the log marginal likelihood,
\begin{align}
    \sum_{i=1}^n \log p_{\boldsymbol{\theta}}(\mathbf{x}_i) = \sum_{i=1}^n \log \frac{p_{\boldsymbol{\theta}}(\mathbf{x} | \mathbf{z})p(\mathbf{z})}{p(\mathbf{z}|\mathbf{x}_i)}. \notag
\end{align}
Optimizing the parameters, $\boldsymbol{\theta}$, typically requires the computation of (or samples from) the posterior distribution. But this is not generally available in closed form. Several methods exist that allow us to perform inference in this model regardless.

To formalize the notion of the FIM, consider a parametric model $p_{\boldsymbol{\theta}}(\mathbf{x})$ describing observed data $\mathbf{x} \in \mathbb{R}^d$. The FIM with respect to $\boldsymbol{\theta}$ is given by
\begin{align}
    F(\boldsymbol{\theta})
    \;=\;
    \mathbb{E}_{\mathbf{x} \sim p_{\boldsymbol{\theta}}(\mathbf{x})}
    \Bigl[
       \nabla_{\boldsymbol{\theta}} \log p_{\boldsymbol{\theta}}(\mathbf{x})
       \;\nabla_{\boldsymbol{\theta}} \log p_{\boldsymbol{\theta}}(\mathbf{x})^\top
    \Bigr], \label{eq:FIMequation}
\end{align}
i.e., it measures the expected curvature of the log-likelihood and captures how sensitively the model's probability distribution depends on its parameters. Further discussions about the importance of the FIM and its relevance in optimization scenarios will follow later.
\subsection{Multivariate Calculus} \label{subsec:multivariatecalculus}
Throughout this thesis, we extensively employ multivariate calculus and carefully outline the notations and underlying concepts.
Given a function $f \, : \, \mathbb{R}^n \rightarrow \mathbb{R}^m$, we can write its Taylor series expansion about a point $\boldsymbol{\theta}^{(0)} \in \mathbb{R}^n$ as,
\begin{align}
    f(\boldsymbol{\theta}) = f(\boldsymbol{\theta}^{(0)}) + J(\boldsymbol{\theta}^{(0)})(\boldsymbol{\theta} - \boldsymbol{\theta}^{(0)}) + \frac{1}{2}(\boldsymbol{\theta} - \boldsymbol{\theta}^{(0)})^\top H(\boldsymbol{\theta}^{(0)})(\boldsymbol{\theta} - \boldsymbol{\theta}^{(0)}) + o(||\boldsymbol{\theta}-\boldsymbol{\theta}^{(0)}||_2^2), \notag
\end{align}
where $J(\boldsymbol{\theta}^{(0)}) \in \mathbb{R}^{m\times n}$ denotes the \emph{Jacobian matrix} evaluated at $\boldsymbol{\theta}^{(0)}$ and $H(\boldsymbol{\theta}^{(0)}) \in \mathbb{R}^{n\times m\times n}$ is the \emph{Hessian tensor} evaluated at $\boldsymbol{\theta}^{(0)}$. We use the notation,
\begin{align}
    [(\boldsymbol{\theta} - \boldsymbol{\theta}^{(0)})^\top H(\boldsymbol{\theta}^{(0)})(\boldsymbol{\theta} - \boldsymbol{\theta}^{(0)})]_j = \sum_{i=1}^n\sum_{k=1}^n (\boldsymbol{\theta} - \boldsymbol{\theta}^{(0)})_iH(\boldsymbol{\theta}^{(0)})_{ijk}(\boldsymbol{\theta} - \boldsymbol{\theta}^{(0)})_k, \notag
\end{align}
that is, the vector-tensor product is applied on the outer indices.
The Jacobian matrix can be defined as the collection of first-order derivatives between all input-output pairs of the function. Explicitly, writing $f = (f_1,\dots ,f_m)$ and $\boldsymbol{\theta} = (\theta_1,\dots,\theta_n)$ the $ij^{\text{th}}$ entry
of $J$ is given by,
\begin{align}
    J(\boldsymbol{\theta}_0)_{ij} = \frac{\partial f_i}{\partial \theta_j}\Big|_{\boldsymbol{\theta} = \boldsymbol{\theta}^{(0)}}. \notag
\end{align}
The Jacobian matrix is a linear operator that determines the rate of change of the function along some direction.
The Hessian tensor can be defined similarly,
\begin{align}
    H(\boldsymbol{\theta}_0)_{ijk} = \frac{\partial^2 f_j}{\partial\theta_i\partial\theta_k}\Big|_{\boldsymbol{\theta}=\boldsymbol{\theta}^{(0)}}. \notag
\end{align}
The Hessian determines the curvature of the function and is a particularly useful object for optimization. For instance, the Hessian can be used to classify stationary points as saddles, minima, or maxima and measures the conditioning of these stationary points as the ratio of the largest and smallest singular values.
\subsection{Neural Networks}
\label{subsec:neuralnetwor}
This thesis is ultimately concerned with understanding the properties of NN and how we can leverage this knowledge to develop a methodology for enhancing the convergence/generalization. Our investigation of NNs ranges from closed-form theoretical analysis to experimental study. In this section, we briefly introduce the notation that we use to describe NNs throughout this thesis — with an emphasis on the NNs that we focus on within our theoretical analysis.

We consider a supervised learning framework with a dataset $\mathbf{D}$ containing $N$ i.i.d samples, $\mathbf{D} \coloneq \{\mathbf{x}_{n}, \mathbf{y}_{n}\}_{n=1}^{N}$ where $\mathbf{x}_{n} \in \mathbb{R}^{d}$ and $\mathbf{y}_{n} \in \mathbb{R}^C$. Loosely speaking, a NN is a parametric function, $f_{\boldsymbol{\theta}} \, : \, \mathbb{R}^d \rightarrow \mathbb{R}^C$, parametrized by $\boldsymbol{\theta}$ where $\theta_i = \text{concat}(W_{i}, \mathbf{b}_i) \in \mathbb{R}^{P_i}$, and $P_i = P_{i}^{out} \times (P_{i}^{in} + 1)$. Let $\mathcal{L}: \mathbb{R}^C \times \mathbb{R}^C \rightarrow \mathbb{R}$ be the loss function defined by the Negative Log-Likelihood (NLL), i.e. $\mathcal{L}(\mathbf{y}, f_{\boldsymbol{\theta}}(\mathbf{x})) \coloneq - \log p_{\boldsymbol{\theta}}(\mathbf{y} | \mathbf{x})$ where $p_{\boldsymbol{\theta}}(\mathbf{y}|\mathbf{x})$ is the likelihood of the NN $f_{\theta}$. The network computes its output $h_{L} = f_{\boldsymbol{\theta}}(\mathbf{x})$ according to 
\begin{align}
    \mathbf{a}_{i} = \boldsymbol{\theta}_{i} \bar{\mathbf{h}}_{i-1}, \, \, \mathbf{h}_{i} = \phi_{i}(\mathbf{a}_{i}),  \, \, \forall \, i \in \{1, \dots, L\} \, \, | \, \, \mathbf{h}_{0} = \mathbf{x}, \notag
\end{align}
where $\bar{\mathbf{h}}_{i-1} = [\mathbf{h}_{i-1}^\top,\mathbf{1}]^\top \in \mathbb{R}^{P_{i-1}^{in}+1}$, and $\phi_i$ is an element-wise nonlinearity applied at layer $i$. For a given input target pair $(\mathbf{x},\mathbf{y})$, the gradient of the loss $\mathcal{L}(\mathbf{y}, f_{\boldsymbol{\theta}}(\mathbf{x}))$ concerning the weights are computed by the backpropagation algorithm \citep{Lecun2001}. For convenience, we adopt the special symbol $\mathbf{s}_{i} = \nabla_{\mathbf{a}_{i}}\mathcal{L}$ for the pre-activation derivative. Starting from $\nabla_{\mathbf{h}_{L}}\mathcal{L} = \partial_{\mathbf{h}_{L}}\mathcal{L}(\mathbf{y},\mathbf{h}_{L})$, we perform
\begin{align}
    \mathbf{s}_{i} \coloneq \nabla_{\mathbf{a}_{i}}\mathcal{L} = \nabla_{\mathbf{h}_{i}}\mathcal{L} \odot \phi_{i}'(\mathbf{a}_{i}), \, \nabla_{\theta_{i}}\mathcal{L} = \mathbf{s}_{i} \bar{\mathbf{h}}_{i-1}^\top, \, \nabla_{\bar{\mathbf{h}}_{i-1}}\mathcal{L} = \boldsymbol{\theta}_{i}^\top \mathbf{s}_{i} \quad | \, \, \forall i \in \{L, \dots, 1\},  \notag
\end{align}
where $\odot$ denotes the element-wise product. Finally, the gradient $\nabla_{\boldsymbol{\theta}} \mathcal{L}$ is retrieved by: $\nabla_{\boldsymbol{\theta}}\mathcal{L} = [\text{vec}(\nabla_{\theta_{1}}\mathcal{L})^\top,\text{vec}(\nabla_{\theta_{2}}\mathcal{L})^\top, \dots, \text{vec}(\nabla_{\theta_{L}}\mathcal{L})^\top]^\top$; $\text{vec}(\cdot)$ denotes the Kronecker vectorization operator which stacks the columns of a matrix into a vector. Throughout this thesis we will use the notation $\mathbf{g}^{(t)} = \nabla_{\boldsymbol{\theta}} \mathcal{L}(\boldsymbol{\theta}^{(t)})$.
\paragraph{Network Jacobian.} Network Jacobian in Section~\ref{subsec:multivariatecalculus} we defined the Jacobian matrix. Here we write NNs as functions that depend on some parameters $\boldsymbol{\theta}$ and operate on some data $\mathbf{x}$. We typically consider the Jacobian matrix of NNs with respect to their parameters $\boldsymbol{\theta}$, and they are thus written as $J_{\boldsymbol{\theta}}(\mathbf{x})$.
\subsection{Optimization}\label{subsec:optimization}
In this section, we review the background material that is relevant to the optimization components of this thesis. We will formalize the optimization problem considered when training a NN.
\paragraph{Setting.} The optimization problems that we consider can be represented via an objective function, $\mathcal{L}$ (e.g., NLL defined in Section~\ref{subsec:neuralnetwor}). In the settings that we study in this thesis, the objective function is evaluated for a given set of model parameters and observations. In the most general case, we use $\boldsymbol{\theta}$ to denote the model parameters. Note that in this thesis we will consider a supervised learning framework. Typically, the objective function can be decomposed as the average of several independent loss terms,
\begin{align}
    \mathcal{L}(\boldsymbol{\theta}) = \frac{1}{N} \sum_{n=1}^N\ell(\boldsymbol{\theta};\mathbf{x}_n,\mathbf{y}_n) \notag
\end{align}
Our goal is to minimize this objective function with respect to the model parameters. In doing so successfully, we recover a global minimum,
\begin{align}
    \boldsymbol{\theta}^* \in \underset{\boldsymbol{\theta}}{\text{arg min}} \, \, \mathcal{L}(\boldsymbol{\theta}) \notag
\end{align}
Notice the use of $\in$ instead of $=$ operator, this is because of the nature of the problem, which in the case of NN optimization, is highly non-convex. In the typical applications that we care about, such as deep learning optimization, $n$ is far too large to evaluate $\mathcal{L}$ efficiently. In this case, we often resort to stochastic optimization where we obtain statistics of the objective function by evaluating it using only a randomly chosen subset of the observations.
\paragraph{Gradient Descent.} A ubiquitous approach to minimizing $\mathcal{L}$ is to start with some initial guess of the optimal parameters, $\boldsymbol{\theta}_0$, and follow the gradient from $\boldsymbol{\theta}_0$ towards an optimum. This algorithm is known as gradient descent, and can be represented by the following recursive computation,
\begin{align}
    \boldsymbol{\theta}^{(t+1)} = \boldsymbol{\theta}^{(t)} - \alpha^{(t)}\mathbf{g}^{(t)}, \label{eq:gradientdescentupdaterule}
\end{align}
where $\alpha^{(t)}$ denotes the learning rate. This algorithm (and extensions of it) still form the basis of modern deep learning optimization. To better understand it, we will now provide a brief review of convex optimization and the behaviour of gradient descent under convex objectives.
\paragraph{Convex Optimization.} In convex optimization, we restrict our attention to the class of convex objective functions.
\begin{definition}
A function $f\, : \, \mathbb{R}^d \rightarrow \mathbb{R}$ is convex if and only if the following holds. For all $t \in [0,1]$, and all $\mathbf{x},\mathbf{y} \in \mathbb{R}^d$,
\begin{align}
    f(t\mathbf{x} + (1-t)\mathbf{y}) \leq tf(\mathbf{x}) + (1-t)f(\mathbf{y}). \notag
\end{align}
\end{definition}
In general, a convex optimization problem may include a set of constraints such that the feasible set of solutions is a convex set. Unless stated otherwise, we restrict our attention to unconstrained convex optimization.

Perhaps the simplest non-trivial convex function is the \emph{quadratic function}, which (without loss of generality (WLG)) takes the form,
\begin{align}
    f(\mathbf{x}) = (\mathbf{x}-\mathbf{x}^*)^\top A(\mathbf{x} - \mathbf{x}^*), \label{eq:quadratic}
\end{align}
where \(A\) is a symmetric positive semidefinite matrix. Convex quadratic optimization problems form the cornerstone of second-order optimization algorithms, since at each iteration the loss function is approximated by a convex quadratic via a second-order Taylor expansion, thereby making the resulting subproblem significantly easier to solve.

Note that throughout this thesis we consider \emph{unconstrained composite convex optimization problems} which are significantly simpler than their constrained counterparts.
\paragraph{Convergence on convex objectives.} The following properties of (unconstrained) convex optimization are key to our success in optimization \citep{boyd2004convex}:
\begin{itemize}
    \item Every locally optimal point of a convex objective function is globally optimal.
    \item If $f$ is differentiable, then a point $\mathbf{x}$ is optimal if and only if $\nabla_{\mathbf{x}} f = 0$.
\end{itemize}
Importantly, the convergence rate of gradient descent at an optimum of an unconstrained convex optimization problem depends on the \emph{condition number} of the Hessian matrix. That is the ratio of the largest eigenvalue to the smallest eigenvalue of the Hessian matrix. The larger the condition number, the slower the convergence.

Consider a quadratic objective with $\kappa$ denoting the condition number of the Hessian matrix ($H= A$) in Eq. (\ref{eq:quadratic}). Then the convergence rate of gradient descent (with optimal learning rate) is given by,
\begin{align}
    f(\mathbf{x}^{(t)}) - f(\mathbf{x^*)} = \left(1-\frac{1}{\kappa}\right)^t(f(\mathbf{x}_0) - f(\mathbf{x}^*)), \notag
\end{align}
where $x_0$ is the initial point. One commonly used method to improve convergence under ill-conditioned objectives is by utilizing acceleration. These methods typically reduce the dependency on the condition number from $1/\kappa$ to $1/\sqrt{\kappa}$.
\section{Deep Learning Dynamics}
\label{sec:dynamics}
The dynamics of deep learning encompass the evolution of network parameters as a function of training time and the interplay between the optimization method, loss landscape, and model architecture. These dynamics are central to understanding how deep networks transition between different regimes of learning, how they navigate highly non-convex landscapes, and how they converge to solutions with desirable generalization properties. In this section, we review the state-of-the-art (SOTA) in deep learning dynamics with a focus on three aspects: training dynamics and phase transitions, loss landscape geometry, and implicit regularization.
\subsection{Training Dynamics and Phase Transitions}
\label{subsec:training_dynamics}
\paragraph{Neural Tangent Kernel (NTK) Theory.}
A powerful theoretical lens on the training behavior of overparameterized NN is provided by the \emph{Neural Tangent Kernel (NTK)} framework \citep{jacot2018neural}. Consider a NN \( f_{\boldsymbol{\theta}} : \mathbb{R}^{d} \to \mathbb{R} \) (WLG, assume a scalar output) parameterized by \(\boldsymbol{\theta} \in \mathbb{R}^p\). At any parameter setting \(\boldsymbol{\theta}\), the \emph{NTK} between two inputs \(\mathbf{x}, \mathbf{x}' \in \mathbb{R}^d\) is defined as
\begin{align}
    \mathcal{K}_{\boldsymbol{\theta}}(\mathbf{x},\mathbf{x}')
    \;=\;
    \nabla_{\boldsymbol{\theta}} f_{\boldsymbol{\theta}}(\mathbf{x})^\top
    \,\nabla_{\boldsymbol{\theta}} f_{\boldsymbol{\theta}}(\mathbf{x}') \notag
    \,.
\end{align}
In the \emph{NTK regime}---often associated with very wide (infinite-width) NNs---the evolution of network parameters under gradient descent can be approximated by linearizing \(f_{\boldsymbol{\theta}}(\mathbf{x})\) around its initialization \(\boldsymbol{\theta}^{(0)}\). Concretely, writing \(\boldsymbol{\theta}^{(t)}\) for the parameters at iteration \(t\), 
\begin{align}
    f_{\boldsymbol{\theta}^{(t)}}(\mathbf{x})
    \;\approx\;
    f_{\boldsymbol{\theta}^{(0)}}(\mathbf{x})
    \;+\;
    \nabla_{\boldsymbol{\theta}} f_{\boldsymbol{\theta}^{(0)}}(\mathbf{x})^\top
    \bigl(\boldsymbol{\theta}^{(t)} \;-\; \boldsymbol{\theta}^{(0)}\bigr). \notag
\end{align}
Under this approximation, the kernel \(\mathcal{K}_{\boldsymbol{\theta}^{(0)}}(\mathbf{x},\mathbf{x}')\) remains nearly constant throughout training. As a result, the training dynamics reduce to a linearized model whose parameters evolve in a predictable manner. While this perspective sheds light on many salient features of training---including the speed of convergence and the capacity for perfect interpolation in overparameterized networks---it only partially captures the highly non-linear aspects observed in real-world scenarios where finite width, adaptive optimizers, and more complex architectures induce additional effects.
\paragraph{Phase Transitions in Training.}
The training process of deep networks is characterized by complex temporal behavior that often manifests as distinct phases \citep{fort2020deep}. Early in training, the network typically undergoes rapid changes, while later stages are marked by slower, diffusion-like dynamics. Recent studies have systematically mapped out these phases by examining how Hyper-Parameters (HP) such as learning rate, depth, and width influence convergence behavior \citep{erhan2010does, xie2022adaptive}. For example, in \cite{kalra2023phase}, the authors present a phase diagram that distinguishes regimes of fast initial descent from later phases where learning slows, thereby highlighting critical transition points that affect both convergence and generalization. In \citep{Jastrzebski2020The}, the authors analyze the importance of the early phase of training and its critical impact on final performance, arguing that key properties of the loss surface are strongly influenced by initial learning dynamics.

A complementary perspective is provided by theoretical frameworks such as the NTK described above. Although the NTK regime simplifies the analysis by approximating the parameter evolution as quasi-linear, it captures only a fraction of the full non-linear behavior observed in practice. In more realistic settings---where non-linear activations, finite widths, and adaptive optimization methods are prevalent---the training dynamics exhibit richer behavior, including abrupt shifts (or phase transitions) that mark the onset of distinct learning regimes.

The inherent stochasticity in optimization methods like Stochastic Gradient Descent (SGD) \citep{marcus-etal-1993-building} introduces noise that can help escape poor local minima and saddle points. For instance, \cite{jin2017escape} demonstrated that by adding noise at each step, gradient descent can escape saddle points in polynomial time. Moreover, \cite{xie2021a} proposes a diffusion theory for deep learning dynamics in which the stochastic fluctuations inherent in SGD introduce an additional regularization effect, steering the optimization trajectory toward broader, flatter minima. This bias toward flat solutions is argued to underpin the superior generalization performance observed in deep networks, as flatter minima are typically more robust to perturbations and less prone to overfitting \citep{frankle2018the}.
\subsection{Loss Landscape Geometry}
\label{subsec:loss_landscape}
\paragraph{What is a loss landscape?} A loss landscape is determined by the output of the loss function over some subset of the optimization parameters. The loss landscape describes the space explored by an optimization algorithm and ultimately determines the optimizer’s trajectory. Figure~\ref{fig:losslandscape} displays a 2D slice of the loss landscape of a fully-connected network. 
\begin{figure}[!ht]
    \centering
    \includegraphics[width=0.7\linewidth]{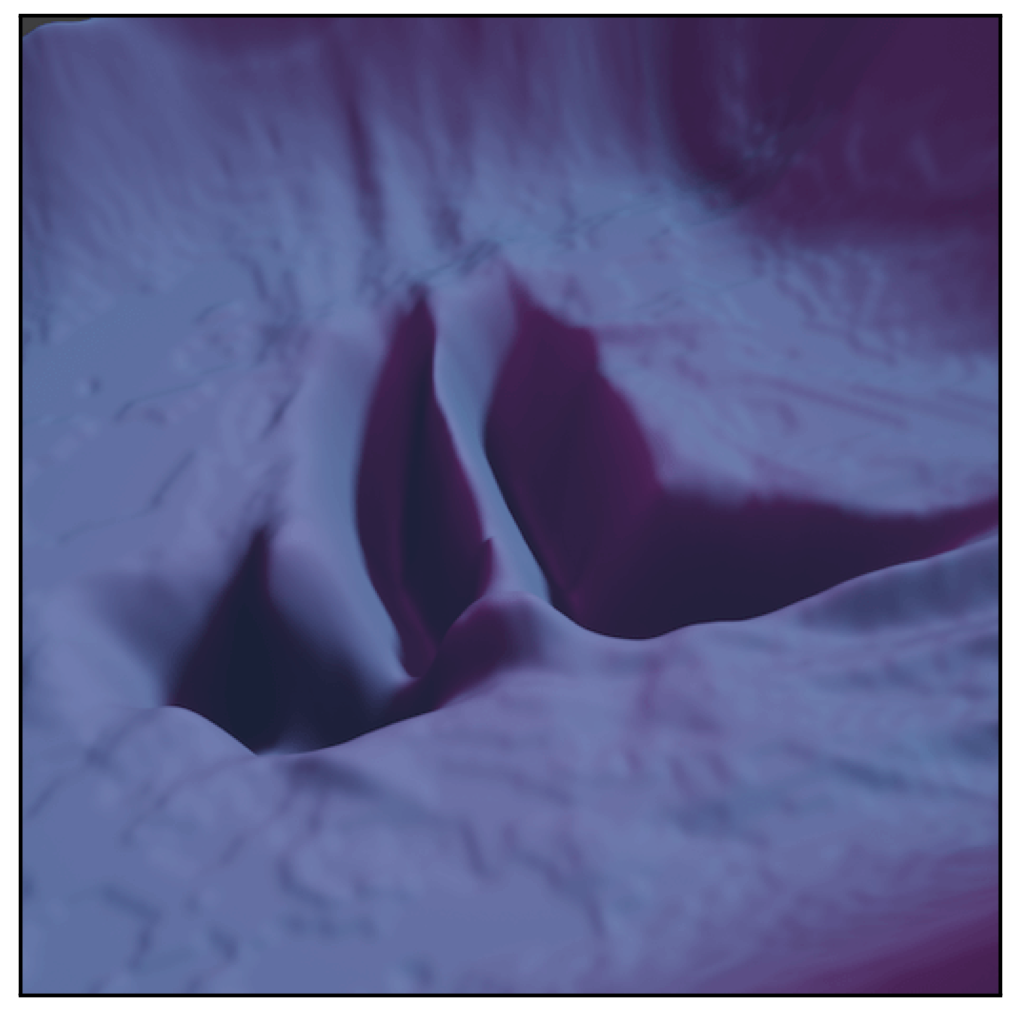}
    \caption{A 2D slice of a loss landscape. This is a 3D rendering of a 2D slice of the loss landscape of a fully-connected network. The slice itself was chosen via gradient descent, using an objective that encouraged this 2D slice to match a target image. Full details are given by \cite{lucas2022optimization}.}
    \label{fig:losslandscape}
\end{figure}
\paragraph{Geometry Intuitions.} The geometry of the loss landscape is a key determinant of both the ease of optimization and the generalization ability of a DNN. Empirical investigations have revealed that the loss surfaces of deep models are \emph{highly non-convex} and characterized by a multitude of local minima and saddle points. Yet, intriguingly, many of these minima are connected via \emph{low-loss pathways}, suggesting that the landscape possesses an underlying structure that facilitates optimization \citep{chiang2023loss, garipov2018loss}. One influential concept in this context is the \emph{Goldilocks zone}, which refers to regions in the parameter space where the curvature is balanced---not too steep and not too flat. Studies such as \cite{fort2019goldilocks} and \cite{vysogorets2024deconstructing} have shown that minima in these zones tend to be flat, and flatness has been empirically associated with better generalization performance. The spectral properties of the Hessian matrix have often been used to characterize these regions: minima with a concentration of small eigenvalues (i.e., flat minima) are typically more robust to perturbations, whereas sharp minima, characterized by large eigenvalues, are more sensitive and prone to overfitting \citep{keskar2016large, dinh2017sharp}. Several works have further illuminated the interplay between optimization dynamics and loss landscape geometry. For instance, \cite{chaudhari2019entropy} introduce Entropy-SGD, which augments the loss function with an entropy term \(-\beta^{-1}\, \mathrm{H}(\boldsymbol{\theta})\) to the loss:
\begin{align}
    \tilde{\mathcal{L}}(\boldsymbol{\theta})
    \;=\;
    \mathcal{L}(\boldsymbol{\theta})
    \;-\;
    \beta^{-1} \,\mathrm{H}(\boldsymbol{\theta}),
    \quad
    \mathrm{H}(\boldsymbol{\theta}) 
    \;=\;
    -\int p(\boldsymbol{\phi} | \boldsymbol{\theta}) \,\ln p(\boldsymbol{\phi} | \boldsymbol{\theta})\,d\boldsymbol{\phi}, \notag
\end{align}
where \(\beta\) is a temperature-like parameter and \(p(\boldsymbol{\phi} |\boldsymbol{\theta})\) is a local Gibbs distribution around \(\boldsymbol{\theta}\).  Similarly, \cite{xie2021a} develops a diffusion theory for deep learning dynamics, modeling the update of parameters \(\boldsymbol{\theta}^{(t)}\) under SGD as
\begin{align}
    \boldsymbol{\theta}^{(t+1)}
    \;=\;
    \boldsymbol{\theta}^{(t)}
    \;-\;
    \alpha \,\nabla_{\boldsymbol{\theta}} \mathcal{L}(\boldsymbol{\theta}^{(t)})
    \;+\;
    \sqrt{2\,\alpha\,T_{\mathrm{eff}}}\,\mathbf{\xi}^{(t)}, \notag
\end{align}
where \(\mathbf{\xi}^{(t)}\) is an isotropic noise term and \(T_{\mathrm{eff}}\) is an effective temperature capturing batch-size and learning-rate dependence. Their analysis shows that the resulting parameter distribution is biased exponentially toward regions where \(H(\boldsymbol{\theta})\) has smaller eigenvalues, thus promoting flatter minima. Their results quantitatively show that SGD exponentially favors flat minima, offering a theoretical underpinning for the empirically observed link between flatness and generalization. The concept of flatness is further enriched by the study of asymmetric valleys. \cite{he2019asymmetric} observe that the local minima in deep networks are often asymmetric: the loss increases abruptly in some directions while rising more gradually in others. Under mild assumptions, they prove that solutions biased toward the flat side of an asymmetric valley generalize better than the exact empirical minimizers. This finding is supported by weight-averaging techniques such as Stochastic Weight Averaging (SWA)~\citep{izmailov2018averaging}, which implicitly guide the model parameters toward flatter regions. Finally, optimization methods that explicitly target flatter regions---such as Sharpness-Aware Minimization (SAM) \citep{foret2021sharpnessaware}---modify the training objective to explicitly penalize local sharpness by modifying the training objective to
\begin{align}
    \mathcal{L}_\mathrm{SAM}(\boldsymbol{\theta})
    \;=\;
    \min_{\boldsymbol{\theta}} \; \max_{\|\delta\|\leq \rho} \mathcal{L}(\boldsymbol{\theta} + \delta), \label{eq:samequation}
\end{align}
with a radius parameter \(\rho > 0\). By ensuring robust performance in a neighborhood around \(\boldsymbol{\theta}\), SAM systematically favors flatter basins of the loss landscape.

Consequently, these studies underscore that a comprehensive understanding of loss landscape geometry—including factors such as curvature, asymmetry, and the connectivity of low-loss regions—is crucial not only for explaining the dynamics of SGD but also for designing algorithms that enhance generalization in deep networks.

\subsection{Implicit Regularization}
\label{subsec:implicit_reg}
Implicit regularization refers to the phenomenon by which the training algorithm itself biases the optimization process towards solutions with desirable generalization properties, even in the absence of explicit regularization terms. A growing body of evidence suggests that the dynamics of methods such as SGD inherently favor flat minima, which are linked to improved robustness and generalization \cite{frankle2018the, kalra2023phase}. This effect arises partly because the noise injected by stochastic sampling encourages exploration of the loss landscape, while the iterative nature of gradient descent effectively smooths out sharp curvatures.

The interplay between implicit regularization and loss landscape geometry is complex. For instance, weight decay—commonly applied as an explicit regularizer—has been shown to interact with the gradient dynamics in ways that further promote the selection of flatter minima \cite{xie2023on}. Moreover, while adaptive optimization methods (e.g., Adam \citep{kingma2017adam}, RAdam \citep{Liu2020On}) often accelerate convergence, they may sometimes undermine the implicit bias towards flatness, leading to a trade-off between optimization speed and generalization quality \citep{Liu2020On}. Recent studies have also drawn connections between implicit regularization and the lottery ticket hypothesis \citep{frankle2018the}, which posits that \emph{sparse subnetworks} within over-parameterized models can perform comparably to the full network when appropriately trained. Such observations highlight the multifaceted role of implicit regularization in deep learning.

In a nutshell, these three facets—training dynamics and phase transitions, loss landscape geometry, and implicit regularization—provide a framework for understanding how deep networks learn and generalize. The following sections will build on this foundation to discuss the key challenges in optimizing deep networks and survey a range of optimization methods that have been developed to tackle these issues.
\section{Key Challenges in Optimizing Deep Networks}
\label{sec:challenges}
The optimization of deep networks presents several interrelated challenges that stem from the inherent complexity of the models and the high-dimensional, non-convex nature of their loss landscapes. In this section, we discuss three core challenges: the difficulties posed by non-convex optimization, the delicate balance between achieving generalization and avoiding overfitting, and the computational constraints that arise when scaling optimization algorithms to large networks.
\subsection{Non-Convex Optimization}
\label{subsec:nonconvex_optimization}
Deep networks are trained by minimizing loss functions that are highly non-convex, characterized by a profusion of local minima, saddle points, and flat regions as illustrated in Figure~\ref{fig:nonconvexillustration}. This non-convexity complicates the task of finding globally optimal solutions and often necessitates reliance on local search methods. SGD and its variants have emerged as the workhorses for deep learning partly because the inherent noise in SGD helps in escaping shallow local minima and saddle points \citep{kalra2023phase, dauphin2014identifying}. Moreover, several studies have shown that the structure of the loss landscape is such that many local minima are of comparable quality in terms of generalization \citep{chiang2023loss, choromanska2015loss}. This observation implies that even if the global optimum is not reached, the solution found by SGD can still generalize well. To further navigate the complex curvature of the loss surface, second-order information is often exploited. Approaches such as Kronecker-Factored Approximate Curvature (K-FAC) \citep{martens2015optimizing} and its efficient variants \citep{9578481, zhang2023eva, benzing2022gradient, mozaffari2023mkor} seek to provide accurate curvature approximations that help steer the optimization process. However, their increased computational overhead remains a significant challenge, especially in large-scale settings. In parallel, classical optimization techniques such as the Limited-memory BFGS algorithm (L-BFGS) \citep{liu1989limited} have also influenced modern algorithmic developments. While L-BFGS is capable of providing precise curvature information, its direct application in stochastic, high-dimensional environments typical of deep learning is hindered by sensitivity to noise and scalability issues. As a result, recent work has focused on hybrid strategies that combine the efficiency of first-order methods with selective second-order insights, aiming to better exploit the geometry of the loss surface without incurring prohibitive computational costs.
\begin{figure}[!ht]
    \centering
    \includegraphics[width=\linewidth]{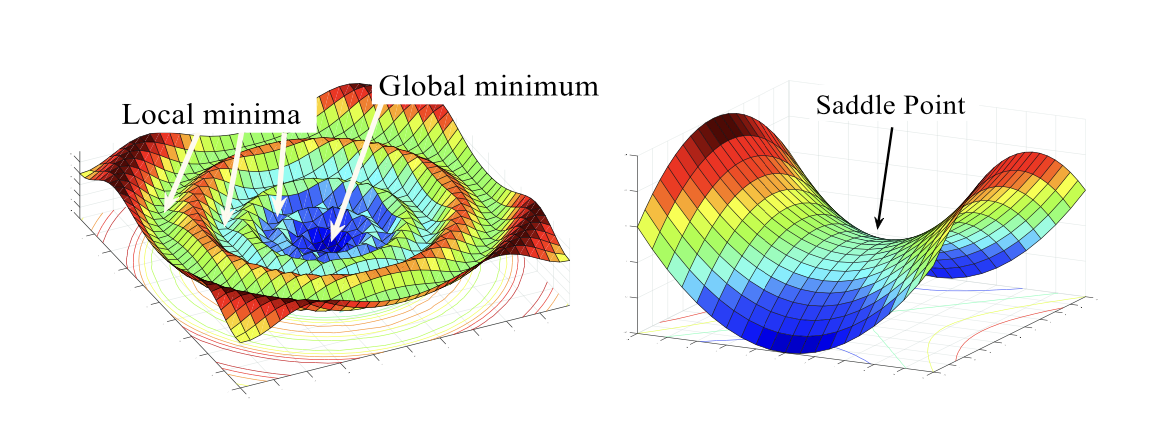}
    \caption{Three types of stationary points in non-convex optimization landscapes:
local minima, global minima, and saddle points.}
    \label{fig:nonconvexillustration}
\end{figure}
\subsection{Generalization vs. Overfitting}
\label{subsec:generalization_overfitting}
A central paradox in deep learning is that over-parameterized models can achieve excellent generalization performance despite their capacity to overfit. The challenge lies in steering the optimization process toward solutions that generalize well rather than merely fitting the training data. Implicit regularization, inherent in methods like SGD, plays a critical role by biasing training towards flatter minima---regions in the loss landscape that are less sensitive to perturbations and are empirically linked to improved generalization \citep{foret2021sharpnessaware, xie2023on, soudry2018implicit}. Explicit regularization techniques such as weight decay and dropout have traditionally been employed to mitigate overfitting. However, recent research suggests that the choice of optimization algorithm itself can exert a significant regularizing effect. For instance, adaptive methods like Adam and its variants---including AdamP and AdaBelief \citep{zhuang2020adabelief}---are known to accelerate convergence; yet, they may sometimes sacrifice the implicit bias toward flat minima that is beneficial for generalization \citep{wilson2017marginal}. Moreover, margin-based analyses have provided theoretical insights into how the implicit regularization of gradient-based methods can yield better generalization bounds even in over-parameterized settings \citep{ bartlett2017spectrally}. The ongoing investigation into the interplay between these dynamics is further enriched by studies on the lottery ticket hypothesis \citep{frankle2018the}, which posit that certain sparse subnetworks are inherently predisposed to generalize well. Collectively, these insights continue to drive the development of optimization strategies that aim to balance fitting accuracy with robust generalization.

\subsection{Computational Constraints}
The high computational cost associated with training deep networks is a persistent challenge, particularly as models continue to scale in size and complexity. While second-order methods, which utilize curvature information, can in principle lead to faster convergence, their direct application is often hindered by prohibitive memory and time requirements. Methods such as Shampoo \citep{gupta2018shampoo} and its variants attempt to incorporate second-order information through Kronecker decompositions, yet these approaches can be computationally intensive, especially for large-scale models.

To address these issues, researchers have proposed a variety of efficient approximations and hybrid methods. For example, efficient implementations of K-FAC \citep{frantar2021mfac, george2018fast, botev2017practical} and recent proposals such as FOOF \citep{benzing2022gradient} and M-FAC \citep{frantar2021mfac} aim to strike a balance between leveraging curvature information and maintaining computational feasibility. Additionally, first-order methods, including adaptive optimizers like RAdam and AdaFactor \citep{shazeer2018adafactor}, continue to be refined to reduce overhead while preserving convergence quality. The constant drive to lower computational burdens has not only spurred algorithmic innovations but also influenced hardware-aware optimizations, ensuring that deep network training remains tractable on modern computational platforms.

Hence, the challenges of non-convex optimization, balancing generalization with overfitting, and managing computational constraints define the current landscape of deep network optimization. Addressing these challenges continues to be a major focus of both theoretical investigations and practical algorithm design, as evidenced by the ongoing contributions to the machine learning community.
\section{Optimization Methods}
\label{sec:methods}
Optimization methods for deep networks can be categorized into four main groups: zeroth-order, first-order, second-order, and hybrid methods. Each category leverages different types of information and offers distinct trade-offs in computational cost, convergence behavior, and memory requirements. In this section, we provide detailed mathematical descriptions of these methods and refer to prominent works published in top conferences.
\subsection{Zeroth-Order Methods}
\label{subsec:zeroth_order}
Zeroth-Order (ZO) methods, also known as \emph{derivative-free} optimization techniques, form a class of algorithms that rely exclusively on function evaluations rather than on explicit gradient information. These approaches prove particularly useful in settings where gradients are unavailable, unreliable, or prohibitively expensive to compute. Applications span a wide range of areas, including black-box optimization, reinforcement learning, adversarial attacks, and other non-differentiable scenarios.
\paragraph{Zeroth-Order Optimization in NNs.} ZO methods address the parameter optimization problem 
\begin{align}
    \min_{\boldsymbol{\theta} \in \mathbb{R}^P}
    \;\mathcal{L}(\boldsymbol{\theta}) \label{eq:minimizationformualtionzeroorder}
\end{align}
for an $L$-layer NN $f_{\boldsymbol{\theta}}$ \emph{without} using backpropagation. Here, $P = \sum_{i=1}^L P_i$ is the total number of parameters. Instead of computing the parameter gradient via automatic differentiation, these methods approximate $\nabla_\theta \mathcal{L}(\boldsymbol{\theta}) $ by sampling \emph{function evaluations} of the loss landscape. When the network is $L$-smooth, meaning there exists a constant $L>0$ such that
\begin{align}
    \|\nabla_{\boldsymbol{\theta}} \mathcal{L}(\boldsymbol{\theta}) \;-\; \nabla_{\boldsymbol{\theta}} \mathcal{L}(\boldsymbol{\theta}')\|
    \;\leq\; 
    L\|\boldsymbol{\theta} \;-\; \boldsymbol{\theta}'\|, \notag
\end{align}
a classical \emph{coordinate-wise} finite-difference estimator takes the form
\begin{align}
    \widehat{\partial_{\theta_i} \mathcal{L}}(\boldsymbol{\theta})
    \;=\;
    \frac{\mathcal{L}(\boldsymbol{\theta} + h\,\mathbf{e_i}) 
          \;-\; 
          \mathcal{L}(\boldsymbol{\theta})}{h}, \notag
\end{align}
where $\mathbf{e_i}$ is the $i^{\text{th}}$ canonical basis vector in $\mathbb{R}^P$, and $h>0$ is a small finite-difference step size (or \emph{perturbation scale}). This method requires $\mathcal{O}(P)$ function evaluations per update, which can be prohibitive for large $P$ \citep{Golovin2020Gradientless}.

Modern implementations often opt for \emph{random directional} derivatives that perturb all parameters simultaneously. A typical estimator is
\begin{align}
    \widehat{\nabla}_{\boldsymbol{\theta}}\,\mathcal{L}(\boldsymbol{\theta})
    \;=\;
    \frac{\mathcal{L}\bigl(\boldsymbol{\theta} + h\,\mathbf{u}\bigr) 
          \;-\;
          \mathcal{L}(\boldsymbol{\theta})}{h}
    \,\mathbf{u},
    \quad
    \mathbf{u}\sim \mathcal{N}(\mathbf{0}, \mathbf{I}_P), \notag
\end{align}
where $\mathbf{u}$ is drawn from a spherically symmetric distribution (e.g., a Gaussian). In expectation, this \emph{directional} scheme optimizes a smoothed version of the loss
\begin{align}
    \mathcal{L}_h(\boldsymbol{\theta})
    \;=\;
    \mathbb{E}_{\mathbf{u}}\bigl[\mathcal{L}(\boldsymbol{\theta} + h\,\mathbf{u})\bigr],
    \quad
    \mathbb{E}_{\mathbf{u}}\bigl[\widehat{\nabla}_{\boldsymbol{\theta}}\,\mathcal{L}(\boldsymbol{\theta})\bigr]
    \;=\;
    \nabla_{\boldsymbol{\theta}} \mathcal{L}_h(\boldsymbol{\theta}), \notag
\end{align}
which provides a form of regularization in high-dimensional spaces \citep{ghadimi2013stochastic}.

The parameter updates mirror those of SGD (see Eq. (\ref{eq:gradientdescentupdaterule})) but employ estimated gradients
\begin{align}
    \boldsymbol{\theta}^{(t+1)} 
    \;=\; 
    \boldsymbol{\theta}^{(t)} 
    \;-\;
    \alpha^{(t)}\,\widehat{\nabla}_{\boldsymbol{\theta}}\,\mathcal{L}\bigl(\boldsymbol{\theta}^{(t)}\bigr). \notag
\end{align}
Under bounded variance, i.e.\ 
\(\mathbb{E}\bigl[\|\widehat{\nabla}_{\boldsymbol{\theta}}\,\mathcal{L}(\boldsymbol{\theta}) - \nabla_{\boldsymbol{\theta}}\,\mathcal{L}(\boldsymbol{\theta})\|^2\bigr] \le \sigma^2\), and assuming $L$-smoothness, the average squared gradient norm over $T$ iterations satisfies
\begin{align}
    \frac{1}{T}\sum_{t=1}^T \mathbb{E}\bigl[\|\nabla_{\boldsymbol{\theta}} \mathcal{L}(\boldsymbol{\theta}^{(t)})\|^2\bigr]
    \;\le\;
    \mathcal{O}\!\Bigl(\frac{L\bigl(\mathcal{L}\bigl(\boldsymbol{\theta}^{(1)}\bigr) - \mathcal{L}^*\bigr)}{T}
    \;+\;
    \frac{P\,L\,\sigma^2}{T}\Bigr), \notag
\end{align}
where $\mathcal{L}^*$ is the global optimum \citep{ghadimi2013stochastic}. Consequently, to achieve an $\epsilon$-stationary solution, one needs $\mathcal{O}\bigl(\frac{P}{\epsilon^2}\bigr)$ iterations, reflecting the \emph{dimension dependence} in zeroth-order estimates. 

Parallelization can partially mitigate this cost. For instance, \cite{lian2016comprehensive} show that with $K$ parallel workers and communication delay bound $\tau$, the iteration complexity improves to $\mathcal{O}\!\bigl(\tfrac{P}{K\,\epsilon^2}\bigr)$ under certain conditions, maintaining a near-linear speedup when $K \leq \mathcal{O}\!\bigl(T^{1/4}\bigr)$. Additionally, \cite{al2020sign} propose \emph{sign-based} compression of the updates
\begin{align}
    \boldsymbol{\theta}^{(t+1)} 
    \;=\;
    \boldsymbol{\theta}^{(t)}
    \;-\;
    \alpha^{(t)}
    \,\mathrm{sign}\!\Bigl(\widehat{\nabla}_{\boldsymbol{\theta}} \mathcal{L}\bigl(\boldsymbol{\theta}^{(t)}\bigr)\Bigr), \notag
\end{align}
which reduces communication overhead. This approach is further motivated by the finding that $\mathrm{sign}\bigl(\nabla_{\boldsymbol{\theta}} \mathcal{L}\bigr)$ often preserves sufficient information about the descent direction, especially when coupled with adaptive step sizes \citep{zhao2025secondorder}. Such sign-based updates highlight the broad flexibility of zeroth-order optimization in large-scale, high-dimensional NN training. 
\paragraph{Recent Advances and Practical Successes.}
Driven by the need to make derivative-free methods more efficient, recent techniques incorporate adaptive sampling and second-order insights. For instance, MeZO \citep{malladi2023fine} circumvents backpropagation by directly estimating gradients through carefully designed perturbations in a \emph{symmetric} manner
\begin{align}
    \widehat{\nabla}_{\boldsymbol{\theta}} \mathcal{L}(\boldsymbol{\theta})
    \;=\;
    \frac{\mathcal{L}(\boldsymbol{\theta} + \epsilon\,\mathbf{z}) - \mathcal{L}(\boldsymbol{\theta} - \epsilon\,\mathbf{z})}{2\,\epsilon}
    \;\mathbf{z}, \notag
\end{align}
where $\mathbf{z} \sim \mathcal{N}(\mathbf{0}, \mathbf{I}_P)$. Under $\mathcal{L}\in C^3$, this estimator yields $\nabla_{\boldsymbol{\theta}} \mathcal{L}(\boldsymbol{\theta})$ up to $\mathcal{O}(\epsilon^2)$ \citep{chen2019zo,malladi2023fine}, and storing random seeds rather than full vectors $\mathbf{z}$ improves memory usage. \cite{zhao2025secondorder} combine variance reduction and adaptive learning rates for faster convergence in high-dimensional settings, extending the work of \cite{ghadimi2013stochastic,lian2016comprehensive} to show that under unbiasedness and bounded-variance assumptions, zeroth-order updates can inherit many classical SGD convergence properties. Other influential contributions have broadened the scope of derivative-free optimization. \cite{Golovin2020Gradientless} highlight the challenges of exponential growth in function evaluations with dimension, emphasizing the importance of efficient sampling schemes. \cite{rando2024optimal} reduce variance via smoothing-based estimators, and \cite{larson2019derivative} address adversarial noise with robust ZO methods. \cite{zhang2022how} propose flexible exploration-exploitation strategies for randomized search. Collectively, these studies underscore that while zeroth-order methods can be a powerful alternative when gradients are unavailable or unreliable, their practical utility depends heavily on \emph{reducing} the number of function evaluations. 

Moreover, other methods such as ZO-SVRG (Zeroth-Order Stochastic Variance Reduced Gradient) \citep{liu2018zeroth} employs control variates to reduce variance in the gradient approximation. Specifically, it defines
\begin{align}
    \mathbf{m}^{(t)} = \widehat{\nabla}_{\boldsymbol{\theta}} \mathcal{L}(\boldsymbol{\theta}^{(t)}) - \widehat{\nabla}_{\boldsymbol{\theta}} \mathcal{L}(\tilde{\boldsymbol{\theta}}) + \tilde{\boldsymbol{\mu}}, \quad \tilde{\boldsymbol{\mu}} = \frac{1}{|\mathcal{B}|}\sum_{i\in \mathcal{B}} \widehat{\nabla}_{\boldsymbol{\theta}} \mathcal{L}_i\bigl(\tilde{\boldsymbol{\theta}}\bigr). \notag
\end{align}
where \(\tilde{\boldsymbol{\theta}}\) is a reference parameter (typically updated periodically), and \(\tilde{\boldsymbol{\mu}}\) is computed as the average of the zeroth-order gradient estimates over a mini-batch \(\mathcal{B}\). This estimator combines the current gradient approximation \(\widehat{\nabla}_{\boldsymbol{\theta}} \mathcal{L}(\boldsymbol{\theta}^{(t)})\) with a correction term that accounts for the discrepancy between the current parameter and the reference parameter, thus reducing the overall variance of the estimate. The update rule for the parameters is then given by 
\begin{align}
\boldsymbol{\theta}^{(t+1)} = \boldsymbol{\theta}^{(t)} - \alpha \,\mathbf{m}^{(t)}, \notag
\end{align}
where \(\alpha\) is the learning rate. By leveraging both the current and historical gradient information, ZO-SVRG is able to reduce the variance from \(\mathcal{O}(d/\sqrt{T})\) to \(\mathcal{O}(d/T + 1/b)\), thereby stabilizing convergence even in high-dimensional settings. 

Furthermore, the adaptive framework has been also used in ZO methods as demonstrated by ZO-AdaMM \citep{chen2019zo} which extends Adam into a zeroth-order framework by iterating
\begin{align}
\mathbf{m}^{(t+1)} =  \beta_1 \,\mathbf{m}^{(t)} + 
(1-\beta_1)\,\widehat{\nabla}_{\boldsymbol{\theta}} \mathcal{L}\bigl(\boldsymbol{\theta}^{(t+1)}\bigr), \quad 
\mathbf{v}^{(t+1)} = \beta_2 \,\mathbf{v}^{(t)} + (1-\beta_2)\,\bigl(\widehat{\nabla}_{\boldsymbol{\theta}} \mathcal{L}(\boldsymbol{\theta}^{(t+1)})\bigr)^{2}, \label{eq:adamzeorhorder}
\end{align}
the update rule is $\boldsymbol{\theta}^{(t+1)} = \boldsymbol{\theta}^{(t)} - \alpha \,\frac{\mathbf{m}^{(t)}}{\sqrt{\mathbf{v}^{(t)}} + \delta}$, where $\beta_1,\beta_2\in(0,1)$ and $\delta>0$ is a regularization constant. This per-parameter adaptation remains valid in a ZO setting provided that the expected magnitude of $(\mathbf{v}^{(t)})^{-1/2}\,\widehat{\nabla}_{\boldsymbol{\theta}} \mathcal{L}(\boldsymbol{\theta}^{(t)})$ remains bounded. 

Integrating second-order free information has also been studied by several works \citep{zhao2025secondorder, bollapragada2023adaptive}, which have explored approximating second-order information without explicit differentiation. In the context of ZO methods, one seeks to recover curvature information through function evaluations rather than analytic derivatives. This is particularly beneficial when gradients or Hessians are either unavailable or too expensive to compute. A common strategy is to estimate Hessian actions via \emph{finite-difference approximations}. For example, a quasi-Newton scheme might estimate the product of the Hessian \(H^{(t)}\) with a vector \(\mathbf{z}\) by using the relation
\begin{align}
H^{(t)}\,\mathbf{z} \;\approx\;
\frac{\mathcal{L}\bigl(\boldsymbol{\theta}^{(t)} + h^{(t)}\,\mathbf{z} + h^{(t)}\,\mathbf{u}\bigr)
      \;-\;
      \mathcal{L}\bigl(\boldsymbol{\theta}^{(t)} + h^{(t)}\,\mathbf{u}\bigr)
     }{(h^{(t)})^{2}}, \notag
\end{align}
where \(\mathbf{u} \sim \mathcal{N}(\mathbf{0}, \mathbf{I}_P)\) serves as a random direction that, when combined with the directional perturbation \(\mathbf{z}\), yields an approximation of the second-order behavior along that direction. The key idea is that the change in the loss function, when evaluated at points that are slightly perturbed in both the desired and a random direction, encodes information about the curvature of the loss surface. This approach circumvents the need for explicit Hessian computation while still capturing essential curvature details. Subsequently, one can employ a BFGS update to build an approximation \(B^{(t)} \approx (H^{(t)})^{-1}\) of the inverse Hessian. This approximated inverse Hessian is then used to precondition the gradient, leading to a curvature-informed update of the form,
\begin{align}
\boldsymbol{\theta}^{(t+1)} = \boldsymbol{\theta}^{(t)} - \alpha^{(t)}\,B^{(t)}\,\widehat{\nabla}_{\boldsymbol{\theta}}\mathcal{L}(\boldsymbol{\theta}^{(t)}), \notag
\end{align}
where \(\widehat{\nabla}_{\boldsymbol{\theta}}\mathcal{L}(\boldsymbol{\theta}^{(t)})\) represents a ZO estimate of the gradient. By adjusting the descent direction based on the local curvature, this method is capable of accelerating convergence, especially in ill-conditioned landscapes where the loss surface may exhibit steep or flat regions. Although these methods are computationally heavier due to the additional function evaluations and the complexity of updating the inverse Hessian approximation, the integration of curvature information can lead to more robust and efficient optimization. In particular, the ability to modulate the step direction and length based on local curvature often results in improved convergence behavior, making these techniques an attractive alternative in settings where traditional gradient-based methods may struggle \citep{shu2023zerothorder}.
\paragraph{Summary and Outlook.}
In summary, zeroth-order methods offer a derivative-free alternative for optimizing NNs, relying solely on function evaluations instead of explicit gradient computations. This characteristic makes them particularly attractive in scenarios where gradients are unavailable, unreliable, or computationally expensive, such as in black-box optimization, reinforcement learning, and adversarial attacks. By approximating gradients through techniques like coordinate-wise finite differences or random directional derivatives, these methods enable optimization in non-differentiable settings, albeit at the cost of increased function evaluations and inherent dimension dependence. Recent advances, including variance reduction schemes such as ZO-SVRG and adaptive frameworks like ZO-AdaMM, have significantly improved the efficiency and stability of zeroth-order optimization. Furthermore, the incorporation of second-order information via finite-difference approximations and quasi-Newton updates has enhanced convergence in ill-conditioned landscapes by providing curvature-aware steps. While challenges remain—particularly regarding the high computational overhead in high-dimensional spaces—the ongoing development of efficient sampling, parallelization, and adaptive techniques holds promise for expanding the practical applicability of derivative-free methods in large-scale deep learning and beyond.
\subsection{First-Order Methods}
\label{subsec:first_order_methods}
First-order methods remain the most widely used optimization techniques in deep learning, as they rely on gradient information to iteratively update model parameters~\citep{goodfellow2016deep}. We consider a function $\mathcal{L}: \, \, \mathbb{R}^d \rightarrow \mathbb{R}$ to be minimized with respect to some parameters $\boldsymbol{\theta}$. In its simplest form, one solves
\begin{align}
    \min_{\boldsymbol{\theta} \in \mathbb{R}^d} \mathcal{L}(\boldsymbol{\theta}),
    \quad
    \mathbf{g}^{(t)} = \nabla \mathcal{L}(\boldsymbol{\theta})\in\mathbb{R}^d, \label{eq:minimizationproblem}
\end{align}
via the \emph{gradient descent} update defined in Eq. (\ref{eq:gradientdescentupdaterule}). Note that the formulation defined in Eq. (\ref{eq:minimizationproblem}) is similar to the one defined in Eq. (\ref{eq:minimizationformualtionzeroorder}), except that we now compute the actual gradient value instead of estimating it. Due to the large scale of modern datasets and models, practitioners typically rely on \emph{mini-batch} SGD where $\mathbf{g}^{(t)}$ is substituted by an unbiased gradient estimate computed on a random mini-batch of training data, i.e. $\widehat{\nabla}_{\boldsymbol{\theta}} \mathcal{L}(\boldsymbol{\theta}^{(t)})$ in Eq. (\ref{eq:gradientdescentupdaterule}). This stochastic approach not only reduces computation but also helps escape sharp minima and saddle points due to the induced noise the computed gradient~\citep{dauphin2015equilibrated, chenconvergence, j.2018on}.

\paragraph{SGD‐Based and Momentum Methods.}
A straightforward yet crucial refinement of SGD is to incorporate \emph{momentum}, a technique originally studied by \cite{polyak1964some} and subsequently refined for NNs \citep{sutskever2013importance}. In its classical form, one initializes \(\boldsymbol{\theta}_0\) and \(\mathbf{m}_0=\mathbf{0}\), then iterates
\begin{align}
    \mathbf{m}^{(t+1)}
    &\;=\;
    \beta\,\mathbf{m}^{(t)}
    \;-\;
    \mathbf{g}^{(t)}, \label{eq:momemtum}\\
    \boldsymbol{\theta}^{(t+1)}
    &\;=\;
    \boldsymbol{\theta}^{(t)}
    \;+\;
    \alpha^{(t)}\,\mathbf{m}^{(t+1)}, \label{eq:updaterulee}
\end{align}
where \(\mathbf{m}^{(t)}\in\mathbb{R}^d\) is a \emph{velocity} term, \(\beta\in(0,1)\) is the damping coefficient, and \(\alpha^{(t)}>0\) denotes the learning rate schedule. By accumulating gradients in \(\mathbf{m}^{(t)},\) momentum can significantly speed up convergence, particularly if \(\beta\) is well-chosen. However, if \(\beta\) is set too small, progress in low-curvature directions may stall, whereas a large \(\beta\) risks instabilities and oscillations in high-curvature regions.

\emph{Nesterov Accelerated Gradient} (NAG) \citep{nesterov1983method, sutskever2013importance} further refines classical momentum by evaluating the gradient at a look-ahead position, effectively correcting errors introduced by moving in the direction of the velocity in Eq. (\ref{eq:momemtum}). One writes
\begin{align}
    \mathbf{m}^{(t+1)}
    \;=\;
    \beta\,\mathbf{m}^{(t)}
    \;-\;
    \nabla_{\boldsymbol{\theta}} \mathcal{L}\bigl(\boldsymbol{\theta}^{(t)} \;+\; \alpha^{(t)}\,\beta\,\mathbf{m}^{(t)}\bigr), \notag
\end{align}
and apply $\boldsymbol{\theta}^{(t+1)} = \boldsymbol{\theta}^{(t)} + \mathbf{m}^{(t+1)}$. Nesterov’s modification helps stabilize updates in non-convex optimization and can accelerate convergence in quadratic settings by implicitly adapting to the curvature.

Beyond these foundational momentum schemes, numerous enhancements address high-dimensional or otherwise challenging NN landscapes. For example, Lion method~\citep{chen2024symbolic} incorporates a diagonal preconditioner into the momentum update. Formally, one can write the update as
\begin{align}
\mathbf{m}^{(t+1)} = \beta \mathbf{m}^{(t)} - (1-\beta) \, P \odot\mathbf{g}^{(t)}, \notag
\end{align}
where \(P\) is a diagonal matrix that rescales the gradient along each coordinate and applies Eq. (\ref{eq:updaterulee}). The intuition behind this approach is to adaptively adjust the learning rate in each dimension based on local curvature, thereby enhancing numerical conditioning and promoting a more balanced descent across the parameter space. GaLore method~\citep{zhao2024galore} adopts a complementary strategy by projecting the gradient into a reduced-dimensional subspace. Let \(U \in \mathbb{R}^{d \times k}\) be an orthonormal matrix whose columns form a basis for a subspace with \(k \ll d\). The projected gradient is then given by
\begin{align}
(\mathbf{g}^{(t)})^{\text{proj}} = U U^\top \mathbf{g}^{(t)}, \label{eq:lion}
\end{align}
and integrating Eq. (\ref{eq:lion}) into Eq. (\ref{eq:momemtum}), the momentum update becomes
\begin{align}
\mathbf{m}^{(t+1)} = \beta \mathbf{m}^{(t)} + (1-\beta) \, (\mathbf{g}^{(t)})^{\text{proj}}. \notag
\end{align}
This projection effectively retains the most critical descent directions while significantly reducing memory usage. The key concept is that even in a high-dimensional setting, the most informative directions for descent may lie on a low-dimensional manifold, thus allowing the method to discard redundant information without compromising convergence. Moreover, WIN method~\citep{zhou2022win} fuses weight decay with a Nesterov-like momentum framework. In this method, the velocity update is modified to incorporate weight regularization directly into the gradient accumulation. i.e.
\begin{align}
\mathbf{m}^{(t+1)} = \beta \mathbf{m}^{(t)} - (1-\beta) \left(\mathbf{g}^{(t)} + \lambda \boldsymbol{\theta}^{(t)} \right), \notag
\end{align}
where \(\lambda\) is the weight decay coefficient. The update rule is similar as Eq. (\ref{eq:updaterulee}). By integrating weight decay into the momentum term, WIN simultaneously enforces regularization and benefits from the anticipatory nature of Nesterov momentum, which results in smoother training trajectories and enhanced convergence behavior.

Despite their diverse implementations, these momentum-oriented methods share a consistent theme: the effective use of a velocity term, which is carefully maintained and updated, can significantly accelerate deep learning optimization. Crucially, the success of these approaches hinges on the appropriate tuning of damping factors and step sizes to balance acceleration with stability.
\paragraph{Adaptive Adam‐Type Approaches.}
While momentum methods amplify or dampen the \emph{direction} of the raw gradient signal, a complementary branch of first-order optimizers focuses on \emph{adaptively} scaling the \emph{magnitude} of each parameter’s update. Early pioneers in this category include Adagrad~\citep{JMLR:v12:duchi11a} and RMSProp~\citep{hinton2012neural}. Adagrad accumulates the squared gradients over time, thus adjusting the learning rate based on the cumulative variance of each parameter’s gradient. Formally, the parameter update then becomes
\begin{align}
    \boldsymbol{\theta}^{(t+1)}
    \;=\;
    \boldsymbol{\theta}^{(t)}
    \;-\;
    \alpha \,\frac{\mathbf{g}^{(t)}}{\sqrt{\sum_{\tau=1}^t (\mathbf{g}^{(\tau)})^2} \;+\;\epsilon}, \notag
\end{align}
where \(\epsilon>0\) avoids division by zero. Although Adagrad excels at handling sparse features by boosting the effective learning rate for infrequently updated parameters, its global learning rate can diminish excessively over time. RMSProp addresses this by introducing an exponential moving average of the squared gradients rather than a strict sum
\begin{align}
    \mathbf{v}^{(t+1)}
    \;=\;
    \beta \,\mathbf{v}^{(t)}
    \;+\;
    (1-\beta)\,(\mathbf{g}^{(t)})^2,
    \quad
    \boldsymbol{\theta}^{(t+1)}
    \;=\;
    \boldsymbol{\theta}^{(t)}
    \;-\;
    \alpha \,\frac{\mathbf{g}^{(t)}}{\sqrt{\mathbf{v}^{(t+1)}} \;+\;\epsilon}, \notag
\end{align}
where \(\beta\in(0,1)\). By continuously discounting older gradients, RMSProp better balances fast adaptation with sustained learning. Adam can be viewed as combining RMSProp’s second-moment tracking with a momentum-like first-moment term, i.e.
\begin{align}
    \mathbf{m}^{(t+1)}
    \;=\;
    \beta_1\,\mathbf{m}^{(t)}
    \;+\;
    (1-\beta_1)\,\mathbf{g}^{(t)},
    \quad
    \mathbf{v}^{(t+1)}
    \;=\;
    \beta_2\,\mathbf{v}^{(t)}
    \;+\;
    (1-\beta_2)\,\bigl(\mathbf{g}^{(t)}\bigr)^2, \label{eq:adam}
\end{align}
where \(\beta_1,\beta_2\in(0,1)\). Note that Eq. (\ref{eq:adamzeorhorder}) is the same as Eq. (\ref{eq:adam}) except that the gradient is estimated rather than explicitly computer. After correcting for initialization bias, the update becomes
\begin{align}
    \boldsymbol{\theta}^{(t+1)}
    \;=\;
    \boldsymbol{\theta}^{(t)}
    \;-\;
    \alpha\,
    \frac{
      \mathbf{m}^{(t+1)}/(1-\beta_1^{(t+1)})
    }{
      \sqrt{\,\mathbf{v}^{(t+1)}/(1-\beta_2^{(t+1)})\,} + \epsilon
    }. \notag
\end{align}
This per-parameter scaling often yields rapid convergence and robustness to HP misconfiguration, making Adam and its variants among the most widely adopted optimizers in deep learning. This per-parameter adaptation often yields faster convergence and greater numerical stability. Building on Adam's popularity, several adaptive methods have been proposed to mitigate specific challenges in large-scale NN optimization. AdaFactor addresses memory constraints by factorizing the second-moment estimates. For a matrix‐valued parameter \(\Theta^{(t)} \in \mathbb{R}^{m \times n}\) with gradient \(G^{(t)}\) (here $\Theta^{(t)} = \text{vec}^{-1}(\boldsymbol{\theta}^{(t)})$ and $G^{(t)} =  \text{vec}^{-1}(g^{(t)})$), instead of storing a full second‐moment tensor, AdaFactor computes two accumulators—a row vector \(\mathbf{r}^{(t)} \in \mathbb{R}^m\) and a column vector \(\mathbf{c}^{(t)} \in \mathbb{R}^n\)—updated as
\begin{align}
\mathbf{r}^{(t+1)}_i = \beta_2\,\mathbf{r}^{(t)}_i + (1-\beta_2) \sum_{j=1}^{n} \left(G_{ij}^{(t+1)}\right)^2, \quad \mathbf{c}^{(t+1)}_j = \beta_2\,\mathbf{c}^{(t)}_j + (1-\beta_2) \sum_{i=1}^{m} \left(G_{ij}^{(t+1)}\right)^2. \notag
\end{align}
The effective second moment for each entry is then approximated via a factorization of these statistics,
\begin{align}
V_{ij}^{(t+1)} = \frac{\mathbf{r}^{(t+1)}_i\,\mathbf{c}^{(t+1)}_j}{\sum_{j=1}^{n} \mathbf{c}^{(t+1)}_j}. \notag
\end{align}
Given this approximation, the update rule for the parameter matrix becomes
\begin{align}
\Theta^{(t+1)} = \Theta^{(t)} - \alpha \,\frac{G^{(t+1)}}{\sqrt{V^{(t+1)}} + \epsilon}.
\end{align}
This update mechanism allows AdaFactor to dramatically reduce memory overhead while still capturing essential second-order information through the factorized approximation, making it particularly attractive for large-scale models such as Transformers \citep{NIPS2017_3f5ee243}. To continue, RAdam aims to improve Adam’s instability during early training by rectifying the variance of the adaptive learning rate. It first computes a statistic
\begin{align}
\mathbf{\rho}^{(t)} = \mathbf{\rho}_\infty - \frac{2t\beta_2^{(t)}}{1-\beta_2^{(t)}}, \qquad \text{with} \qquad \mathbf{\rho}_\infty = \frac{2}{1-\beta_2} - 1, \notag
\end{align}
which indicates whether the variance of the second-moment estimate has stabilized. When \(\mathbf{\rho}^{(t)}\) is sufficiently large, a rectification factor \(\phi(\mathbf{\rho}^{(t)})\) is applied to the update
\begin{align}
\boldsymbol{\theta}^{(t+1)} = \boldsymbol{\theta}^{(t)} - \alpha \, \frac{\phi(\mathbf{\rho}^{(t+1)})\, \mathbf{m}^{(t+1)}}{\sqrt{\mathbf{v}^{(t+1)}} + \epsilon}, \notag
\end{align}
where \(\phi(\mathbf{\rho}^{(t)})\) is defined as
\begin{align}
\phi(\mathbf{\rho}^{(t)})=\begin{cases}
\sqrt{\frac{(\mathbf{\rho}^{(t)}-4)(\mathbf{\rho}^{(t)}-2)\mathbf{\rho}_\infty}{(\mathbf{\rho}_\infty-4)(\mathbf{\rho}_\infty-2)\mathbf{\rho}^{(t)}}}, & \text{if } \mathbf{\rho}^{(t)} > 4,\\[1ex] 
1, & \text{otherwise}.
\end{cases} \notag
\end{align}
This rectification ensures that, during the early stages of training, the update resembles SGD with momentum, and as the second-moment estimate stabilizes, the optimizer gradually transitions to a fully adaptive behavior. Moreover, AdaBound~\citep{luo2018adaptive} mitigates extreme early behavior by dynamically bounding the effective learning rate. In this method, the adaptive learning rate is clipped between two time-dependent bounds \(\alpha_{\mathrm{low}}(t)\) and \(\alpha_{\mathrm{high}}(t)\), leading to the modified step size
\begin{align}
\hat{\alpha}^{(t)} = \mathrm{clip}\left(\frac{\alpha}{\sqrt{\mathbf{v}^{(t)}} + \epsilon},\, \alpha_{\mathrm{low}}(t),\, \alpha_{\mathrm{high}}(t)\right), \notag
\end{align}
with the update rule given by $ \boldsymbol{\theta}^{(t+1)} = \boldsymbol{\theta}^{(t)} - \hat{\alpha}^{(t+1)} \, \mathbf{m}^{(t+1)}$. Here, \(\alpha\) denotes the base learning rate, which sets the overall scale of the update before the adaptive scaling and clipping are applied. As the bounds tighten over time, the algorithm increasingly resembles vanilla SGD, which can yield improved generalization in the later stages of training. Then, AdaBelief modifies Adam’s second-moment estimation by tracking the deviation of the observed gradient from its exponential moving average. The second-moment accumulator is updated as
\begin{align}
\mathbf{v}^{(t+1)} = \beta_2 \mathbf{v}^{(t)} + (1-\beta_2)(\mathbf{g}^{(t+1)} - \mathbf{m}^{(t+1)})^2. \notag
\end{align}
The parameter update follows the familiar form $\boldsymbol{\theta}^{(t+1)} = \boldsymbol{\theta}^{(t)} - \alpha \, \frac{\mathbf{m}^{(t+1)}}{\sqrt{\mathbf{v}^{(t+1)}} + \epsilon}$.
By emphasizing the ''belief'' in the current gradient direction (i.e., penalizing unexpected deviations), AdaBelief tends to yield a more stable and reliable convergence. AdamW~\citep{loshchilov2018decoupled} resolves a conceptual inconsistency in Adam by decoupling weight decay from the gradient-based update. Instead of mixing L2 regularization into the gradient, AdamW applies weight decay as a separate multiplicative factor
\begin{align}
\boldsymbol{\theta}^{(t+1)} = \boldsymbol{\theta}^{(t)} (1-\lambda) - \alpha \, \frac{\mathbf{m}^{(t+1)}}{\sqrt{\mathbf{v}^{(t+1)}} + \epsilon}. \notag
\end{align}
where $\lambda$ is the weight decay HP. This decoupling clarifies the role of regularization and has shown empirical benefits, especially in large-scale language models. Furthermore, NaDam \citep{dozat2016incorporating} is a variant of the Adam optimizer that integrates Nesterov momentum into the adaptive framework. In standard momentum-based methods, the momentum term $\mathbf{m}^{(t+1)}$, is updated as in Eq. (\ref{eq:adam}). In NaDam, rather than using the momentum term alone to update the parameters, the update rule incorporates a weighted combination of both the momentum and the current gradient
\begin{align}
\boldsymbol{\theta}^{(t+1)} = \boldsymbol{\theta}^{(t)} - \alpha \left( \beta_1 \mathbf{m}^{(t+1)} + (1-\beta_1) \mathbf{g}^{(t+1)} \right). \notag
\end{align}
This formulation mimics the essence of NAG where NaDam approximates this look-ahead evaluation. The intuition behind this approach is that the weighted sum $ \beta_1 \mathbf{m}^{(t+1)} + (1-\beta_1) \mathbf{g}^{(t+1)}$ serves as an effective surrogate for the gradient computed at a future point. This anticipatory update direction can lead to a more informed descent step, potentially improving convergence rates by better aligning the parameter update with the long-term trajectory of the optimization. AdaInject~\citep{dubey2022adainject} represents a further evolution by incorporating approximations of second-order information into an Adam-like update. In addition to the standard first and second-moment estimates, AdaInject computes a curvature-aware term \(\mathbf{s}^{(t)}\) that modulates the adaptive learning rate
\begin{align}
\boldsymbol{\theta}^{(t+1)} = \boldsymbol{\theta}^{(t)} - \alpha \, \frac{\mathbf{m}^{(t+1)}}{\sqrt{\mathbf{v}^{(t+1)} + \gamma \mathbf{s}^{(t+1)}} + \epsilon}, \notag
\end{align}
where \(\gamma\) governs the influence of the injected second-order statistics. This integration of partial curvature information aims to combine the simplicity of first-order methods with the robustness of second-order approaches, potentially leading to more effective convergence.
\paragraph{Sharpness‐Aware Extensions.}
While most first‐order methods emphasize convergence speed or adaptive scaling, sharpness‐aware approaches directly target the geometry of the loss surface by controlling the \emph{sharpness} of local minima. SAM reformulates the standard objective into a minimax problem defined in Eq. (\ref{eq:samequation}). This formulation forces the optimization to favor parameter regions where the loss remains stable under small adversarial perturbations, effectively promoting flatter minima that are empirically associated with better generalization. In practice, SAM first computes a tentative update by taking a gradient step from \(\theta\) and then re-evaluates the loss in a local neighborhood around this point, ensuring that the descent direction leads to a region of the loss surface with reduced sharpness.

Several variants extend SAM by integrating adaptive scaling or refined geometric insights. AdaSAM~\citep{sun2024adasam} merges the neighborhood‐based minimax framework of SAM with Adam’s per‐parameter scaling, yielding an approach that leverages local curvature information while retaining the benefits of adaptive learning rates. ASAM~\citep{sun2024adasam} further refines the method by dynamically adjusting the perturbation radius \(\rho\) based on the magnitudes of the gradients, thereby fine‐tuning the sensitivity of the sharpness measure to the local loss landscape. Meanwhile, GSAM~\citep{zhuang2022surrogate} incorporates the top eigenvalue of the local Hessian into the sharpness estimation, offering a more detailed account of the loss curvature. Collectively, these sharpness‐aware extensions demonstrate that by incorporating geometric considerations into the optimization process, purely first‐order methods can be enhanced to locate flatter regions of the loss surface, potentially leading to models with superior generalization.

\paragraph{Summary and Outlook.}
In summary, first-order optimizers in deep learning branch into several intertwined families. Momentum-based methods augment raw SGD to accelerate convergence and reduce erratic updates, while adaptive Adam-type approaches further modulate step sizes on a per-parameter basis, often leading to faster training and improved stability. More recently, sharpness-aware optimizers use local curvature information to search for flatter minima, reflecting broader interest in explaining why certain solutions generalize better. Despite their diversity, these methods all share a reliance on gradient information, underscoring the central role of first-order techniques in modern deep learning optimization. 
\subsection{Second-Order Methods}
\label{subsec:second_order_methods}
Second-order methods incorporate curvature information into the optimization process by leveraging the Hessian matrix or its approximations. This additional information enables more informed parameter updates compared to first-order techniques, often leading to faster convergence and improved performance, especially in regions where the loss landscape exhibits significant curvature. Consider an objective function \(\mathcal{L}: \mathbb{R}^d \rightarrow \mathbb{R}\) that we wish to minimize with respect to the parameters \(\boldsymbol{\theta}\). The classical formulation of the problem is similar to Eq. (\ref{eq:minimizationproblem}), but with the integration of second order information, is defined as 
\begin{align}
    \min_{\boldsymbol{\theta} \in \mathbb{R}^d} \mathcal{L}(\boldsymbol{\theta}),
    \quad
    \mathbf{g}^{(t)} = \nabla_{\boldsymbol{\theta}} \mathcal{L}(\boldsymbol{\theta})\in\mathbb{R}^d, \quad H(\boldsymbol{\theta}) = \nabla_{\boldsymbol{\theta}}^2 \mathcal{L}(\boldsymbol{\theta}). \label{eq:minimizationproblemsecond}
\end{align}
The prototypical second-order method is Newton-Raphson’s method \citep{Dedieu2015}, which updates the parameters according to
\begin{align}
\boldsymbol{\theta}^{(t+1)} = \boldsymbol{\theta}^{(t)} - H(\boldsymbol{\theta}^{(t)})^{-1}\,\mathbf{g}(\boldsymbol{\theta}^{(t)}). \label{eq:updatesecondorderinfo}
\end{align}
While Newton’s method can achieve rapid convergence, the direct computation and inversion of the Hessian matrix \(H(\boldsymbol{\theta}^{(t)})\) becomes prohibitively expensive in high-dimensional settings, such as DNNs. To overcome these limitations, a variety of approximations and alternative strategies have been developed. These methods aim to retain essential curvature information while alleviating the computational burden associated with Hessian inversion. In what follows, we categorize these second-order methods based on the techniques they employ—ranging from quasi-Newton approximations and Hessian-free methods to structured or low-rank representations of the Hessian. This taxonomy highlights the trade-offs between computational efficiency and the fidelity of curvature information, paving the way for more robust and scalable optimization strategies in large-scale deep learning.
\paragraph{Hessian-Based Diagonal and Quasi-Newton Methods.}
Many second-order optimizers aim to capture curvature information without incurring the full cost of computing and inverting the Hessian \(H(\boldsymbol{\theta})\). A common strategy is to approximate the Hessian by either its diagonal or a structured factorization, thereby striking a balance between extracting useful curvature cues and maintaining computational efficiency. For instance, methods such as AdaHessian~\citep{yao2021adahessian} and Apollo~\citep{ma2020apollo} focus on approximating the diagonal elements of the Hessian. In these approaches, for each parameter \(\theta_i\) a scalar curvature estimate is obtained via
\begin{align}
D_i \approx \frac{\partial^2 \mathcal{L}(\boldsymbol{\theta})}{\partial \theta_i^2}, \notag
\end{align}
which is then used to rescale the gradient in the update
\begin{align}
\theta_i^{(t+1)} = \theta_i^{(t)} - \alpha \, \frac{m_i^{(t)}}{\sqrt{D_i^{(t)}} + \epsilon}. \notag
\end{align}
This rescaling captures directional curvature, allowing the optimizer to take larger steps in flat regions and more cautious updates in steep directions, thereby improving convergence behavior. However, while diagonal approximations are computationally efficient compared to the full Hessian, they inherently neglect cross-parameter interactions that can be critical in highly non-isotropic loss landscapes.

Other methods, such as Sophia~\citep{liu2024sophia} and INNA-Prop~\citep{bolte2024second}, also leverage diagonal Hessian approximations or use incremental updates to refine the gradient-based steps. These techniques are often embedded within Adam-like frameworks, where the adaptive scaling is enriched by curvature estimates, thus enabling more nuanced adjustments that better reflect the underlying loss surface geometry. In contrast, M-FAC focuses on efficiently computing inverse Hessian-vector products. Instead of explicitly forming or inverting \(H(\boldsymbol{\theta})\), M-FAC approximates its action on a vector by representing the Hessian as a low-rank or structured factorization,
\begin{align}
H(\boldsymbol{\theta}) \approx Q \Lambda Q^\top, \notag
\end{align}
where \(Q\) contains dominant eigenvectors and \(\Lambda\) is a diagonal matrix of eigenvalues. The inverse Hessian is then approximated by
\begin{align}
H^{-1}(\boldsymbol{\theta}) \approx Q \Lambda^{-1} Q^\top, \notag
\end{align}
which is used in the update defined in Eq. (\ref{eq:updatesecondorderinfo}). This approach allows the optimizer to incorporate global curvature information while avoiding the full cost of Hessian inversion.

Classic quasi-Newton methods, such as L-BFGS, remain influential by iteratively constructing an approximation of the inverse Hessian from successive gradient differences. In L-BFGS, one collects updates \(\mathbf{s}^{(t)} = \boldsymbol{\theta}^{(t+1)} - \boldsymbol{\theta}^{(t)}\) and \(\mathbf{y}^{(t)} = \mathbf{g}^{(t+1)} - \mathbf{g}^{(t)}\) to update the inverse Hessian estimate \(B^{(t)}\) through formulas such as
\begin{align}
B^{(t+1)} = (V^{(t)})^\top B^{(t)} V^{(t)} + \boldsymbol{\rho}^{(t)} \mathbf{s}^{(t)} (\mathbf{s}^{(t)})^\top, \label{eq:quasinewton}
\end{align}
with \(\boldsymbol{\rho}^{(t)} = \frac{1}{(\mathbf{y}^{(t)})^\top \mathbf{s}^{(t)}}\) and \(V^{(t)} = I - \boldsymbol{\rho}^{(t)} \mathbf{y}^{(t)} (\mathbf{s}^{(t)})^\top\). This quasi-Newton strategy efficiently captures curvature information from recent iterations, but its performance in deep learning is often limited by the high dimensionality of the parameter space and the stochasticity introduced by mini-batch training.

The central intuition behind these Hessian-based methods is to allow the optimizer to adjust step sizes based on the local geometry of the loss landscape. By rescaling the gradient using curvature information, these methods enable larger updates in flat regions and more conservative moves in steep directions, which can lead to faster convergence and improved stability. Nevertheless, challenges remain. Diagonal approximations may oversimplify the curvature by ignoring off-diagonal interactions, while low-rank or quasi-Newton approximations can be sensitive to noise and may incur non-negligible computational overhead. Balancing the fidelity of curvature information with computational tractability is thus a critical design consideration in the development of effective second-order methods for deep learning.
\paragraph{Generalized Gauss-Newton (GGN) and Hybrid Approaches.}
A further category of second-order techniques leverages the GGN matrix or other surrogates for the Hessian, providing a more tractable means of incorporating curvature information into the optimization process. In many NN settings, the Hessian is both expensive to compute and can be indefinite. The GGN offers an attractive alternative by approximating the Hessian using the model’s Jacobian \(J(\boldsymbol{\theta})\) and the Hessian of the loss with respect to the network outputs. Specifically, the GGN is defined as
\begin{align}
G(\boldsymbol{\theta}) = J(\boldsymbol{\theta})^\top\,\nabla_{f(\boldsymbol{\theta})}^2 \mathcal{L}\bigl(f(\boldsymbol{\theta})\bigr)\,J(\boldsymbol{\theta}), \notag
\end{align}
For many losses that are locally quadratic, such as least-squares losses, this Hessian is \emph{positive semidefinite}, ensuring that the curvature information provided by \(G(\boldsymbol{\theta})\) is both meaningful and stable.

Building on this idea, methods such as SOFO~\citep{yu2024secondorder} approximate curvature via the GGN to precondition the gradient efficiently. For instance, one might update parameters as
\begin{align}
\boldsymbol{\theta}^{(t+1)} = \boldsymbol{\theta}^{(t)} - \alpha\,\left(G(\boldsymbol{\theta}^{(t)}) + \lambda I\right)^{-1}\,\mathbf{g}(\boldsymbol{\theta}^{(t)}), \notag
\end{align}
where \(\lambda>0\) is a damping factor that ensures numerical stability and \(I\) is the identity matrix. This update captures essential curvature while sidestepping the prohibitive cost of a full Hessian inversion.

Hybrid approaches further integrate second-order insights into first-order frameworks to enhance both convergence speed and generalization. LocoProp~\citep{amid2022locoprop} exemplifies this strategy by introducing a local second-order correction within a predominantly first-order update scheme. Rather than computing a full curvature matrix, LocoProp performs a localized approximation around the current iterate, adjusting the gradient step based on this curvature estimate. Similarly, FOSI~\citep{sivan2024fosi} combines standard gradient information with a curvature correction term. A representative update rule in such hybrid methods is
\begin{align}
\boldsymbol{\theta}^{(t+1)} = \boldsymbol{\theta}^{(t)} - \alpha\left(\mathbf{g}(\boldsymbol{\theta}^{(t)}) + \beta\,\Delta\boldsymbol{\theta}_{\text{curv}}\right), \notag
\end{align}
where \(\Delta\boldsymbol{\theta}_{\text{curv}}\) embodies the second-order correction derived from partial curvature information and \(\beta\) modulates its influence. This blend of first- and second-order cues enables the optimizer to make more informed updates without incurring the full cost of second-order methods.

The intuition behind these GGN and hybrid approaches is to harness curvature information to achieve a more adaptive step size—allowing for larger steps in flat regions and more cautious moves in steep directions—thus accelerating convergence and potentially improving generalization. Nevertheless, challenges remain. While the GGN matrix is more computationally tractable than the full Hessian, its computation still incurs significant overhead in very large models, and the quality of the curvature approximation can be sensitive to the network architecture and the loss function. Moreover, hybrid methods require a delicate balance between first- and second-order contributions; an overly aggressive curvature correction may lead to instability, whereas insufficient correction may fail to realize the benefits of second-order information. Overall, these methods reflect a growing trend towards embedding partial curvature insights into existing optimization frameworks, striving for an optimal trade-off between computational cost and the accuracy of curvature information.
\paragraph{K-FAC and Its Extensions.}

\emph{Natural gradient descent} (NGD), introduced by Amari \cite{Amari2000MethodsOI}, leverages the FIM as a \emph{Riemannian metric} to perform parameter updates that are invariant to reparameterization, thereby often leading to more efficient convergence than standard gradient descent. In this framework, the update is given by
\begin{align}
\boldsymbol{\theta}^{(t+1)} = \boldsymbol{\theta}^{(t)} - \alpha\, (F^{(t)})^{-1} \mathbf{g}^{(t)},\label{eq:FIMeuqation}
\end{align}
where \(F\) is the FIM. However, directly computing and inverting \(F\) is generally intractable for high-dimensional models. To address this issue, K-FAC  approximates the FIM by exploiting the layer-wise structure of deep networks. Specifically, using Eq. (\ref{eq:FIMequation}) the FIM is given by
\begin{align}
F = \sum_{n=1}^{N} \mathbb{E}_{\mathbf{y} \sim p(\mathbf{y}|f_{\boldsymbol{\theta}}(\mathbf{x}_n))} \left[ \nabla_{\boldsymbol{\theta}}\log p_{\boldsymbol{\theta}}(\mathbf{y}|\mathbf{x}_n)\nabla_{\boldsymbol{\theta}}\log p_{\boldsymbol{\theta}}(\mathbf{y}|\mathbf{x}_n)^\top\right] = \mathbb{E}\left[\nabla_{\boldsymbol{\theta}}\mathcal{L}\, (\nabla_{\boldsymbol{\theta}}\mathcal{L})^\top\right] = \mathbb{E}[\mathbf{g}\mathbf{g}^\top], \notag
\end{align}
where \(F\) quantifies the expected information that an observable \(\mathbf{y}\) conveys about the parameter \(\boldsymbol{\theta}\). For simplicity, we write \(\mathbb{E}\) in place of the full expectation \(\mathbb{E}_{\mathbf{y} \sim p(\mathbf{y}|f_{\boldsymbol{\theta}}(\mathbf{x}_n))}\). K-FAC further simplifies FIM computation by adopting a block-diagonal approximation tailored to deep neural networks. By decomposing \(F\) into layer-specific blocks and employing a Kronecker-factorization based on the identity \(\text{vec}(uv^\top) = v \otimes u\), K-FAC efficiently approximates each block, thereby enabling scalable natural gradient updates. K-FAC capitalizes on the structure of DNNs by simplifying the computation of the FIM via a block-diagonal approximation. In this framework, the full FIM is viewed as a block matrix with \(L \times L\) blocks (for a network with \(L\) layers). The \((i,j)\)th block is defined by
\begin{align}
F_{i,j} = \mathbb{E}\!\left[\text{vec}(g_i)\,\text{vec}(g_j)^\top\right],\notag
\end{align}
where \(g_i\) denotes the gradient with respect to the weights of the \(i^{\text{th}}\) layer. Using the Kronecker-vectorization identity, \(\text{vec}(\mathbf{u}\mathbf{v}^\top) = \mathbf{v} \otimes \mathbf{u}\) (see, e.g., \citep{Petersen2008}), we can express the gradient for layer \(i\) as
\begin{align}
g_i = \mathbf{s}_{i}\,\bar{\mathbf{h}}_{i-1}^\top \quad \Longrightarrow \quad \text{vec}(g_i) = \bar{\mathbf{h}}_{i-1} \otimes \mathbf{s}_{i},
\end{align}
where \(\bar{\mathbf{h}}_{i-1}\) represents the activations (augmented with a bias term) from the previous layer, and \(\mathbf{s}_{i}\) represents the sensitivity or pre-activation derivatives at layer \(i\).

Segmenting the FIM into layer-specific blocks, we have
\begin{align}
\hat{F}_{i,j} = \mathbb{E}\!\left[\text{vec}(g_i)\,\text{vec}(g_j)^\top\right] = \mathbb{E}\!\left[\bar{\mathbf{h}}_{i-1}\bar{\mathbf{h}}_{j-1}^\top \otimes \mathbf{s}_{i}\mathbf{s}_{j}^\top\right]. \notag
\end{align}
Under the assumption that the activations \(\bar{\mathbf{h}}_{i-1}\) and the pre-activation derivatives \(\mathbf{s}_{i}\) are statistically independent, this expectation can be approximated as the Kronecker product of individual expectations
\begin{align}
\hat{F}_{i,j} \approx \mathbb{E}\!\left[\bar{\mathbf{h}}_{i-1}\bar{\mathbf{h}}_{j-1}^\top\right] \otimes \mathbb{E}\!\left[\mathbf{s}_{i}\mathbf{s}_{j}^\top\right] \;=\; \mathcal{H}_{i-1,j-1} \otimes \mathcal{S}_{i,j}, \notag
\end{align}

Where $\mathcal{H}_{i-1,j-1} = \mathbb{E}\!\left[\bar{\mathbf{h}}_{i-1}\bar{\mathbf{h}}_{j-1}^\top\right]$ and $ \mathcal{S}_{i,j} =  \mathbb{E}\!\left[\mathbf{s}_{i}\mathbf{s}_{j}^\top\right]$ are the Kronecker Factors (KF). In practice, a further simplification is often made by assuming that the weight derivatives for different layers are uncorrelated. This leads to a block-diagonal approximation of the FIM, denoted as the Empirical FIM (EFIM)
\begin{align}
\hat{F} = \text{diag}\left(\hat{F}_{1,1}, \ldots, \hat{F}_{L,L}\right) = \begin{bmatrix}
    \hat{F}_{1}& & \\
    &\hat{F}_{2}&\\
    & &\ddots &\\
    &&&\hat{F}_{L}
\end{bmatrix}. \notag
\end{align}
This Kronecker-factorization significantly reduces the computational burden: instead of inverting a huge matrix of size proportional to the total number of parameters, one only needs to invert the much smaller matrices \(\mathcal{H}_{i-1,i-1}\) and \(\mathcal{S}_{i,i}\) for each layer. The resulting efficient inversion is based on the property $\left(A \otimes B\right)^{-1} = A^{-1} \otimes B^{-1}$. Consequently, for a given layer $i$, the parameter update rule of K-FAC is expressed as
\begin{align}
\theta_i^{(t+1)} = \theta_i^{(t)} - \alpha\, \left((\mathcal{H}_{i-1,i-1}^{(t)})^{-1} \otimes (\mathcal{S}_{i,i}^{(t)})^{-1}\right) \nabla_{\theta_i} \mathcal{L}(\theta_i^{(t)}). \notag
\end{align}
Figure~\ref{fig:background_saved} illustrates the computation of the EFIM via K-FAC for a given layer \(i\).
\begin{figure}[!ht]
    \centering
    \includegraphics[width=0.7\linewidth]{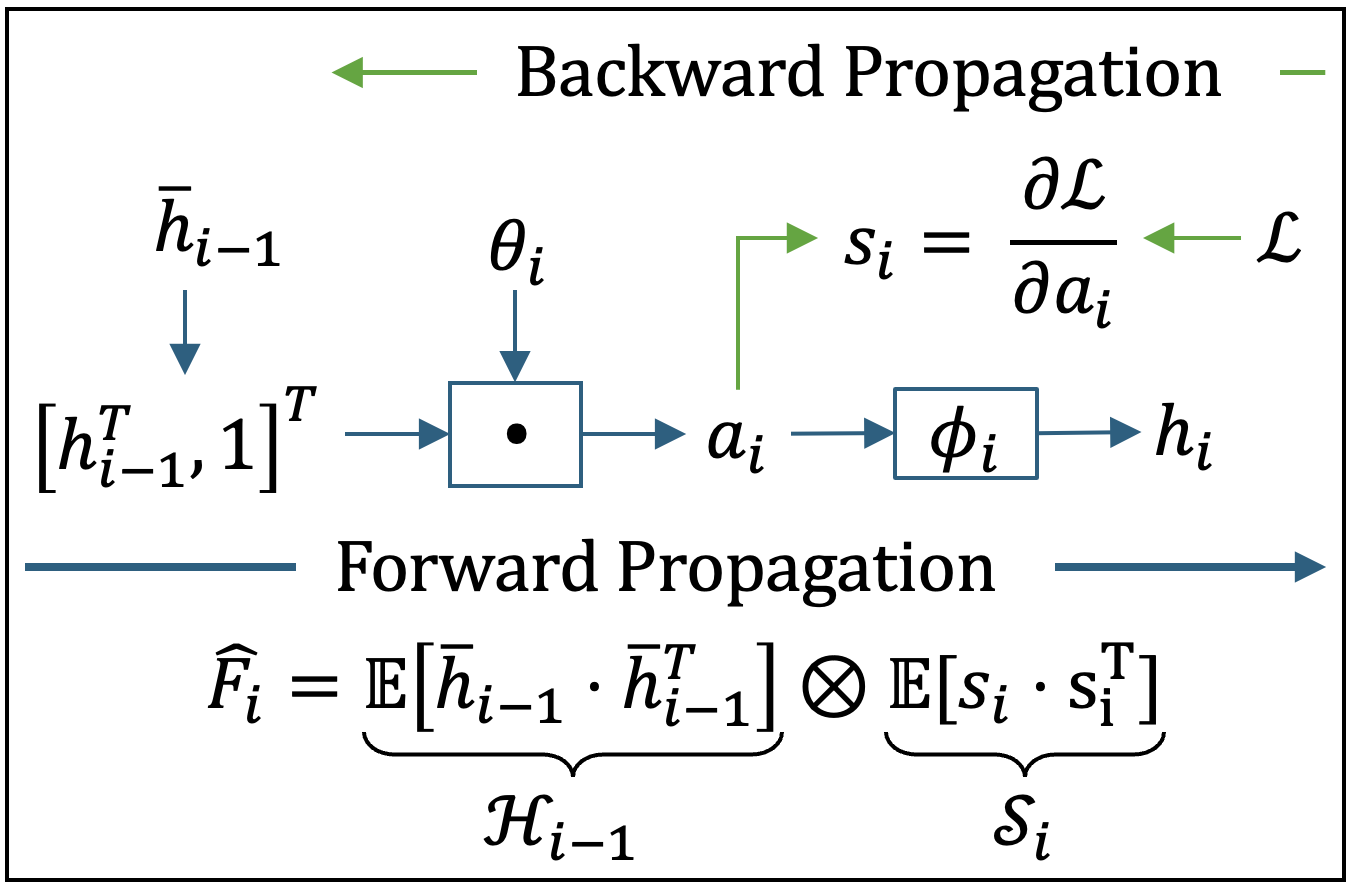}
    \caption{Illustration of EFIM computation using K-FAC for a given layer $i$.}
    \label{fig:background_saved}
\end{figure}

Building upon the foundational K-FAC framework, several extensions have been proposed to address its computational and approximation challenges. For example, S-KFAC~\citep{9578481} introduces strategies to further reduce computational overhead by incorporating sparsity and more aggressive damping, modifying the update to
\begin{align}
\theta_i^{(t+1)} = \theta_i^{(t)} - \alpha\, \left((\mathcal{H}_{i-1,i-1}^{(t)} + \lambda I)^{-1} \otimes (\mathcal{S}_{i,i}^{(t)} + \lambda I)^{-1}\right) \nabla_{\theta_i} \mathcal{L}( \theta_i^{(t)}). \notag
\end{align}
where \(\lambda\) is a damping parameter that enhances numerical stability. Eva~\citep{zhang2023eva} refines the estimation of the covariance matrices \(\mathcal{H}\) and \(\mathcal{S}\) by using adaptive averaging techniques, which leads to a more accurate curvature approximation and, consequently, to faster convergence and improved generalization.

Other methods, such as FOOF and MKOR~\citep{mozaffari2023mkor}, adopt a block-diagonal strategy that partitions the parameter space into smaller, more manageable sub-blocks, thereby reducing memory requirements and computational complexity while still capturing essential second-order information within each block. Quasi-Newton extensions such as K-BFGS~\citep{goldfarb2020practical} and KFRA~\citep{botev2017practical} integrate the Kronecker-factorized approach with iterative updates reminiscent of the BFGS algorithm. In these methods, one maintains an estimate of the inverse curvature matrix \(B^{(t)}\) and updates it using the standard quasi-Newton formula defined in Eq. (\ref{eq:quasinewton}). The novelty in K-BFGS and KFRA lies in organizing the updates so that the Kronecker structure is preserved, which allows these methods to capture curvature more precisely while still benefiting from the efficiency of factorization. EKFAC~\citep{george2018fast} further extends the K-FAC paradigm by incorporating an eigenbasis decomposition of the covariance matrices:
\begin{align}
\mathcal{H} = Q_\mathcal{H} \Lambda_\mathcal{H} Q_\mathcal{H}^\top \quad \text{and} \quad \mathcal{S} = Q_\mathcal{S} \Lambda_\mathcal{S} Q_\mathcal{S}^\top. \notag
\end{align}
In this framework, the approximated FIM is expressed as
\[
\hat{F} \approx (Q_\mathcal{H} \otimes Q_\mathcal{S})(\Lambda_\mathcal{H} \otimes \Lambda_\mathcal{S})(Q_\mathcal{H} \otimes Q_\mathcal{S})^\top,
\]
and its inversion is straightforward since the Kronecker product of the eigenvalue matrices is diagonal. This eigenbasis approach offers a diagonalized structure that balances fidelity to the true Hessian with the need for scalability in large models.

Collectively, these K-FAC extensions and related methods embody a common goal: to efficiently incorporate curvature information into the training of DNNS. By leveraging the inherent layer-wise structure and exploiting the properties of the Kronecker product, they enable second-order updates that are both computationally tractable and effective at capturing the nuanced geometry of the loss landscape. Despite their advances, challenges remain in terms of the assumptions required (e.g., independence of activations and pre-activation derivatives) and the computational overhead involved in maintaining and updating the necessary covariance estimates. Nevertheless, these methods continue to evolve, offering promising directions for scalable second-order optimization in deep learning.
\paragraph{Shampoo-Based Approaches.}
Shampoo is an influential second-order optimizer that constructs layer-wise preconditioners by leveraging the tensor structure of weight matrices. Unlike methods that restrict themselves to diagonal approximations of the curvature, Shampoo collects full second-moment statistics along different dimensions of a weight tensor. Shampoo computes two second-moment matrices,
\begin{align}
G_r = \sum_{t} \nabla_{\boldsymbol{\theta}} \mathcal{L}(\boldsymbol{\theta}^{(t)})\, \nabla_{\boldsymbol{\theta}} \mathcal{L}(\boldsymbol{\theta}^{(t)})^\top \quad \text{and} \quad G_c = \sum_{t} \nabla_{\boldsymbol{\theta}} \mathcal{L}(\boldsymbol{\theta}^{(t)})^\top\, \nabla_{\boldsymbol{\theta}} \mathcal{L}(\boldsymbol{\theta}^{(t)}), \notag
\end{align}
which capture the curvature information along the row and column dimensions, respectively. The core idea is to form a preconditioner as a Kronecker product of the inverses of fractional powers of these matrices, leveraging the similar property that K-FAC leverages, i.e. $\left(G_r \otimes G_c\right)^{-1} = G_r^{-1} \otimes G_c^{-1}$. In practice, Shampoo typically employs a fractional exponent (often \(-\tfrac{1}{4}\)) for numerical stability, leading to the update rule
\begin{align}
\boldsymbol{\theta}^{(t+1)} = \boldsymbol{\theta}^{(t)} - \alpha\, G_r^{-1/4} \, \nabla_{\boldsymbol{\theta}} \mathcal{L}(\boldsymbol{\theta}^{(t)}) \, G_c^{-1/4}, \notag
\end{align}
This update effectively preconditions the gradient by incorporating a full-matrix approximation of the curvature, thereby allowing for more adaptive step sizes that account for the local geometry of the loss surface.

Building on the Shampoo framework, subsequent approaches have sought to integrate additional first-order insights and further reduce computational overhead. For instance, SOAP~\citep{vyas2025soap} merges Shampoo’s second-order preconditioning with adaptive mechanisms similar to those found in Adam, blending the robustness of curvature information with the flexibility of adaptive learning rates. Meanwhile, a 4-bit variant of Shampoo~\citep{wang2024bit} addresses memory limitations by quantizing both the parameters and the preconditioners. This quantization reduces the memory footprint dramatically, making it feasible to apply second-order updates to very large models where storage and computational resources are at a premium.

The key intuition behind Shampoo and its extensions is to capture richer, multidimensional curvature information than what is available through simple diagonal approximations. By preconditioning the gradient with matrices that reflect the full second-moment structure along different dimensions, these methods enable more nuanced adjustments: taking larger steps in flatter directions and smaller, more cautious steps in regions of high curvature. Although this enhanced curvature awareness can lead to faster convergence and better adaptation to complex loss landscapes, it comes at the cost of increased computational and memory requirements. Extensions like SOAP and the 4-bit variant are motivated by the need to mitigate these challenges, offering scalable implementations that make the advantages of second-order information accessible even in massive deep learning models.
\paragraph{Summary and Outlook.}
Overall, second-order methods aim to utilize curvature information—via the Hessian, Fisher Information Matrix, or generalized approximations—to guide more informed and often more stable parameter updates. Approaches like K-FAC and Shampoo factorize large matrices to preserve crucial curvature cues without incurring prohibitive computational costs. Meanwhile, diagonal or block-diagonal Hessian approximations enable partial second-order adaptation that is more practical for massive models. Although these techniques have evolved significantly, challenges remain in scaling them to the largest networks and reconciling second-order insights with the stochasticity inherent in modern deep learning. Ongoing research continues to push the boundaries, as evidenced by the growing variety of second-order and hybrid optimizers, each offering a unique balance between computational feasibility and the richness of curvature information.
\section{Emerging Directions}
\label{sec:emerging}
The field of deep network optimization continues to evolve rapidly, with novel methods emerging to tackle both longstanding theoretical challenges and pressing practical constraints. Recent work has pushed the boundaries in several key areas, opening up exciting new avenues for research and application.
\paragraph{Integration of Curvature and Adaptive Strategies.}
A prominent trend in recent research is the seamless blending of adaptive first-order techniques with more refined curvature approximations. New hybrid methods, such as INNA-Prop and FOSI, aim to combine the low computational overhead and robustness of first-order updates with the enhanced directionality offered by partial second-order information. These methods typically incorporate curvature-aware preconditioners into adaptive frameworks, effectively adjusting the step sizes based not only on gradient magnitudes but also on local curvature estimates. The overarching goal is to achieve a more balanced optimization process that is both efficient and sensitive to the underlying loss landscape. While early results are promising, challenges remain in accurately estimating curvature in noisy, high-dimensional settings and in determining the optimal balance between adaptive scaling and curvature corrections.
\paragraph{Memory-Efficient Architectures.}
As NN models continue to grow in size, memory efficiency becomes an increasingly critical factor. Recent innovations such as 4-bit Shampoo and GaLore exemplify this trend by employing quantization and low-dimensional projection techniques to compress both parameters and the auxiliary statistics required for second-order updates. These methods make it feasible to deploy complex second-order optimizers even in resource-constrained environments or when dealing with models that contain hundreds of millions or even billions of parameters. By reducing the memory footprint without severely compromising the quality of curvature estimates, such techniques promise to bridge the gap between theoretical advances and practical deployment on modern hardware accelerators.
\paragraph{Robustness and Generalization in the Context of Large Models.}
Generalization remains a central challenge as models become larger and more complex. Insights from studies on the loss landscape, including works like \cite{chiang2023loss}, \cite{fort2019goldilocks}, and \cite{vysogorets2024deconstructing}, have underscored the importance of flat minima in achieving robust generalization. Building on these insights, methods such as SAM and its derivatives (e.g., AdaSAM) are being refined to not only accelerate convergence but also to guide the optimizer towards regions of the parameter space that generalize well. Recent research has further explored how phase transitions in deep learning dynamics (as discussed in \cite{kalra2023phase}) might be harnessed to design optimizers that are inherently more robust to overfitting while still maintaining computational efficiency.
\paragraph{Optimization for Fine-Tuning Large Language Models.}
The rapid rise of Large Language Models (LLM) has introduced unique optimization challenges, particularly in the fine-tuning stage. New methods, such as HiZOO, are designed to leverage Hessian information to capture fine-grained local curvature properties. This is crucial for stable fine-tuning, where the risk of drastic parameter changes can lead to overfitting or catastrophic forgetting. These approaches often integrate second-order insights into scalable frameworks that can handle the massive parameter counts typical of transformer architectures. Future research in this area is expected to further refine these techniques, ensuring that fine-tuning processes are both robust and efficient.
\paragraph{Theoretical Advances and Practical Algorithms.}
There is a growing effort to reconcile theoretical optimization guarantees with the practical requirements of deep learning. Recent studies have investigated the implicit regularization effects of SGD~\cite{frankle2018the} and the role of curvature in shaping generalization properties~\cite{xie2023on}. These investigations are guiding the development of new algorithms that seek to balance exploration—via stochasticity and adaptive updates—with exploitation—through careful curvature estimation. By grounding practical algorithms in strong theoretical foundations, researchers aim to develop optimizers that not only converge quickly but also offer provable guarantees on generalization performance.
\paragraph{Summary and Outlook.}
In summary, the emerging directions in deep network optimization are multifaceted and dynamic, addressing issues of efficiency, robustness, scalability, and theoretical rigor. By integrating advanced curvature approximations with adaptive strategies, embracing memory-efficient architectures, and targeting both robustness and fine-tuning challenges in large models, the field is moving toward a new generation of optimizers. These innovations promise to make deep learning training more efficient and reliable, even as model sizes and dataset complexities continue to increase.
\section{Concluding Remarks}
In this chapter, we have provided a comprehensive survey of the literature on deep network optimization, weaving together theoretical insights, practical algorithms, and emerging research directions. Our review has spanned the full spectrum of techniques and challenges that define the SOTA in deep learning optimization. We began by establishing the foundational concepts and notation in Section~\ref{sec:preliminaries}, setting the stage for a rigorous exploration of deep learning dynamics. In Section~\ref{sec:dynamics}, we examined the evolution of training processes, including phase transitions and the critical role of loss landscape geometry. These elements, along with implicit regularization mechanisms, have been shown to play pivotal roles in guiding training toward solutions with strong generalization properties.

The discussion then transitioned to the practical challenges of optimizing deep networks in Section~\ref{sec:challenges}. Here, we highlighted the inherent difficulties posed by non-convex optimization, the trade-offs between overfitting and generalization, and the significant computational burdens that arise as models grow in scale and complexity. Our in-depth review of optimization methods in Section~\ref{sec:methods} revealed a rich and diverse landscape. We explored:
\begin{itemize}
    \item \textbf{Zeroth-Order Methods} (\ref{subsec:zeroth_order}), which rely solely on function evaluations and finite-difference approximations to operate in gradient-free settings.
    \item \textbf{First-Order Methods} (\ref{subsec:first_order_methods}), such as SGD and adaptive variants like Adam and RAdam, which leverage gradient information to form the backbone of deep network training.
    \item \textbf{Second-Order Methods} (\ref{subsec:second_order_methods}), including Hessian- and Fisher-based approaches like K-FAC and Shampoo, which incorporate curvature information to achieve more informed and potentially faster convergence, albeit at a higher computational cost.
\end{itemize}
Furthermore, Section~\ref{sec:emerging} outlined emerging directions that promise to reshape the field. Researchers are increasingly integrating adaptive strategies with refined curvature approximations, developing memory-efficient architectures tailored to the demands of massive models, and designing specialized algorithms for fine-tuning large-scale transformers and language models. These trends are driven not only by the quest for enhanced performance but also by the need for robust, scalable, and theoretically grounded optimization techniques.

Overall, the insights gathered throughout this chapter not only highlight the remarkable progress achieved in deep network optimization but also underscore the exciting opportunities that lie ahead. By bridging theoretical foundations with practical implementations, the methods reviewed here pave the way for next-generation optimizers that can handle the complexities of modern NNs. As research continues to push the boundaries, future work will likely focus on further integrating adaptive and curvature-aware strategies, enhancing computational efficiency, and developing algorithms that deliver both strong convergence guarantees and superior generalization. In sum, the landscape of deep network optimization is rich with innovation and promise. The diverse approaches surveyed in this chapter provide a robust framework for understanding the multifaceted challenges of modern deep learning, and they set the stage for continued advancements in the field.

\section*{Bridging the Literature Review with AdaFisher}
\label{sec:bridging}

The comprehensive survey presented in this chapter has illuminated the rich tapestry of approaches in deep network optimization—from the efficiency of first-order adaptive methods to the nuanced curvature approximations offered by second-order techniques. Methods such as Adam and RAdam have set the benchmark for adaptive scaling, while curvature-based strategies like K-FAC and Shampoo have demonstrated the power of leveraging the FIM and its approximations to guide optimization. These approaches underscore a central theme: integrating adaptive behavior with curvature-aware preconditioning can significantly enhance both convergence speed and generalization performance.

AdaFisher \citep{gomes2025adafisher} optimization method builds directly on these insights. By fusing adaptive strategies with Fisher-based curvature information, AdaFisher seeks to overcome several limitations identified in the literature. Specifically, it addresses the challenges of computational scalability and stability by dynamically constructing efficient, layer-wise preconditioners derived from the FIM. In doing so, AdaFisher not only inherits the robustness of adaptive optimizers but also benefits from the precise curvature estimates provided by second-order approximations.

In the subsequent Chapters, we detail the theoretical underpinnings of AdaFisher, outline its algorithmic framework, and present empirical evaluations that demonstrate its superior performance on large-scale deep learning tasks. This work represents a natural evolution of the ideas surveyed in this chapter, aiming to bridge the gap between first-order adaptability and second-order precision in a single, scalable optimization method.
\chapter{Efficient Approximation of the FIM} 
\label{chap:efficientfisherapprox}
The FIM is fundamental to statistical estimation and natural gradient descent, providing critical curvature information about the loss landscape in DL. However, directly computing the FIM is infeasible for modern deep networks due to its high dimensionality and complexity. This challenge has spurred efficient approximations—most notably, the K-FAC method—which decomposes the FIM into tractable Kronecker product factors. In this chapter, we investigate the structural properties of the FIM and its approximations. We first review the supervised learning framework and the role of the FIM in natural gradient descent, then examine K-FAC and its extensions across various layer types (fully connected, convolutional, and normalization layers). Empirical analyses reveal that the Kronecker factors exhibit a strong diagonal concentration even under noise, which motivates our novel diagonal approximation of the FIM that significantly reduces computational overhead while preserving essential curvature information.

The chapter is organized as follows. Section~\ref{sec:background} provides the necessary theoretical background. Section~\ref{sec:kroneckerdetail} presents an in-depth study of the diagonal concentration of Kronecker factors, and Section~\ref{sec:effcomputFIM} outlines our efficient computation method for the FIM. Together, these contributions form a scalable, curvature-aware optimization framework for DL.
\section{Background} \label{sec:background}
In this section, we briefly review the core concepts and notation that form the foundation of our methodology for efficient FIM approximation. For a more comprehensive treatment, we strongly recommend first reading Section~\ref{sec:preliminaries} and Section~\ref{subsec:second_order_methods}, where these ideas are explored in depth. The discussion below assumes familiarity with the detailed exposition provided earlier. We consider a supervised learning framework with a dataset \(\mathbf{D} \coloneq \{(\mathbf{x}_{n}, \mathbf{y}_{n})\}_{n=1}^{N},\)
where \(\mathbf{x}_{n} \in \mathbb{R}^{d}\) and \(\mathbf{y}_{n} \in \mathbb{R}^{C}\). Let \(f_{\boldsymbol{\theta}}: \mathbb{R}^{d} \rightarrow \mathbb{R}^{C}\) be an \(L\)-layer NN parameterized by \(
\boldsymbol{\theta} = \text{vec}\left(\{\theta_i\}_{i=1}^{L}\right)\), where \( \theta_i = \text{concat}(W_i, \mathbf{b}_i) \in \mathbb{R}^{P_i}\), where \( P_i = P_{i}^{out}\times (P_{i}^{in}+1)\). The network computes its output through successive layers via
\[
\mathbf{a}_{i} = \theta_{i}\bar{\mathbf{h}}_{i-1}, \quad \mathbf{h}_{i} = \phi_{i}(\mathbf{a}_{i}), \quad \text{with } \mathbf{h}_{0} = \mathbf{x}_{n} \text{ and } \bar{\mathbf{h}}_{i-1} = [\mathbf{h}_{i-1}^\top, \mathbf{1}]^\top.
\]
For a given input-target pair \((\mathbf{x},\mathbf{y})\), the loss is defined as the negative log-likelihood, \(\mathcal{L}(\mathbf{y}, f_{\boldsymbol{\theta}}(\mathbf{x})) \coloneq -\log p_{\boldsymbol{\theta}}(\mathbf{y}|\mathbf{x})\), and the gradients are computed via backpropagation. We denote the pre-activation derivatives by  \(\mathbf{s}_{i} = \nabla_{\mathbf{a}_{i}}\mathcal{L}\),
and the gradient with respect to the weights of layer \(i\) is given by \(\nabla_{\theta_i}\mathcal{L} = \mathbf{s}_{i}\,\bar{\mathbf{h}}_{i-1}^\top\). To capture curvature information, we use the FIM as an approximation to the Hessian
\[
F = \mathbb{E}\left[\nabla_{\boldsymbol{\theta}}\log p_{\boldsymbol{\theta}}(\mathbf{y}|\mathbf{x})\,\nabla_{\boldsymbol{\theta}}\log p_{\boldsymbol{\theta}}(\mathbf{y}|\mathbf{x})^\top\right] = \mathbb{E}[\mathbf{g}\mathbf{g}^\top], \quad \mathbf{g} = \nabla_{\boldsymbol{\theta}} \mathcal{L}(\boldsymbol{\theta})
\]
The K-FAC approach further simplifies the FIM by exploiting the layer-wise structure of deep networks. Viewing \(F\) as a block matrix with blocks\( F_{i,j} = \mathbb{E}[\text{vec}(g_i)\,\text{vec}(g_j)^\top]\), and using the Kronecker-vectorization identity \(\text{vec}(\mathbf{u}\mathbf{v}^\top) = \mathbf{v} \otimes \mathbf{u}\), we express \(\text{vec}(g_i) = \bar{h}_{i-1} \otimes s_{i}\). Assuming that activations and pre-activation derivatives are mutually independent, each block is approximated as
\[
\hat{F}_{i,j} \approx \mathbb{E}[\bar{\mathbf{h}}_{i-1}\bar{\mathbf{h}}_{j-1}^\top] \otimes \mathbb{E}[\mathbf{s}_{i} \mathbf{s}_{j}^\top] = \mathcal{H}_{i-1,j-1} \otimes \mathcal{S}_{i,j}.
\]
By further assuming that weight derivatives across distinct layers are uncorrelated, the FIM is approximated by its block-diagonal form: \( \hat{F} = \text{diag}\left(\hat{F}_1,\ldots,\hat{F}_L\right), \quad \text{with} \quad \hat{F}_i \approx \mathcal{H}_{i-1,i-1} \otimes \mathcal{S}_{i,i}\).
\section{Diagonal Concentration of KFs}
\label{sec:kroneckerdetail}
Inspired by the Gershgorin circle theorem \citep{Horn_Johnson_2012}, we empirically study the diagonal concentration of Kronecker Product Factors (KFs) within DNNs. Our focus centers on the eigenvalue distribution and perturbation behavior of weight matrices, with a particular emphasis on the $37$th layer of ResNet-18 \citep{he2016deep} after 50 epochs of training on the CIFAR-10 dataset \citep{krizhevsky2009learning}. Figure~\ref{fig:gershgorin_and_perturbation} illustrates the eigenvalue spectrum (red crosses) alongside Gershgorin discs (black circles), revealing that the eigenvalues cluster near the diagonal elements of the matrix. This observation underscores a pronounced diagonal dominance, which the Gershgorin circle theorem formalizes by asserting that every eigenvalue \(\lambda\) of a complex square matrix \(\mathcal{A}\) resides within at least one Gershgorin disc \(D(a_{ii}, R_i)\), where \(R_i = \sum_{j \neq i} |a_{ij}|\) quantifies the row-based off-diagonal magnitude. To investigate the stability of this diagonal dominance under stochastic disturbances, we inject Gaussian noise \(\mathcal{N}(0, \sigma^2)\) with \(\sigma = 10^{-3}\) into the off-diagonal elements. The perturbed matrix \(\hat{\mathcal{M}}\) is then given by
\[
\hat{\mathcal{M}} = \mathcal{A} + \mathcal{E}, \quad \text{where } \mathcal{E} = [e_{ij}], \, e_{ij} \sim \mathcal{N}(0, \sigma^2) \text{ for } i \neq j.
\]
Our perturbation experiments show that the principal eigenvalues satisfying the Kaiser criterion are minimally affected \citep{braeken2017empirical}, implying that the bulk of the matrix's energy is concentrated along the diagonal. Both the Gershgorin disc visualization and the eigenvalue perturbation analysis confirm a high degree of \textbf{diagonal dominance}.
\begin{figure}[!h]
    \centering
    \includegraphics[width=\linewidth]{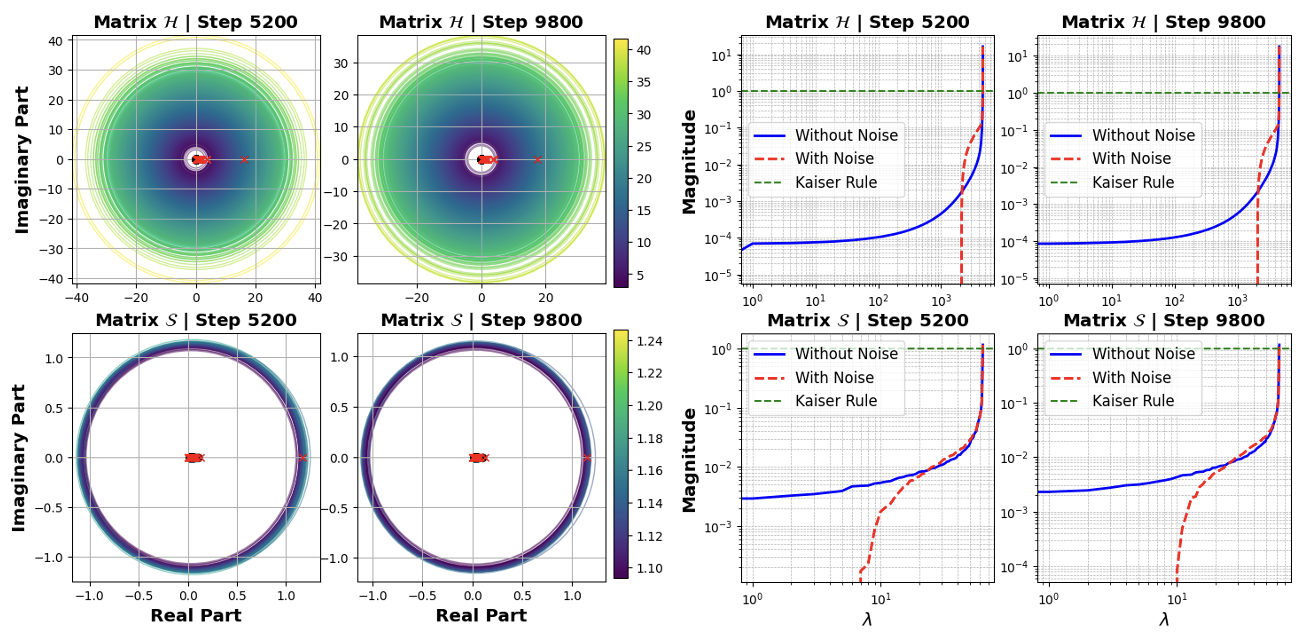}
    \caption{Gershgorin discs and eigenvalue perturbations for the $37$th convolutional layer of ResNet-18 at steps 5200 (middle of training) and 9800 (end of training). Left: Gershgorin discs in the complex plane; Right: Eigenvalue spectrum with and without Gaussian noise added to off-diagonal entries.}
    \label{fig:gershgorin_and_perturbation}
\end{figure}
\paragraph{Extending Gershgorin Analysis to Linear Layers.}
This diagonal concentration is not confined to convolutional layers alone. We observe an analogous pattern in the $41$st (linear) layer of ResNet-18, as shown in Figure~\ref{fig:gershgorin_and_perturbation_lin}. Once again, the eigenvalues (red crosses) lie predominantly within the Gershgorin discs, centered along the diagonal (black circles), underscoring the pervasiveness of this diagonal dominance across diverse layer types.
\begin{figure}[!h]
\centering
\includegraphics[width=\textwidth]{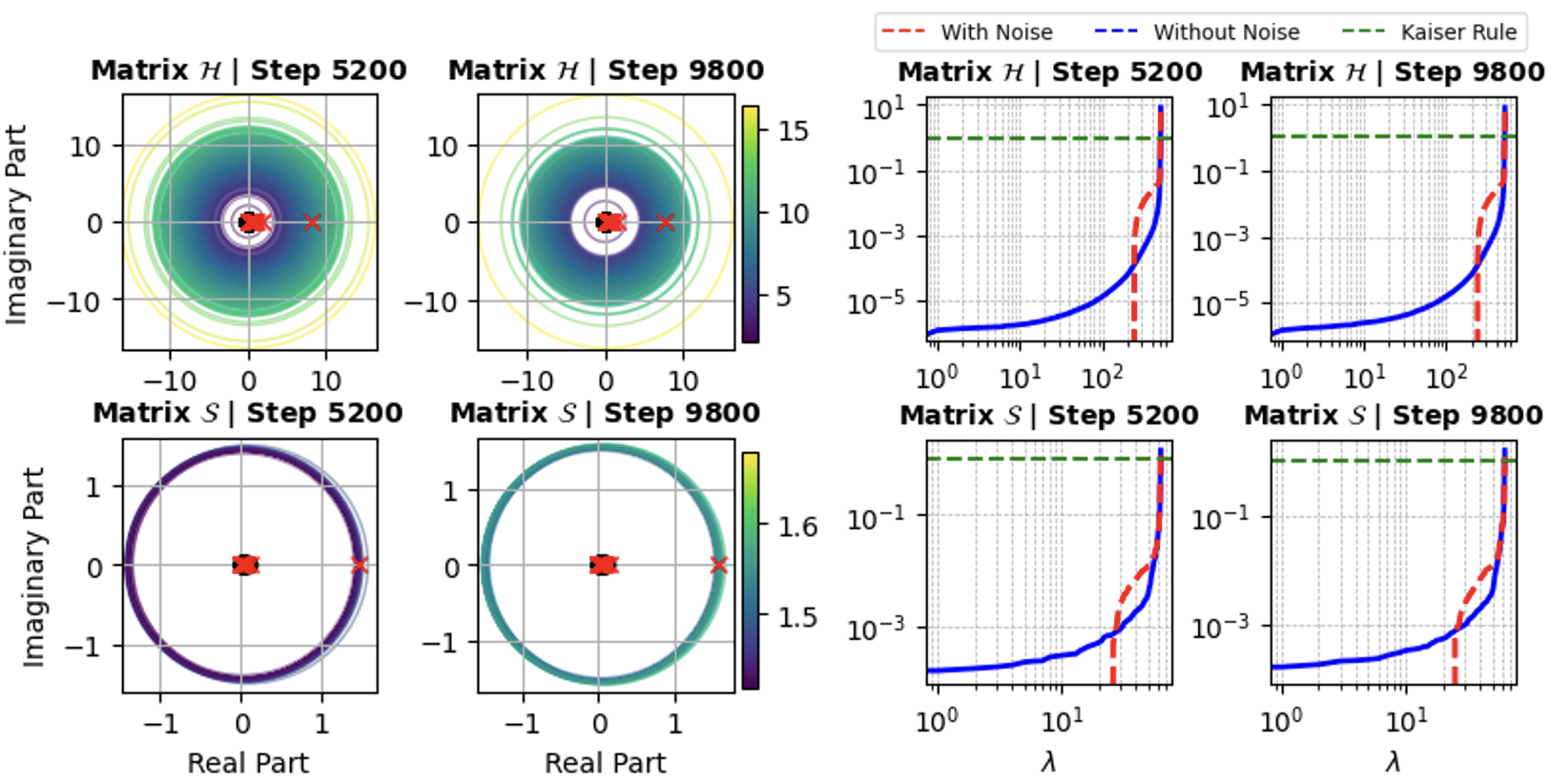}
\caption{Gershgorin discs and eigenvalue perturbation analysis for matrices \(\mathcal{H}\) and \(\mathcal{S}\) at training steps 5200 and 9800 in the linear ($41$st) layer of ResNet-18. The left panel displays Gershgorin discs in the complex plane, while the right panel depicts the eigenvalue spectra with and without Gaussian noise.}
\label{fig:gershgorin_and_perturbation_lin}
\end{figure}
\paragraph{Spectral Insights via FFT.}
Moving beyond eigenvalue-based diagnostics, we also perform a frequency-domain examination using the Fast Fourier Transform (FFT) to study how noise injection impacts the KFs \(\mathcal{H}\) and \(\mathcal{S}\). Let \( A \in \mathbb{C}^{m \times n} \) be a matrix, whose FFT is defined by
\[
\mathcal{F}(A)_{kl} \;=\; \sum_{p=0}^{m-1} \sum_{q=0}^{n-1} A_{pq} \cdot e^{-2\pi i \left(\frac{pk}{m} + \frac{ql}{n}\right)}.
\]
Figure~\ref{fig:fft_conv37} showcases the FFT magnitude plots for both noise-free and noisy conditions at training steps 5200 (middle of training) and 9800 (end of training). Even under Gaussian perturbations, the primary diagonal structure of the KFs remains conspicuous in the frequency domain, signifying that the principal information is predominantly concentrated along the diagonal.
\begin{figure}[!h]
\centering
\includegraphics[width=\textwidth]{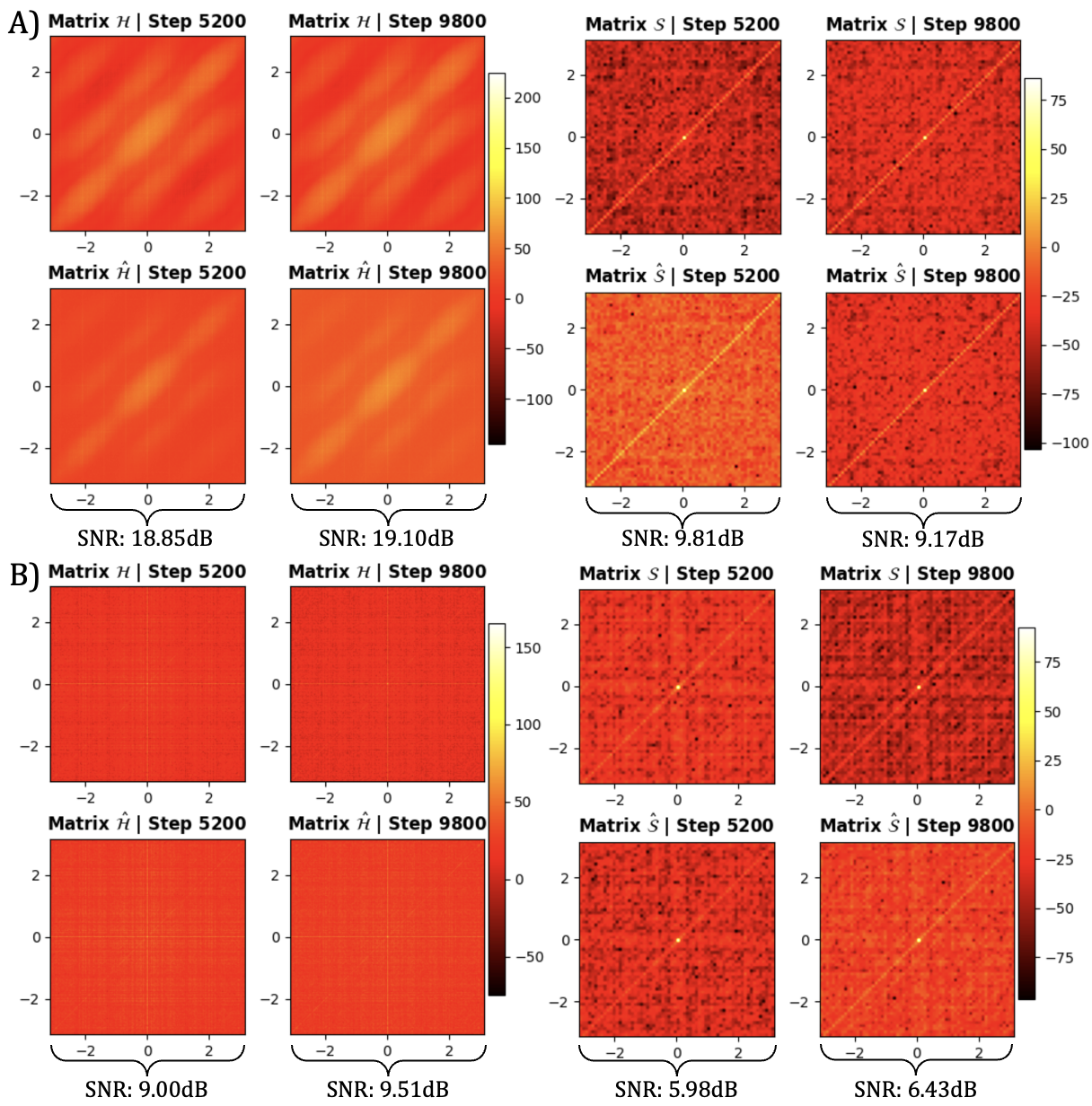}
\caption{FFT-based spectral analysis of KFs \(\mathcal{H}\) and \(\mathcal{S}\) for convolutional (37th layer) and linear (41st layer) segments of a ResNet-18 at iterations 5200 and 9800. (A) Noise-free (top) vs. noisy (bottom) FFT results for \(\mathcal{H}\) from the convolutional layer; (B) analogous results for \(\mathcal{S}\) in the linear layer.}
\label{fig:fft_conv37}
\end{figure}
In addition, we compute the Signal-to-Noise Ratio (SNR) to quantify the effect of noise on off-diagonal entries. For a matrix \(\mathcal{M}\) and its perturbed version \(\hat{\mathcal{M}}\), the SNR is defined by
\begin{equation}
\text{SNR} = 10 \cdot \log_{10}\left(\frac{\sum_{i=1}^{N} |\mathcal{M}_{ii}|^2}{\sum_{j>i}^{N} |\hat{\mathcal{M}}_{ij}|^2}\right). \notag 
\end{equation}
We observe that the SNR remains sufficiently high across training steps, implying minimal eigenvalue displacement and reinforcing our diagonal-dominance hypothesis.
\paragraph{Matrix Visualization and Diagonal Energy.}
Figure~\ref{fig:conv37} provides direct visualizations of the KFs \(\mathcal{H}\) and \(\mathcal{S}\) at steps 5200 and 9800 for both the convolutional ($37$th) and linear ($41$st) layers. Each subplot reveals a pronounced diagonal band, further highlighting the structural consistency of these factors over the course of training. The resilience of this diagonal concentration against off-diagonal noise underscores a form of robust feature extraction, where the primary informational energy resides along the main diagonal.
\begin{figure}[!h]
\centering
\includegraphics[width=\textwidth]{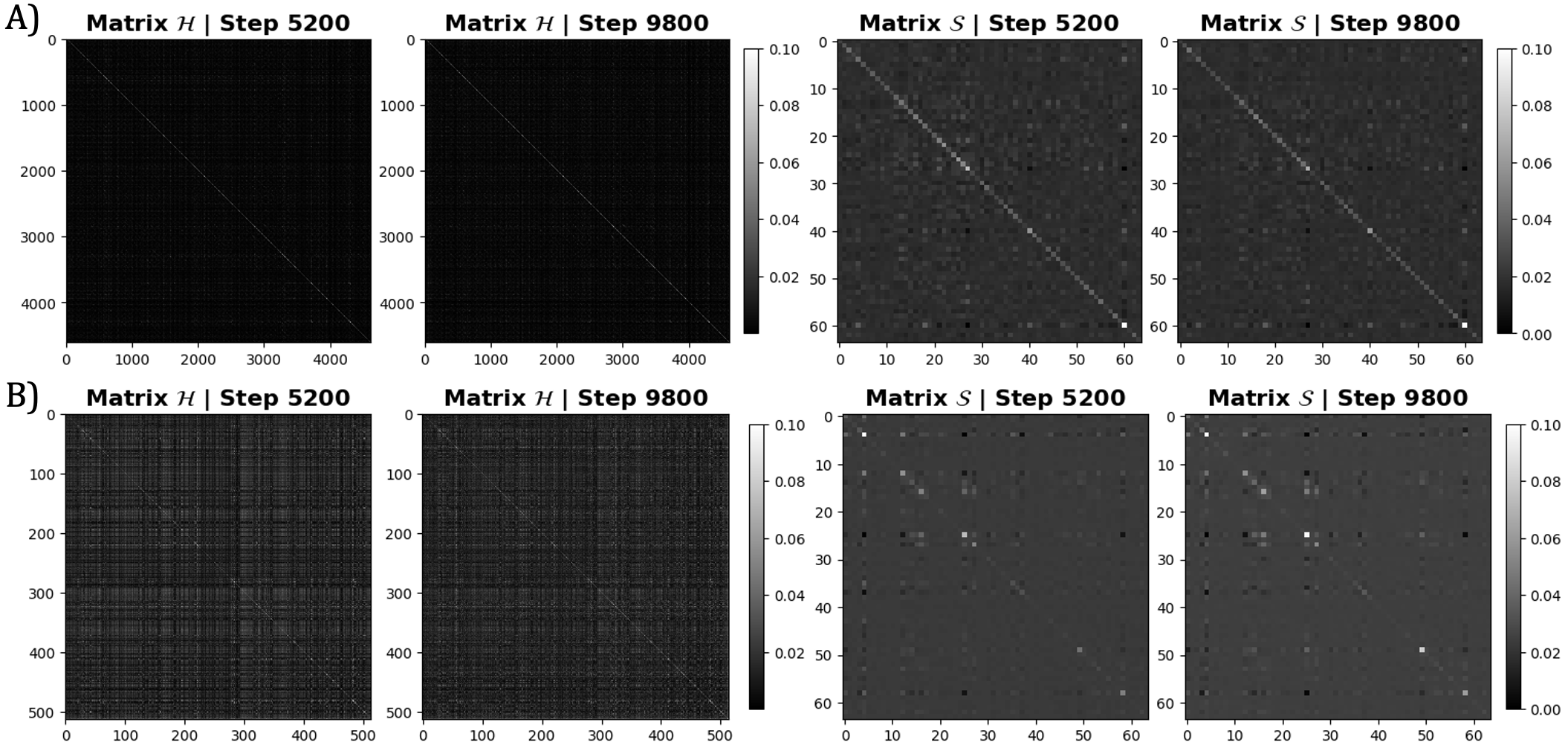}
\caption{Visualization of KFs \(\mathcal{H}\) and \(\mathcal{S}\) in convolutional (A) and linear (B) layers of ResNet-18 at different iteration steps (5200 and 9800). The first two plots in (A) depict factor \(\mathcal{H}\) at the two training stages, while the next two plots illustrate factor \(\mathcal{S}\). Analogously, (B) shows the evolution of \(\mathcal{H}\) and \(\mathcal{S}\) in the linear layer. The diagonal salience persists throughout training, attesting to a stable, diagonal-centric structure.}
\label{fig:conv37}
\end{figure}
\paragraph{Conclusion.}
Both Gershgorin disc analysis and perturbation experiments converge on the finding that the KFs in DNNs, particularly in ResNet-18 layers, exhibit strong diagonal dominance. Noise introduced into off-diagonal elements has a minimal impact on the eigenvalues—especially those surpassing the Kaiser criterion—indicating that essential information is heavily concentrated along the diagonal. This property remains stable over successive training iterations, further corroborating the structural resilience of these matrices. Our Fourier-domain investigation reinforces these observations, revealing that even in the frequency spectrum, the pivotal data of the KFs is localized near the diagonal, thereby reinforcing the robustness and tenacity of Kronecker structures under noisy conditions.
\section{Efficient Computation of the FIM} \label{sec:effcomputFIM}
In the realm of optimization, NGD offers a geometrically nuanced adaptation of the classical steepest descent approach (in Euclidean space), transitioning the focus from parameter space to the model's distribution space underpinned by the adoption of a \emph{Riemannian metric}, \cite{Amari2000MethodsOI}. Using Eq. (\ref{eq:FIMeuqation}), the formulation of the \emph{preconditioned} gradient $\bar{g}^{(t)}$ given the NGD method is articulated as
\begin{align}\label{eq:NGD}
\bar{\mathbf{g}}^{(t)} = (F^{(t)})^{-1} \mathbf{g}^{(t)},
\end{align}
One of the distinguishing features of NGD within this framework is its re-parametrization invariance, a direct consequence of leveraging the model's distribution properties rather than its parameters. Nevertheless, the direct FIM computation is highly demanding, and we solve this by adopting the diagonal approximation of the KFs as supported by our analyses from Section~\ref{sec:kroneckerdetail}. In addition, a critical component of modern DNNs is known to be the normalization of the layers by introducing scale and shift parameters (e.g. batch-normalization \citep{ioffe2015batch}, layer-normalization \citep{ba2016layer}). This is to adjust the network's dynamics (e.g., reducing covariance shift) in a non-trivial way \cite{Huang2023NormalizationTechniques} where the lack of FIM approximation on such normalization layers can lead to suboptimal preconditioning. Therefore, we introduce a method for calculating the KFs for normalization layers detailed in Proposition~\ref{prop:proposition_normalization}.
\begin{proposition}[EFIM for normalization layer]
\label{prop:proposition_normalization}
Let $(\boldsymbol{\nu}_i, \boldsymbol{\beta}_i) \in \mathbb{R}^{C_i}$ be the scale and shift parameters of a normalization layer $i$. The empirical KFs for the FIM approximation are
\begin{align}
\mathcal{H}_{i-1}\Bigr|_{\nu_i} = \frac{1}{|\mathcal{T}_i|} \sum_{x \in \mathcal{T}_i} \mathbf{h}_{i-1,x}\mathbf{h}_{i-1,x}^\top, \, \mathcal{H}_{i-1}\Bigr|_{\beta_i} = \mathbf{1}\mathbf{1}^\top, \quad   
\mathcal{S}_i = \frac{1}{|\mathcal{T}_i|} \sum_{x \in \mathcal{T}_i} \mathbf{s}_{i,x}\mathbf{s}_{i,x}^\top, \notag
\end{align}
where $\mathbf{h}_{i-1}, \mathbf{s}_i \in \mathbb{R}^{C_i \times |\mathcal{T}_i|}$ represent the pre-normalized activations and gradients, respectively. Here, $\mathcal{T}_i$ is the set of dimensions over which normalization statistics are computed, and $C_i$ is the channels/features size.
\end{proposition}
The proof of this proposition is given in Appendix~\ref{chap:ApdxA}. The KFs for other type of layers are computed following methodologies similar to those described in \cite{grosse2016kroneckerfactored, pmlr-v37-martens15}. This section revisits the key equations used for this computation.
For a given layer $i$ in a NN, the empirical KFs are computed as follows:
\begin{itemize}
    \item For \textbf{fully connected layers}, the KFs are:
    \begin{align*}
    \mathcal{H}_{D_{i-1}} =\text{diag}(\bar{\mathbf{h}}_{i-1}\bar{\mathbf{h}}_{i-1}^\top), \quad \mathcal{S}_{D_{i}} = \text{diag}(\mathbf{s}_{i}\mathbf{s}_{i}^\top);
    \end{align*}
    \item For \textbf{convolutional layers}, the computation accounts for the spatial positions within the layer, denoted as $\mathcal{T}$:
    \begin{align*}
    \mathcal{H}_{D_{i-1}} = \text{diag} \left(\frac{\llbracket \bar{\mathbf{h}}_{i-1} \rrbracket \llbracket \overline{\mathbf{h}}_{i-1} \rrbracket^\top}{|\mathcal{T}|}\right), \quad \mathcal{S}_{D_{i}} = \text{diag} \left(\frac{\mathbf{s}_{i}\mathbf{s}_{i}^\top}{|\mathcal{T}|} \right);
    \end{align*}
     The algorithm employs the expansion operation denoted by $\llbracket \cdot \rrbracket$~\citep{grosse2016kroneckerfactored}. \textit{This operation essentially takes the patches surrounding spatial locations, stretches them into vectors, and compiles these vectors into a matrix}.
    \item For \textbf{Normalization layers} (BatchNorm \& LayerNorm) refer to Proposition. \ref{prop:proposition_normalization}
    \item For all \textbf{other type of layers} the KFs are:
     \begin{align*}
    \mathcal{H}_{D_{i-1}} = \textbf{I}_{P_{i-1}^{out}+1}, \quad \mathcal{S}_{D_{i}} = \textbf{I}_{P_{i}^{out}};
    \end{align*}
\end{itemize} 
Note that in the context of online and stochastic optimization, the KFs for a given layer \( i \) can be estimated using an Exponentially Moving Average (EMA) scheme across batches defined by
\begin{align}
    \mathcal{H}_{i-1}^{(t)} = \gamma \mathcal{H}_{i-1}^{(t)} + (1-\gamma) \mathcal{H}_{i-1}^{(t-1)}, \, \mathcal{S}_{i}^{(t)} = \gamma \mathcal{S}_{i}^{(t)} + (1-\gamma) \mathcal{S}_{i}^{(t-1)}, \label{eq:expkronfactors}
\end{align}
where \( 0 < \gamma \leq 1 \) is the exponential decay factor at step time $t$, $ \mathcal{H}_{i-1}^{(t)}$ and $\mathcal{S}_{i}^{(t)}$ in the right-hand side are new KFs calculated during each forward and backward pass computation. This EMA scheme is commonly used in methods involving diagonal or block-diagonal approximations to the curvature matrix (e.g. \cite{8c8eccbbe8a040118afa8f8423da1fe2, PARK2000755, schaul2013pesky}). Such schemes have the desirable property that they allow the curvature estimation to depend on much more data than what can be reasonably processed in a single mini-batch.

Our study from Section~\ref{sec:kroneckerdetail} suggests that the FIM's critical information predominantly resides along its diagonal. Building upon this, we propose a novel approximation for the FIM, described in Proposition~\ref{prop:proposition1}, that conceptualizes the KFs as diagonal matrices denoted as $\tilde{F}_{D_{i}}$ for layer $i$.
\begin{proposition}[Efficient EFIM] \label{prop:proposition1}
Assume that $\mathcal{H}_{i-1}$ and $\mathcal{S}_{i}$ can be closely approximated by diagonal matrices, denoted by $\mathcal{H}_{D_{i-1}}$ and $\mathcal{S}_{D_{i}}$ respectively at layer $i$, such that $\mathcal{H}_{D_{i-1}} = \text{Diag}(\mathcal{H}_{i-1})$, $\mathcal{S}_{D_{i}} = \text{Diag}(\mathcal{S}_{i})$ where $\text{Diag}$ denote the diagonal of a matrix. Accordingly, the Empirical FIM is defined by
\begin{align}
\Tilde{F}_{D_{i}} \triangleq \mathcal{H}_{D_{i-1}}' \otimes \mathcal{S}_{D_{i}}' + \lambda \mathbf{I},\label{eq:FIMDiag}
\end{align}
where $\mathcal{H}_{D_{i-1}}'$ and $\mathcal{S}_{D_{i}}'$ denote the Min-Max normalization of \(\mathcal{H}_{D_{i-1}}\) and \(\mathcal{S}_{D_{i}}\) \citep{patro2015normalization} and $\lambda$ is a regularization parameter. 
\end{proposition}
The proof of this proposition is given in Appendix~\ref{chap:ApdxA}. This approximation strikes a balance between computational time and space complexity and the accuracy of performance, as discussed in Chapter~\ref{chap:experiments}. We set the regularization parameter \(\lambda = 0.001\), which acts as a damping factor following the Tikhonov regularization principle \cite{pmlr-v37-martens15}, enhancing computational stability and conditioning of the FIM. The closed-form solution for the preconditioned gradient \(\bar{\mathbf{g}}^{(t)}\) is derived from the diagonal approximation of the FIM, given by 
\begin{align}
    \bar{\mathbf{g}}^{(t)} = (\tilde{F}_{D}^{(t)})^{-1} \mathbf{g}^{(t)},
\end{align}
for time step $t$. Notice this is similar to the NGD update defined in Eq. (\ref{eq:NGD}), except that we approximate the FIM with a novel efficient method. This represents the AdaFisher, which will be presented in Chapter~\ref{chap:adafisher}, augmented gradient and incorporates local loss \emph{curvature information}. It focuses on the diagonal elements to reduce computational overhead while maintaining a reasonable FIM approximation. This simplification enhances the efficiency of the optimization process, which is crucial for training DNNs where computational resources are limited.
\section{Concluding Remarks}
\label{sec:concluding_remarks_chapter_fisher}
In this chapter, we explored the empirical FIM within DNNs and investigated efficient approximations that facilitate second-order optimization. Our analysis began with a review of the negative log-likelihood loss in supervised learning and demonstrated how the FIM, closely related to the Hessian of the log-likelihood, captures essential curvature information for parameter updates. We then focused on K-FAC methods, which decompose the FIM into KFs for tractability in large-scale networks. By dissecting the diagonal concentration of these factors using Gershgorin circle theorem, we empirically observed a pronounced diagonal dominance across both convolutional and linear layers of a ResNet-18, even under noise perturbations. The spectral and Fourier-domain analyses further corroborated this diagonal-centric structure, indicating that injecting Gaussian noise into the off-diagonal entries yields only marginal shifts in the eigenvalues and minimal deterioration in signal-to-noise ratios. Such observations underscore the stability and robustness of the KFs, suggesting that the critical information for optimization resides predominantly along their diagonals. Building on these insights, we proposed a diagonal approximation of the KFs in Proposition~\ref{prop:proposition1}, offering a simplified yet effective route to approximate the FIM in DL. This diagonalization strategy significantly reduces computational overhead while preserving key curvature cues, aligning with the practical demands of large-scale optimization. 

Additionally, we extended K-FAC to normalization layers by establishing empirical factors for scale and shift parameters, thereby ensuring comprehensive coverage of modern deep network architectures. Our results revealed that even these normalization-centric matrices exhibit strong diagonal behavior, reinforcing the broader applicability of our diagonal approximation across diverse layer types. The final step involved deriving a closed-form preconditioned gradient under the diagonal FIM approximation, leading to an \emph{AdaFisher} update, which will be presented in Chapter~\ref{chap:adafisher}, that incorporates local curvature without incurring the full cost of second-order methods. Overall, the work in this chapter highlights how diagonal-dominant structures in the FIM can be exploited to achieve computationally efficient second-order optimization. This diagonalization perspective not only aligns with empirical observations of robust eigenvalue spectra but also provides a principled foundation for designing fast, curvature-aware training algorithms. Subsequent chapters will build upon these insights, evaluating the performance and convergence properties of diagonal-based Fisher approximations in both theoretical and experimental settings.

\chapter{AdaFisher: An Adaptive Second order Optimization via Fisher Information}\label{chap:adafisher}
In this chapter, we introduce \emph{AdaFisher}, a novel adaptive second-order optimization method that leverages the FIM to guide DNN training. Building on the adaptive framework established by Adam, AdaFisher employs a refined diagonal block-Kronecker approximation of the FIM, leading to more accurate curvature estimation. This enhancement not only accelerates convergence but also steers the optimizer toward flatter minima, thereby improving model generalization.

We begin by outlining how the FIM is integrated into the adaptive optimization framework and discuss the key differences between AdaFisher and conventional optimizers. Next, we present the distributed implementation of AdaFisher, demonstrating its scalability and efficiency in multi-GPU environments. Finally, we analyze the convergence properties of AdaFisher in both convex and non-convex settings, supported by theoretical insights and empirical results.

Overall, this chapter provides a comprehensive overview of the theoretical foundations and practical benefits of incorporating second-order information via the FIM into modern DL optimization.
\begin{table}[hpt]
\centering
\caption{Summary of the first moment (\(m^{(t)}\)), second moment (\(v^{(t)}\)), regret bounds, and applicability used in various optimizers for updating model parameters \(\theta^{(t+1)} = \theta^{(t)} - \alpha\, m^{(t)}/\sqrt{v^{(t)}}\). Here, \(\beta_1\) and \(\beta_2\) denote the hyperparameters for the first and second moments, respectively; \(L^{(t)}\) and \(R^{(t)}\) refer to the preconditioning matrices used in Shampoo \cite{gupta2018shampoo}; \(g^{(t)} = \text{vec}(G^{(t)})\); and \(T\) is the total number of steps. For AdaFactor, \(r^{(t)}\) and \(c^{(t)}\) denote the row and column accumulators that aggregate squared gradients to yield the factored second moment estimate. For Sophia, \(\hat{h}^{(t)}\) is an approximation of the Hessian diagonal (with decay \(\rho\) and regularization \(\epsilon\)). For SOAP, \(Q^{(t)}\) denotes the eigenvector matrix of Shampoo's preconditioner, used to transform the gradients into its eigenbasis for Adam-like updates. Please refer to Chapter~\ref{chap:literaturereview} for more details.}
\label{tab:summaryopt}
\resizebox{\textwidth}{!}{
\begin{tabular}{|l|c|c|c|c|c|}
\hline
\textbf{Optimizer} & \(m^{(t)}\) & \(v^{(t)}\) & \textbf{Regret Bound} & \multicolumn{2}{c|}{\textbf{Applicability}} \\
\cline{5-6}
 & & & & CNNs & Transf. \\
\hline
Adam & \(\frac{(1-\beta_1) \sum_{i=1}^t \beta_1^{t-i} g_i}{1-\beta_1^t}\) & \(\left(\frac{(1-\beta_2) \sum_{i=1}^t \beta_2^{t-i} g_i g_i}{1-\beta_2^t}\right)^{1/2}\)  & \(O(\log T\sqrt{T})\) & \(\checkmark\) & \(\checkmark\) \\
AdaHessian  & \(\frac{(1-\beta_1) \sum_{i=1}^t \beta_1^{t-i} g_i}{1-\beta_1^t}\) & \(\left(\frac{(1-\beta_2) \sum_{i=1}^t \beta_2^{t-i} D_i^{(s)} D_i^{(s)}}{1-\beta_2^t}\right)^{1/2}\) &  \(O(\log T\sqrt{T})\) & \(\checkmark\) & \(\checkmark\) \\
K-FAC  & \((\hat{F}^{(t)})^{-1} \mathbf{g}^{(t)}\) & \(1\) & \(O(\sqrt{T})\) & \(\checkmark\) & \(\times\) \\
Shampoo & \((L^{(t)})^{\frac{-1}{4}} G^{(t)} (R^{(t)})^{\frac{-1}{4}}\) & \(1\) & \(O(\sqrt{T})\) & \(\checkmark\) & \(\times\) \\
Sophia & \(\frac{(1-\beta_1) \sum_{i=1}^t \beta_1^{\,t-i} g_i}{1-\beta_1^t}\) & \(\rho\, \hat{h}^{(t)} + \epsilon\) & \(O(\log T\sqrt{T})\) & \(\checkmark\) & \(\checkmark\) \\
SOAP & \(\frac{(1-\beta_1)\sum_{i=1}^{t} \beta_1^{\,t-i} Q^{(t)\top} g_i}{1-\beta_1^t}\) & \(1\) & \(O(\sqrt{T})\) & \(\checkmark\) & \(\times\)\\
AdaFactor & \(\frac{(1-\beta_1) \sum_{i=1}^t \beta_1^{t-i} g_i}{1-\beta_1^t}\) & \(\frac{r^{(t)}_i\, c^{(t)}_j}{\sum_k r^{(t)}_k}\) & \(O(\log T\sqrt{T})\) & \(\checkmark\) & \(\checkmark\) \\
\hline
\rowcolor{PaleGreen}
AdaFisher    & \(\frac{(1-\beta_1) \sum_{i=1}^t \beta_1^{t-i} g_i}{1-\beta_1^t}\) & \(\tilde{F}_{D}^{(t)}\) & \(O(\log T\sqrt{T})\) & \(\checkmark\) & \(\checkmark\) \\
\hline
\end{tabular}
}
\end{table}
\begin{figure}[!h]
    \centering
    \includegraphics[width=\textwidth]{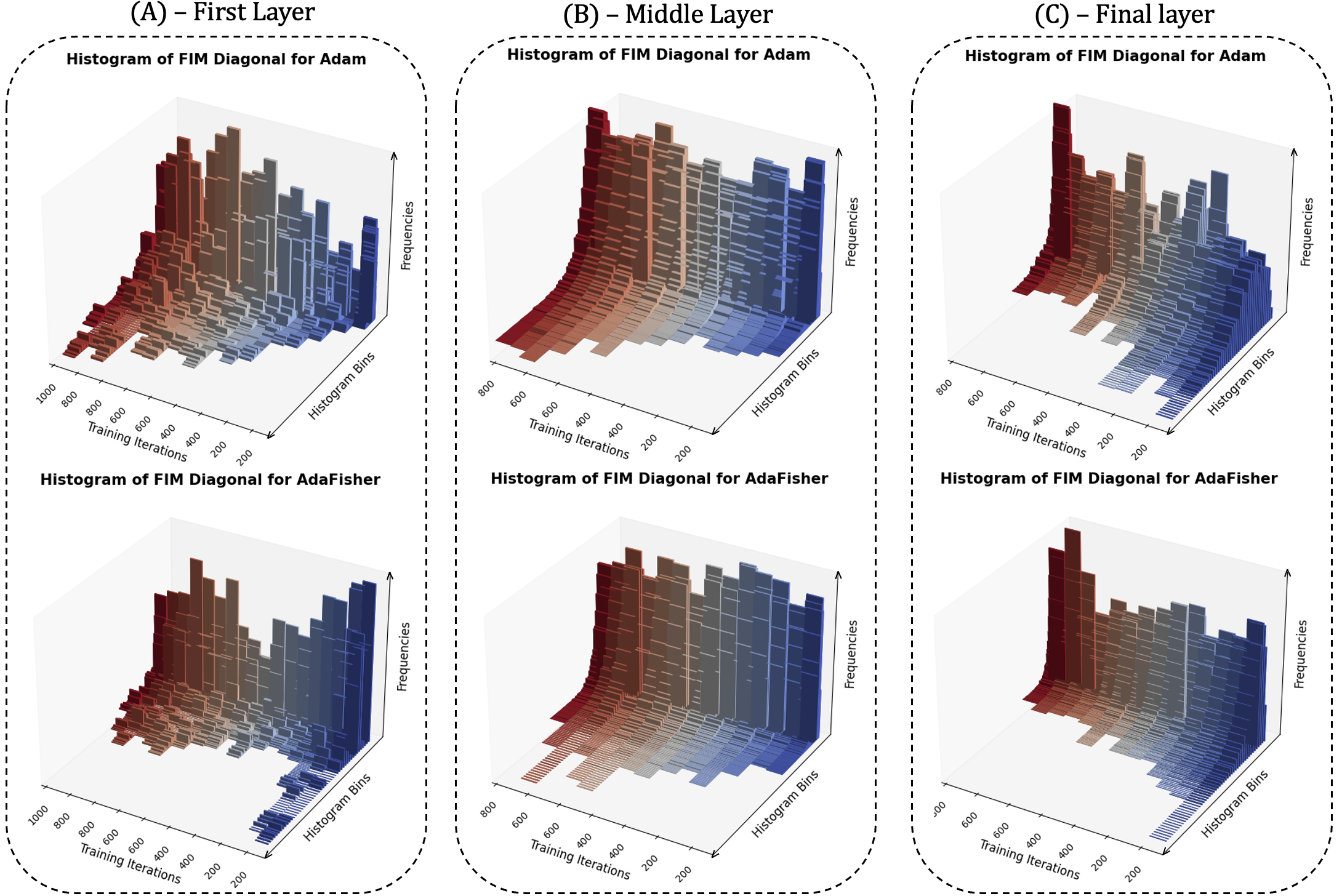}
    \caption{Comparison of FIM Diagonal Histograms during ResNet18 Training on CIFAR10 with Adam and AdaFisher over 1,000 training iterations. Panel (A) displays the FIM diagonal elements for the first convolutional layer; Panel (B) illustrates the FIM diagonal elements for the middle convolutional layer; Panel (C) shows the FIM diagonal elements for the last Linear layer.}
    \label{fig:hist_fisher}
\end{figure}
\section{Integrating FIM into Adaptive Optimization Framework} \label{sec:augmentingFIMadam}
Following the spirit of the adaptive optimization framework from Adam, which combines momentum and RMSProp principles \citep{pmlr-v28-sutskever13}, the parameters are updated via
\begin{align*}
    \boldsymbol{\theta}^{(t+1)} = \boldsymbol{\theta}^{(t)} - \alpha \frac{\mathbf{m}^{(t)}}{\mathbf{v}^{(t)}}.
\end{align*} 
Here, $\alpha$ represents the learning rate, while $\mathbf{m}^{(t)}$ and $\mathbf{v}^{(t)}$ denote the first and second moment estimates, respectively, for time step $t$. Although Adam is widely used, its approximation of the second moment using simple diagonal elements of second-order statistics through squared gradients can mirror stability challenges \citep{NEURIPS2019_46a558d9}. We overcome these challenges by utilizing a more refined diagonal block-Kronecker approximation of the FIM introduced in Section \ref{sec:effcomputFIM}, a more precise approximation of the Hessian from Taylor series expansion viewpoint as discussed in Section~\ref{subsec:multivariatecalculus}. This structured alternative to Adam’s diagonal approximation enables precise curvature estimation, mitigating stability issues and improving convergence in non-convex settings.  AdaFisher distinguishes itself from Adam by providing a higher-fidelity approximation of the FIM, which enhances both optimization efficiency and model robustness throughout training. As shown in Figure~\ref{fig:hist_fisher}, the FIM values computed by AdaFisher exhibit reduced variance and lower mean values during training—indications that the optimizer is converging toward flatter minima and inducing an implicit regularization effect. This behavior is especially pronounced during the final training stages, where AdaFisher outperforms its counterparts (see Chapter~\ref{chap:experiments}). The improvement stems from the FIM’s dual role: in the early and mid-training phases, it approximates the Hessian (acting similarly to the Generalized Gauss-Newton Matrix~\citep{eschenhagen2024kronecker}), and near local minima, it aligns more precisely with the Hessian~\citep{martens2020new}, thereby accelerating convergence. Our analysis of the diagonal block-Kronecker FIM distribution during ResNet18 training on CIFAR10 further corroborates these findings. Across various network layers—including the first (Panel A) and middle (Panel B) convolutional layers, as well as the final linear layer (Panel C)—AdaFisher consistently produces narrower FIM distributions with lower values compared to Adam, confirming its tendency to converge toward flatter local minima. This characteristic ultimately leads to more stable generalization when transitioning from training to testing distributions~\citep{cha2021swad}. 

Additionally, AdaFisher eliminates the square root and the conventional EMA on the second moment. This is because the FIM naturally incorporates an EMA of its KFs (see Eq. (\ref{eq:expkronfactors})). Removing the square root is also consistent with the theoretical foundations of second-order methods, which rely on a second-order Taylor expansion approximation rather than a square-root transformation. A comparative summary of various moment estimates, $\mathbf{m}^{(t)}$ and $\mathbf{v}^{(t)}$, along with their regret bounds and applicability to different optimizers, is provided in Table~\ref{tab:summaryopt}. Notably, while SGD-based optimizers such as K-FAC or Shampoo exhibit a regret bound of $O(\sqrt{T})$, AdaFisher incurs a regret bound of $O(\log T\,\sqrt{T})$ as it belongs to the Adam's family, which is theoretically higher. However, this theoretical disadvantage is offset by AdaFisher's superior practical convergence. Specifically, AdaFisher leverages an approximate FIM to precondition its update steps, thereby reducing the variance of the gradients and adapting more accurately to the local curvature of the loss surface. In real-world scenarios—where the loss landscape is nonconvex and the worst-case assumptions underlying regret bounds are overly pessimistic—this refined curvature estimation facilitates more stable and efficient progress during training. Consequently, even though the worst-case bound for AdaFisher is higher, its enhanced exploitation of second-order information consistently results in faster and more robust convergence compared to its SGD-based counterparts.

Finally, inspired by AdamW—which improves Adam by directly integrating weight decay into the weight update to counteract suboptimal decay behavior and boost performance—we introduce AdaFisherW. This variant leverages the curvature information from AdaFisher within the AdamW framework to further enhance optimization. The implementation for both AdaFisher variants is delineated in the pseudo-code presented in Algorithm~\ref{alg:adaFisher1}. 
\begin{algorithm}[!h]
   \caption{AdaFisher optimization algorithm. Good default settings for the tested machine learning problems are $\alpha = 0.001$ (learning rate), $\lambda = 0.001$ (Tikhonov damping parameter),$\gamma = 0.8$ (Exponentially decaying factor). [Default parameters are: $\beta = 0.9$ (Exponentially decaying factor of Adam), $\kappa$ (weight decay) (\citet{kingma2017adam}, \citet{loshchilov2018decoupled})].}
   \label{alg:adaFisher1}
\begin{algorithmic}[1]
    \REQUIRE Step size $\alpha$; Exponential decay rate for KFs $\gamma \in [0,1)$; Tikhonov damping parameter $\lambda$; Exponential decay rate for first moments $\beta$ in $[0,1)$; Initial parameters $\theta$ \\
    \textbf{Initialize} 1st moment variable $m = 0$;  FIM  $\Tilde{F}_{D_{i}} = \mathbf{I}$; time step $t = 0$\\
    \WHILE{stopping criterion not met}
    \STATE Sample a minibatch of $M$ examples from the training set $\{(x_n,y_n)\}_{n=1}^{M}$ \\
    \vspace{1mm}
    \STATE Compute $\mathcal{H}_{D_{i-1}}$, $\mathcal{S}_{D_{i}}$ for $i \in \{1, \dots, L\}$ using Section~\ref{sec:effcomputFIM} (notice that: $\mathcal{H}_{D_{0}} = x$) \\
    \vspace{1mm}
    \STATE  Compute EMAs of $\mathcal{H}_{D_{i-1}}$ and $\mathcal{S}_{D_{i}}$ using Eq. (\ref{eq:expkronfactors})\\
    \vspace{1mm}
    \STATE Compute $\tilde{F}_{D_{i}}$ for $i \in \{1, \dots, L\}$ using Eq. (\ref{eq:FIMDiag})\\
    \vspace{1mm}
    \STATE $g^{(t)} \leftarrow \frac{1}{M} \sum_{n} \nabla_{\theta^{(t)}}\mathcal{L}(f(x_n; \theta^{(t)}), y_n)$ (Compute gradient) \\
    \vspace{1mm}
    \STATE $m^{(t+1)} \leftarrow \frac{\beta m^{(t)} + (1 - \beta)h^{(t)}}{1 - \beta^{t}}$ (Update and correct biased first moment) \\
    \vspace{1mm}
    \STATE \textbf{Case AdaFisher:} $\Delta \theta^{(t)} = - \alpha (\tilde{F}_{D}^{(t)})^{-1}m^{(t)}$\\ \textbf{Case AdaFisherW:} $\Delta \theta^{(t)} = - \alpha \left((\tilde{F}_{D}^{(t)})^{-1}m^{(t)} + \kappa \theta^{(t)}\right)$ \\
    \vspace{1mm}
    \STATE  $\theta^{(t+1)} \leftarrow \theta^{(t)} + \Delta \theta^{(t)}$ (Apply update)\\
    \vspace{1mm}
    \STATE $t \leftarrow t + 1$ \\
    \ENDWHILE
\end{algorithmic}
\end{algorithm}
\section{Distributed Implementation and Algorithm Overview} \label{sec:distributedImplementation}
To complement the integration of the FIM into the adaptive optimization framework, we now detail the algorithmic implementation of AdaFisher and its compatibility with distributed multi-GPU environments. This section provides a concise overview of the core update rules, the mechanism for aggregating curvature information across GPUs, and a summary of the pseudo-code for AdaFisher. In distributed settings, aggregating these KFs across GPUs is crucial before updating model parameters. For a training setup with $K$ GPUs and for any given layer $i$, the KFs are computed and aggregated as follows
\begin{align}
\mathcal{H}_{D{i-1}}^{(\text{SUM})} = \frac{1}{K} \sum_{k=1}^{K} \mathcal{H}_{D{i-1}}^{(k)}, \quad \mathcal{S}_{D_i}^{(\text{SUM})} = \frac{1}{K} \sum_{k=1}^{K} \mathcal{S}_{D_i}^{(k)} \label{eq:distributedkf}
\end{align}
The theoretical justification for this aggregation lies in the linearity of expectation and the unbiasedness of the local KF estimates. Specifically, if each $\mathcal{H}_{D_{i-1}}^{(k)}$ and $\mathcal{S}_{D_i}^{(k)}$ are unbiased estimators of their respective true factors $\mathcal{H}_{D_{i-1}}$ and $\mathcal{S}_{D_i}$ for $k=1, \dots, K$, then the averaged factors $\mathcal{H}_{D_{i-1}}^{\text{(SUM)}}$ and $(\mathcal{S}_{D_i})^{\text{(SUM)}}$ remain unbiased estimators of $\mathcal{H}_{D_{i-1}}$ and $\mathcal{S}_{D_i}$. Consequently, using Eq. (\ref{eq:distributedkf}), the aggregated EFIM for layer $i$ can be calculated as
\begin{align*}
    \tilde{F}_{D_i}^{\text{SUM}} = \mathcal{H}_{D_{i-1}}'^{(\text{SUM})} \otimes \mathcal{S}_{D_i}'^{(\text{SUM})} + \lambda \mathbf{I}
\end{align*}
where $\lambda$ is a regularization parameter that ensures numerical stability. This strategy allows each GPU to contribute effectively to the model update, thereby enhancing convergence and performance in large-scale distributed training. More details, results and ablation about the multi-GPU environments are detailed in Chapter~\ref{chap:experiments}.
\section{Convergence Analysis} \label{sec:convergenceanalysis}
In this section, we provide a theoretical analysis of AdaFisher's convergence in both convex optimization and non-convex stochastic optimization. We first present a standard convergence behavior of Eq. (\ref{eq:NGD}) for a simple strongly convex and strictly smooth function $f(J)$.  
\begin{proposition}[Convergence in convex optimization]\label{prop:convex-case}
For the FIM defined in Eq. (\ref{eq:FIMDiag}), the updating scheme $\boldsymbol{\theta}^{(t+1)} = \boldsymbol{\theta}^{(t)} -\alpha  (\tilde{F}_{D}^{(t)})^{-1} \nabla J(\boldsymbol{\theta}^{(t)})$ converges. Moreover, if $\nabla J$ is Lipschitz, i.e., $||\nabla J(\boldsymbol{\theta}) - \nabla J(\boldsymbol{\theta}')||_2 \leq L ||\boldsymbol{\theta} - \boldsymbol{\theta}'||$ for any $\boldsymbol{\theta} $ and $ \boldsymbol{\theta}'$, then for the $k$-step iteration with a fixed step size $\alpha\leq 1/L$, then 
\begin{equation*}
J(\boldsymbol{\theta}^{(k)}) - J(\boldsymbol{\theta}^*) \leq \frac{||\boldsymbol{\theta}^{(0)} - \boldsymbol{\theta}^*||_2^2}{2\alpha k},
\end{equation*}
where $J(\boldsymbol{\theta}^*)$ is the optimal value.  
\end{proposition}
For non-convex cases, we adopt the similar derivations of \cite{chen2018on} since AdaFisher belongs to the family of generalized Adam-type methods. 

\begin{proposition}[Convergence in non-convex stochastic optimization] \label{prop:non-convex} Under the assumptions:\\
(i) $f$ is lower bounded and differentiable; $||\nabla J(\boldsymbol{\theta}) - \nabla J(\boldsymbol{\theta}')||_2\leq L ||\boldsymbol{\theta} - \boldsymbol{\theta}'||_2$, $||\tilde{F}_{D}^{(t)}||_\infty<L,\, \forall t, \boldsymbol{\theta}, \boldsymbol{\theta}'$.\\
(ii) Both the true and stochastic gradient are bounded, i.e. $||\nabla J(\boldsymbol{\theta}^{(t)})||_2 \leq \lambda$ and $||g_t||_2 \leq \lambda$, $\forall t$ for some $\lambda>0$.\\
(iii) Unbiased and independent noise in $\mathbf{g}^{(t)}$, i.e. $\mathbf{g}^{(t)} = \nabla J(\boldsymbol{\theta}^{(t)}) + \zeta^{(t)}$, $\mathbb{E}[\zeta^{(t)}] = 0$, and $\zeta^{(t)}\perp\zeta^{(t)}$, $\forall i \neq j$.

Assume $\eta^{(t)} = \frac{\eta}{\sqrt{t}}$, $\beta^{(t)}\leq \beta\leq 1$ is non-increasing, $\frac{\tilde{F}_{D}^{(t-1)}[j]}{\eta^{(t-1)}}\leq \frac{\tilde{F}_{D}^{(t)}[j]}{\eta^{(t)}}$, $\forall t\in [T], j\in [d]$, we then have 
\begin{equation*}
\min_{t\in [T]}\mathbb{E}[||\nabla J(\boldsymbol{\theta}^{(t)})||_2^2] \leq \frac{L}{\sqrt{T}}(C_1 \eta^2 \lambda^2(1+ \log T) + C_2 d\eta + C_3 d\eta^2 + C_4)
\end{equation*}
where $C_1, C_2, C_3$ are constants independent of $d$ and $T$, $C_4$ is a constant independent of $T$, the expectation is taken w.r.t all the randomness corresponding to $\{g^{(t)}\}$.
\end{proposition}
The proofs of propositions~\ref{prop:convex-case} and \ref{prop:non-convex} are given in Appendix~\ref{chap:ApdxA}. Proposition \ref{prop:non-convex} implies the convergence rate for AdaFisher in the non-convex case is at $O(\log T /\sqrt{T})$, which is similar to Adam-type optimizers. While DNNs often include non-smooth components like ReLU and max pooling, which create non-differentiable points in the loss landscape, optimizers like AdaFisher handle these cases effectively, as shown by our results in Chapter~\ref{chap:experiments}. We further demonstrate AdaFisher's convergence on a simple model and dataset described in Appendix~\ref{chap:Appendixvisualization}. As illustrated in Figure~\ref{fig:weighttrajectories}, across various seeds, AdaFisher consistently outperforms the baseline optimizer by converging toward the optimal local minima in the loss landscape with respect to parameters \(W_1\) and \(W_2\). These results further validate the convergence analysis discussed above.
\begin{figure}[!h]
    \centering
    \includegraphics[width=\linewidth]{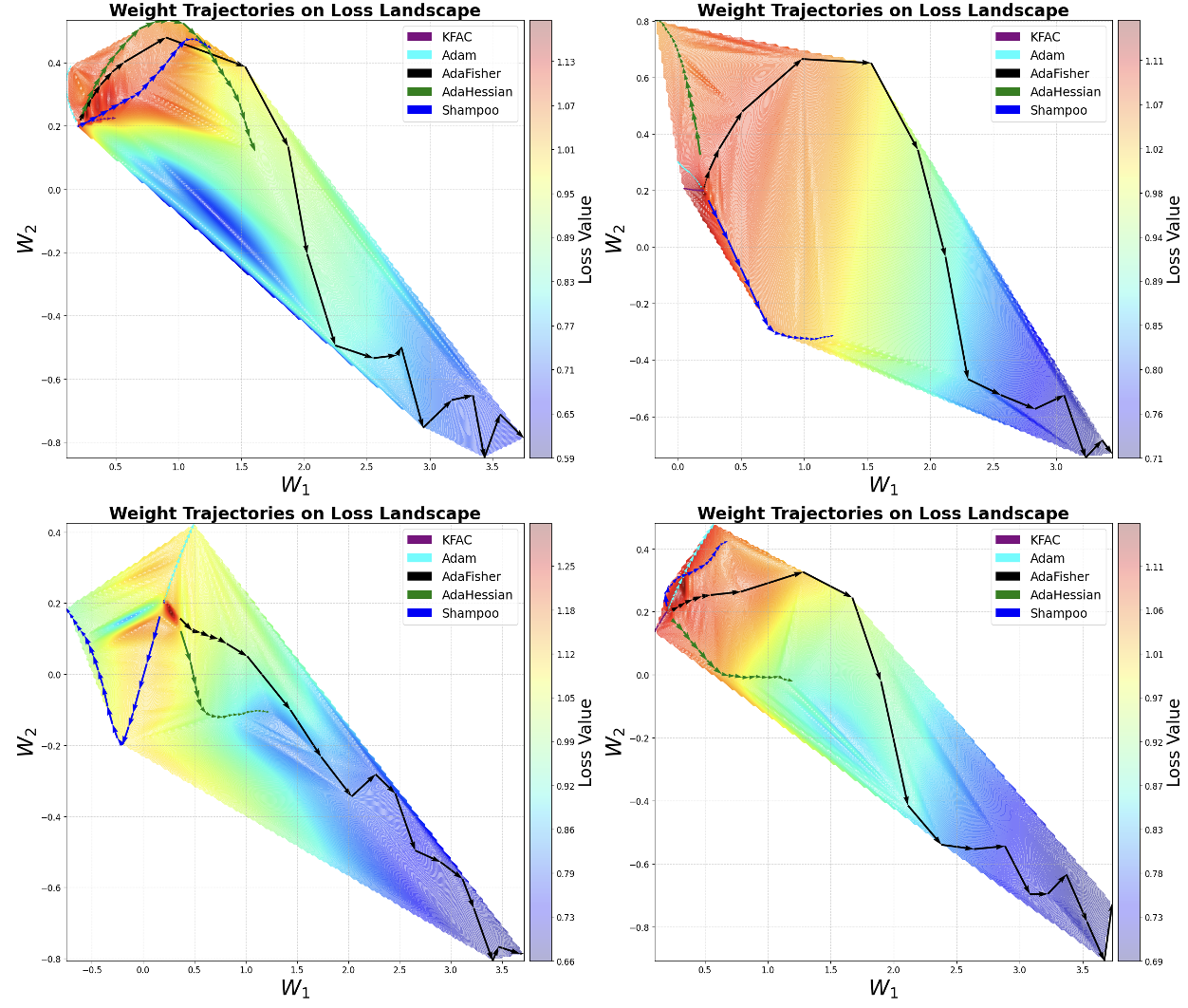}
    \caption{Weight trajectories within different loss landscapes (evaluated using four seeds) of a toy model for different optimizers.}
    \label{fig:weighttrajectories}
\end{figure}
\section{Concluding Remarks}
In this chapter, we introduced \emph{AdaFisher}, a novel adaptive second-order optimization method that leverages refined approximations of the FIM to enhance deep network training. By integrating a diagonal block-Kronecker approximation into an adaptive framework inspired by Adam, AdaFisher offers a higher-fidelity estimation of curvature compared to conventional methods, thus enabling more accurate and stable updates. Our approach not only accelerates convergence by effectively capturing local curvature information but also steers the optimization trajectory toward flatter minima, which is beneficial for generalization. We detailed the theoretical foundations underlying AdaFisher, including its derivation from natural gradient descent principles and the approximation of the FIM using layer-wise Kronecker factorization. The distributed implementation further demonstrates that AdaFisher scales efficiently in multi-GPU environments, making it practical for large-scale deep learning applications. Our convergence analysis, covering both convex and non-convex settings, shows that AdaFisher attains convergence rates comparable to other SOTA Adam-type optimizers while benefiting from a more robust curvature estimation mechanism.

From a theoretical standpoint, AdaFisher distinguishes itself by achieving convergence rates comparable to Adam-type methods—specifically, an $O(\log T/\sqrt{T})$ rate in non-convex settings—while incorporating more precise curvature information through a diagonal block-Kronecker approximation of the FIM. This refined curvature estimation not only improves the stability of updates by mitigating the effects of noisy gradients but also guides the optimization trajectory toward flatter minima, which are empirically linked to better generalization. In contrast, first-order methods like Adam rely solely on momentum and squared gradient accumulation, potentially compromising convergence in regions of high curvature. On the other hand, second-order methods such as K-FAC and Shampoo offer rapid convergence through full curvature information but often incur substantial computational overhead and reduced scalability. Hence, AdaFisher provides a compelling balance, leveraging enhanced second-order insights to maintain robust convergence while remaining computationally efficient for large-scale DL applications.

Empirical results, as summarized in Table~\ref{tab:summaryopt} and illustrated in Figure~\ref{fig:hist_fisher}, provide strong evidence that AdaFisher yields reduced variance in the Fisher estimates and drives the optimizer toward regions of the loss landscape that generalize well. Furthermore, the introduction of AdaFisherW, which integrates weight decay in the spirit of AdamW, underscores the flexibility and practical relevance of our method. Overall, the contributions presented in this chapter highlight a promising direction for the integration of second-order information into adaptive optimization. By bridging the gap between theoretical rigor and practical efficiency, AdaFisher paves the way for future research aimed at developing even more robust, scalable, and generalizable optimization methods for DNNs.
\chapter{Experiments}
\label{chap:experiments}
In this chapter, we present a comprehensive evaluation of AdaFisher across a range of tasks in both CV and NLP. Our experiments span image classification (using CIFAR-10, CIFAR-100 \citep{krizhevsky2009learning}, Tiny ImageNet \citep{le2015tiny}, and ImageNet-1k \citep{5206848}) and language modeling (using WikiText-2 \citep{merity2017pointer} and Penn Treebank (PTB) \citep{marcus-etal-1993-building}). We compare AdaFisher with several baseline optimizers, including SGD, Adam/AdamW, K-FAC, AdaHessian, and Shampoo. In addition, we perform a transfer learning task using pretrained ImageNet-1k weights \citep{paszke2019pytorch}. We finalize this section with a thorough ablation and stabilization studies of AdaFisher.

\section{Hyper-Parameter Optimization} \label{sec:HPdetailed}
To ensure a fair comparison, we carefully tuned HP for all optimizers across the different architectures and datasets. The experiments were conducted on six benchmark datasets, with HP tuning performed separately for CNN-based models and Vision Transformers (ViTs). Below, we detail our approach. The baseline optimizers compared include SGD, Adam/AdamW, AdaHessian, K-FAC, and Shampoo. While we conducted a single experiment for ImageNet-1k and Tiny ImageNet datasets due to computation time requirements, the CIFAR results (with or without pretrained weights) represent averaged performance metrics across five independent runs. We conducted one experiment for each language modeling datasets.
\paragraph{Hardware.} In total, we had a server with 6 NVIDIA RTX 6000 Ada Generation GPUS with 48 gigabytes of VRAM and 128 gigabytes of RAM available for all experiments. All experiments described in this report were conducted on a system equipped with a single NVIDIA RTX 6000 Ada Generation GPU and 64 gigabytes of RAM, except for training AdaFisher on ImageNet-1k with batch sizes of 512 and 1024, where four GPUs were utilized.
\subsection{Image Classification Experiments}  
We provide further results and detailed descriptions of our image classification experiments in this section. We conducted five trials with random initializations for the CIFAR experiments and one trial each for Tiny ImageNet and ImageNet-1k, presenting the mean and standard deviation of the results for these trials.

\textbf{Note on training time.} Given that various optimizers demonstrate significantly different epoch durations, we have standardized our comparisons by restricting training to the total wall-clock time (WCT) consumed by 200 epochs using AdaFisher for both CIFAR and Tiny ImageNet experiments. Conversely, for ImageNet-1k, we report the results based on 90 WCT training epochs using Adam, as, surprisingly, AdaFisher and Adam exhibited the same duration in this experiment. The final selected number of epochs for each optimizer is detailed in Table~\ref{num_epochs}. Note that we were unable to train AdaHessian on ImageNet-1k due to the significant computational resources required by this optimizer.

\begin{table*}[!ht]
    \centering
    \caption{Comparison of the final epoch counts for various optimizers across different datasets.}
    \scriptsize
    \setlength{\tabcolsep}{2pt}
    \begin{tabular}{c|c|c|c|c|c|c||c|c|c|c}
         & \multicolumn{6}{c||}{CIFAR10/100 \& Tiny ImageNet} & \multicolumn{4}{c}{ImageNet-1k} \\ \midrule
         Optimizers & SGD & Adam/AdamW & AdaHessian & K-FAC & Shampoo & AdaFisher/AdaFisherW & Adam & K-FAC & Shampoo & AdaFisher  \\
         \midrule
         Epochs & 226 & 210 & 89 & 107 & 36 & 200 & 90 & 60 & 26 & 90 
         \label{num_epochs}
    \end{tabular}
\end{table*}

\paragraph{HP Tuning.} Effective hyperparameter (HP) tuning is crucial for optimizing the performance of deep learning models. In this study, we systematically explored various hyperparameters for both CNNs and ViTs across multiple image classification tasks. The following sections detail the tuning strategies employed for each model architecture and dataset.

\paragraph{CNNs.} For all image classification tasks involving CNNs, we utilized ResNet18 as the backbone architecture and evaluated its performance on the CIFAR-10 dataset with a fixed batch size of 256, trained for 50 epochs. The HP tuning process encompassed several components. First, for optimizer selection and learning rate tuning, each optimizer was fine-tuned using ResNet18 on CIFAR-10. We performed a grid search to identify the optimal learning rate from the set \(\{0.0001, 0.0003, 0.0005, 0.0009, \dots, 0.1, 0.3, 0.5, 0.9\}\). Second, a cosine annealing learning rate decay strategy was employed, aligning with the number of training epochs specified for each optimizer in Table~\ref{num_epochs}. This approach follows the methodology proposed by \citet{loshchilov2018decoupled} and was determined to be optimal for our experimental setup. Third, weight decay was applied uniformly at \(5 \times 10^{-4}\) across all optimizers for CIFAR-10 and Tiny ImageNet; an exception was made for MobileNetV3, where the weight decay was set to \(1 \times 10^{-5}\), while for experiments on ImageNet-1k the weight decay was established at \(1 \times 10^{-4}\). Fourth, for damping parameter tuning, different strategies were applied. For AdaFisher, K-FAC, and Shampoo, the damping parameter for K-FAC and AdaFisher was searched within the set \(\{0.0001, 0.0003, 0.0005, 0.0009, 0.001, 0.003,\) \newline \(0.005, 0.009, 0.01, 0.03, 0.05, 0.09\}\), a range chosen based on prior research \citep{pmlr-v37-martens15} and our own experiments that indicated optimal damping values around \(1 \times 10^{-3}\). For Shampoo, the damping parameter was tuned within the set \(\{1 \times 10^{-6}, 3 \times 10^{-6}, 5 \times 10^{-6}, 9 \times 10^{-6}, 1 \times 10^{-5}, 3 \times 10^{-5}, 5 \times 10^{-5}, 9 \times 10^{-5}, 1 \times 10^{-4}, 3 \times 10^{-4}, 5 \times 10^{-4}, 9 \times 10^{-4}\}\), as optimal values typically reside around \(1 \times 10^{-5}\). In addition, for AdaHessian the Hessian power was tuned within the range \(\{0.1, 0.2, \dots, 0.9, 1.0\}\); for SGD the momentum was tuned within the range \(\{0.1, 0.2, \dots, 0.9, 1.0\}\); and for AdaFisher the decay factor \(\gamma\) was tuned within the set \(\{0.1, 0.2, \dots, 0.9, 0.99\}\), with the optimal value determined to be \(\gamma = 0.8\). Finally, for implementation details, we utilized the ASDL library as implemented in PyTorch provided by \citet{osawa2023asdl} for both the Shampoo and K-FAC optimizers.

\paragraph{ViTs.} For ViT-based image classification tasks, we employed the Tiny Swin Transformer on the CIFAR-10 dataset with a batch size of 256. The HP tuning strategy for ViTs included several elements. Weight decay values were set as indicated in the respective original publications for each model: \(1 \times 10^{-2}\) for Tiny Swin, \(5 \times 10^{-2}\) for FocalNet, and \(6 \times 10^{-2}\) for CCT-2/3$\times$2. For learning rate tuning, a grid search was conducted over the set \(\{0.3, 0.15, 0.1, 0.05, 0.03, 0.015, 0.01, 0.005,\)\newline \(0.003, 0.0015, 0.001\}\) for optimizers such as SGD, AdaFisher, AdaHessian, K-FAC, and Shampoo, since these optimizers typically operate with higher learning rates compared to Adam-based optimizers. For AdamW, the learning rates were adopted from the original publications: \(1 \times 10^{-4}\) for Tiny Swin and FocalNet, and \(5.5 \times 10^{-5}\) for CCT-2/3$\times$2. The same grid search approach was applied for tuning the damping parameter for K-FAC, Shampoo, and AdaFisher, as well as for the Hessian power for AdaHessian, the momentum for SGD, and the decay factors for AdaFisher, as explained in the CNNs section. 

This meticulous HP tuning process ensures that each optimizer is optimally configured for the respective model architectures and datasets, thereby facilitating a fair and comprehensive comparison of their performance across different image classification tasks. The final learning rates for all optimizers and models are detailed in Table~\ref{learning_rate}.
\begin{table*}[!ht]
    \centering
    \caption{Final selected learning rates for each optimizer, tuned using ResNet18 (for CNN) and Tiny Swin (for ViT) on CIFAR10 using a batch size of 256. We selected based on final validation top-1 accuracy.}
    \scriptsize
    \setlength{\tabcolsep}{7pt}
    \begin{tabular}{c|c|c|c|c|c|c|c|c}
    Architecture & SGD & Adam & AdamW & AdaHessian & K-FAC & Shampoo & AdaFisher  & AdaFisherW\\
    \midrule
    CNNs & 0.1& 0.001 & - & 0.15 & 0.3 & 0.3 & 0.001 & -\\
    ViTs & 0.01& - & 0.0001/0.000055 & 0.01 & 0.003 & 0.003 & - & 0.001
    \label{learning_rate}
    \end{tabular}
\end{table*}

\paragraph{Dataset Details and Data Augmentation.} The CIFAR-10/100 datasets contain 50k training and 10k test images. We use a batch size of 256. Data augmentation includes \(32 \times 32\) random-resize cropping, random horizontal flipping, and Cutout \citep{devries2017improved} (see \cite{Takahashi_2020} for details). Tiny ImageNet Comprises 100k training and 10k test images, we perform \(64 \times 64\) random-resize cropping and random horizontal flipping with a batch size of 256. Finally, ImageNet-1k With 1,281,167 training and 150k test images, we set a batch size of 256. Training employs random resized cropping to \(224 \times 244\) and random horizontal flipping, while testing uses a resize to \(256 \times 256\) followed by \(224 \times 224\) center cropping.
\paragraph{Transfer Learning Experiments.} For transfer learning, model weights are initialized using publicly available PyTorch checkpoints, except for the first convolutional layer of ResNet and the final dense layers, which are randomly initialized. Models are trained with a weight decay of \(1 \times 10^{-4}\) (or \(1 \times 10^{-5}\) for MobileNetV3). A grid search over learning rates in \(\{0.3, 0.15, 0.1, 0.03, 0.015, 0.01, \dots,\) \newline \( 1 \times 10^{-5}\}\) was performed, with the final selections provided in Table~\ref{learning_rate_pretrained}. A batch size of 256 and cosine learning rate decay were used. Final epoch counts for each optimizer are listed in Table~\ref{table:epoch_num_pretrained}.

\begin{table*}[!ht]
    \centering
    \caption{Final selected learning rates for each optimizer in the transfer learning task, tuned using ResNet50 on CIFAR-10 (batch size 256).}
    \scriptsize
    \setlength{\tabcolsep}{20pt}
    \begin{tabular}{c|c|c|c|c|c}
    SGD & Adam & AdaHessian & K-FAC & Shampoo & AdaFisher \\
    \midrule
    0.01 & 0.0001 & 0.15 & 0.3 & 0.03 & 0.001 
    \end{tabular}
    \label{learning_rate_pretrained}
\end{table*}

\begin{table*}[!ht]
    \centering
    \caption{Final epoch counts for various optimizers in the transfer learning task.}
    \scriptsize
    \setlength{\tabcolsep}{15pt}
    \begin{tabular}{c|c|c|c|c|c}
    SGD & Adam/AdamW & AdaHessian & K-FAC & Shampoo & AdaFisher/AdaFisherW\\
    \midrule
    58 & 55 & 22 & 27 & 18 & 50
    \end{tabular}
    \label{table:epoch_num_pretrained}
\end{table*}

\subsection{Language Modeling Experiments}

For language modeling, we utilize a streamlined GPT-1 architecture, which incorporates four self-attention layers, a reduction from the original twelve. This configuration retains core modeling capabilities while reducing complexity, encompassing a total of 28,351,488 learnable parameters. To expedite training, we employ pretrained embeddings from OpenAI's GPT, leveraging the benefits of parameter sharing for enhanced efficiency and faster convergence. The models were trained on WikiText-2 and PTB for 50 wall clock time epochs using AdaFisher. Final epoch counts for each optimizer are provided in Table~\ref{table:epoch_num_LM}. For AdamW, we follow the learning rate settings in \cite{elnokrashy2022depth}. For the other optimizers, a grid search over \(\{0.3, 0.15, 0.1, 0.05, 0.03, 0.015, 0.01, \dots, 1 \times 10^{-5}\}\) was conducted, and the selected values are summarized in Table~\ref{learning_rate_lm}. A batch size of 32 and weight decay of 0.1 were used. Notably, Shampoo failed to converge and K-FAC could not be trained on these tasks.

\begin{table*}[!ht]
    \centering
    \caption{Final epoch counts for various optimizers in the language modeling task.}
    \scriptsize
    \setlength{\tabcolsep}{33pt}
    \begin{tabular}{c|c|c|c}
    AdamW & AdaHessian & Shampoo & AdaFisherW\\
    \midrule
    55 & 18 & 12 & 50
    \end{tabular}
    \label{table:epoch_num_LM}
\end{table*}

\begin{table*}[!ht]
    \centering
    \caption{Final selected learning rates for each optimizer, tuned using GPT1 on WikiText-2 and PTB (batch size 32), based on final validation perplexity (PPL).}
    \scriptsize
    \setlength{\tabcolsep}{33pt}
    \begin{tabular}{c|c|c|c}
    AdamW & AdaHessian & Shampoo & AdaFisherW\\
    \midrule
    \(5 \times 10^{-5}\) & 0.015 & 0.003 & \(1 \times 10^{-4}\)
    \end{tabular}
    \label{learning_rate_lm}
\end{table*}
\section{AdaFisher for Image Classification Tasks}
We commence our analysis by assessing the convergence and generalization capabilities of various models on image classification tasks. Specifically, we deploy ResNet architectures (ResNetX where $X \in \{18, 50, 101\}$), DenseNet121~\citep{huang2017densely}, MobileNetV3 \citep{howard2019searching}, Tiny Swin~\citep{liu2021swin}, FocalNet~\citep{yang2022focal} and CCT-2/3$\times$2~\citep{hassani2021escaping} on CIFAR10 and CIFAR100, while utilizing standard ResNet50 for Tiny ImageNet and ImageNet-1k. The performance outcomes for CIFAR datasets are detailed in Table~\ref{table_cutout_vs_nocutout}. Our empirical evaluation of  AdaFisher optimizer across these models and datasets illustrates its efficiency in optimizing image classification, surpassing all SOTA optimizers. We employ the WCT method with a cutoff of 200 epochs for AdaFisher's training, except for ImageNet-1k, where we use a 90-epoch WCT for Adam, which surprisingly matched AdaFisher's training duration. Results confirm AdaFisher's superior classification accuracy on both CNNs and ViTs. The training losses and test errors for the CIFAR experiments, both with and without cutout, are visually represented in Figures~\ref{fig:results_cifar10_no_cutout}, \ref{fig:results_cifar100_no_cutout}, \ref{fig:results_cifar10_cutout}, and \ref{fig:results_cifar100_cutout} in Appendix~\ref{chap:results}. Furthermore, Table~\ref{TinyImageNet} displays the results for the Tiny ImageNet dataset using ResNet50 and Big Swin networks, with visualizations provided in Figure~\ref{fig:results_TinyImageNet}. AdaFisher and AdaFisherW consistently outperform current SOTA optimizers. Notably, Figure~\ref{fig:results_TinyImageNet} illustrates that although AdaFisher converges slower than K-FAC during ResNet50 training, it achieves superior generalization. This is evidenced by lower testing errors, suggesting that AdaFisher tends to converge to a flatter local minimum, enabling smoother transitions between training and testing datasets with minimal generalization loss. For further explanation, see \cite{cha2021swad}. Note that due to AdaHessian's high memory consumption, we were unable to train it on Big Swin.
\begin{table*}[!ht]
    \caption{Performance metrics (mean, std) of different networks and optimizers on CIFAR10 and CIFAR100 using batch size 256 (a) without Cutout and (b) with Cutout. Reported using WCT of 200 AdaFisher training epochs as the cutoff.}
    \label{table_cutout_vs_nocutout}
    \begin{subtable}{\textwidth}
    \noindent
    \setlength{\tabcolsep}{0.001pt}
    \caption{Without Cutout}
    \scriptsize
    \centering
    \begin{tabular}{c|c|c|c|c|c|c||c|c|c|c|c|c}
    &\multicolumn{6}{c||}{CIFAR10}&\multicolumn{6}{c}{CIFAR100}\\ \cline{2-13} Network&SGD&Adam&AdaHessian&K-FAC&Shampoo&AdaFisher&SGD&Adam&AdaHessian&K-FAC&Shampoo&AdaFisher\\
    \midrule ResNet18&$94.89_{0.1}$&$93.64_{ 0.1}$&$94.05_{0.1}$&$94.04_{0.2}$&$94.52_{0.1}$&$\mathbf{95.02_{ 0.1}}$&$76.42_{0.1}$&$72.71_{0.2}$&$73.64_{0.2}$&$74.79_{0.2}$&$76.53_{0.1}$ & $\mathbf{77.10_{0.2}}$\\ ResNet50&$95.07_{0.2}$&$93.89_{02}$&$94.26_{0.1}$&$94.25_{0.1}$&$94.92_{0.1}$&$\mathbf{95.42_{0.2}}$&$77.50_{0.2}$&$73.12_{0.7}$&$75.29_{0.3}$&$75.49_{0.2}$&$77.81_{0.2}$ & $\mathbf{78.91_{0.9}}$\\ ResNet101&$94.77_{0.1}$&$93.14_{0.1}$&$94.73_{0.9}$&$94.23_{0.1}$&$94.22_{0.1}$&$\mathbf{95.51_{0.1}}$&$78.76_{0.2}$&$73.23_{0.4}$&$72.19_{0.2}$&$75.46_{0.3}$&$78.82_{0.1}$ & $\mathbf{79.74_{0.3}}$\\ DenseNet121&$95.11_{0.1}$&$93.74_{0.2}$&$94.54_{0.1}$&$94.97_{0.1}$&$94.99_{0.1}$&$\mathbf{95.29_{0.1}}$&$78.61_{0.2}$&$75.38_{0.3}$&$72.54_{0.9}$&$77.09_{0.3}$&$78.70_{0.3}$ & $\mathbf{79.03_{0.2}}$\\
    MobileNetV3&$92.13_{0.2}$&$91.95_{0.1}$&$91.4_{3.1}$&$91.92_{0.1}$&$91.91_{0.2}$&$\mathbf{92.89_{0.1}}$&$73.81_{0.2}$&$65.64_{0.2}$&$60.78_{3.6}$&$69.87_{0.3}$&$68.01_{0.2}$ & $\mathbf{73.15_{0.2}}$\\
    \hline
    Tiny Swin&$80.08_{0.2}$&$87.47_{0.2}$&$78.34_{0.2}$&$66.84_{0.3}$&$68.44_{0.2}$&$\mathbf{89.08}_{0.1}$&$57.43_{0.3}$&$62.20_{0.2}$&$54.12_{0.3}$&$36.12_{0.3}$&$33.75_{0.3}$ & $\mathbf{66.47_{0.2}}$\\
    FocalNet&$80.87_{0.2}$&$85.65_{0.1}$&$71.03_{0.3}$&$42.92_{0.2}$&$41.49_{0.2}$&$\mathbf{86.92}_{0.1}$&$45.66_{0.3}$&$52.88_{0.3}$&$38.05_{0.3}$&$11.23_{0.3}$&$11.06_{0.3}$ & $\mathbf{52.9_{0.1}}$\\
    CCT-2/3$\times$2 & $73.12_{0.2}$&$83.95_{0.1}$ & $-$ & $34.63_{1.1}$ & $35.1_{0.8}$ & $\mathbf{84.63_{0.3}}$ & $52.12_{1.2}$& $60.14_{1.1}$ & $-$ & $8.06_{0.6}$ & $9.76_{0.3}$ &$\mathbf{60.63_{0.6}}$
    \end{tabular}\\
    $^*$Note that Adam and AdaFisher were used for all CNN architectures, while AdamW and AdaFisherW were applied for all ViT experiments.
    \end{subtable}
    \bigskip 
    \begin{subtable}{\textwidth}
    \noindent
    \centering
    \setlength{\tabcolsep}{0.001pt}
    \caption{With Cutout}
    \scriptsize
    \centering
    \begin{tabular}{c|c|c|c|c|c|c||c|c|c|c|c|c}
    &\multicolumn{6}{c||}{CIFAR10}&\multicolumn{6}{c}{CIFAR100}\\ \cline{2-13} Network&SGD&Adam&AdaHessian&K-FAC&Shampoo&AdaFisher&SGD&Adam&AdaHessian&K-FAC&Shampoo&AdaFisher\\
    \midrule ResNet18&$95.64_{0.1}$&$94.85_{ 0.1}$&$95.44_{0.1}$&$95.17_{0.2}$&$94.08_{0.2}$&$\mathbf{96.25_{ 0.2}}$&$76.56_{0.2}$ &$75.74_{0.1}$&$71.79_{0.2}$&$76.03_{0.3}$&$76.78_{0.2}$ & $\mathbf{77.28_{0.2}}$\\ ResNet50 &$95.71_{0.1}$&$94.45_{0.2}$&$95.54_{0.1}$&$95.66_{0.1}$&$94.59_{0.1}$&$\mathbf{96.34_{0.2}}$&$78.01_{0.1}$&$74.65_{0.5}$&$75.81_{0.3}$&$77.40_{0.4}$&$78.07_{0.4}$ & $\mathbf{79.77_{0.4}}$\\ ResNet101&$95.98_{0.2}$&$94.57_{0.1}$&$95.29_{0.6}$&$96.01_{0.1}$&$94.63_{0.1}$&$\mathbf{96.39_{0.1}}$&$78.89_{0.2}$&$75.56_{0.3}$&$73.38_{0.2}$&$77.01_{0.4}$&$78.83_{0.2}$ & $\mathbf{80.65_{0.4}}$\\ DenseNet121&$96.09_{0.1}$&$94.86_{0.1}$&$96.11_{0.1}$&$96.12_{0.1}$&$95.66_{0.1}$&$\mathbf{96.72_{0.1}}$&$80.13_{0.4}$&$75.87_{0.4}$&$74.80_{0.9}$&$79.79_{0.2}$&$80.24_{0.3}$ & $\mathbf{81.36_{0.3}}$\\
    MobileNetV3&$94.43_{0.2}$&$93.32_{0.1}$&$92.86_{3.1}$&$94.34_{0.1}$&$93.81_{0.2}$&$\mathbf{95.28_{0.1}}$&$73.89_{0.3}$&$70.62_{0.3}$&$56.58_{4.5}$&$73.75_{0.3}$&$70.85_{0.3}$ & $\mathbf{77.56_{0.1}}$\\
    \hline
    Tiny Swin&$82.34_{0.2}$&$87.37_{0.6}$&$84.15_{0.2}$&$64.79_{0.5}$&$63.91_{0.4}$&$\mathbf{88.74_{0.4}}$&$54.89_{0.4}$ &$60.21_{0.4}$&$56.86_{0.5}$&$34.45_{0.4}$&$30.39_{1.2}$ & $\mathbf{66.05_{0.5}}$\\
    FocalNet&$82.03_{0.2}$&$86.23_{0.1}$&$64.18_{0.2}$&$38.94_{0.8}$&$37.96_{0.7}$&$\mathbf{87.90}_{0.1}$&$47.76_{03}$&$52.71_{0.5}$&$32.33_{0.3}$&$9.98_{0.6}$&$9.18_{0.1}$ & $\mathbf{53.69_{0.3}}$\\
    CCT-2/3$\times$2&$78.76_{0.3}$&$83.89_{0.4}$&$-$&$33.08_{2.3}$&$35.16_{0.4}$&$\mathbf{84.94}_{0.3}$&$54.05_{0.4}$&$59.78_{0.5}$&$-$&$7.17_{0.2}$&$8.60_{0.1}$ & $\mathbf{62.91_{0.5}}$
    \end{tabular}\\
    $^*$Note that Adam and AdaFisher were used for all CNN architectures, while AdamW and AdaFisherW were applied for all ViT experiments.
    \end{subtable}
\end{table*}
\begin{table*}[!ht]
    \noindent
    \centering
    \setlength{\tabcolsep}{18pt}
    \caption{Performance of various networks and optimizers on Tiny ImageNet using batch size 256. Reported using wall clock time of 200 AdaFisher training epochs as the cutoff.}
    \scriptsize
    \centering
    \begin{tabular}{c|c|c|c|c|c}
    Network&Adam&AdaHessian&K-FAC&Shampoo&AdaFisher\\
    \midrule ResNet50&$53.06$&$50.21$&$50.05$&$53.53$&$\mathbf{57.41}$\\
    Big Swin&$48.11$&$-$&$8.89$&$4.11$&$\mathbf{48.86}$
    \label{TinyImageNet}
    \end{tabular}
\end{table*}
\begin{figure}[!h]
    \centering
    \includegraphics[width=\textwidth]{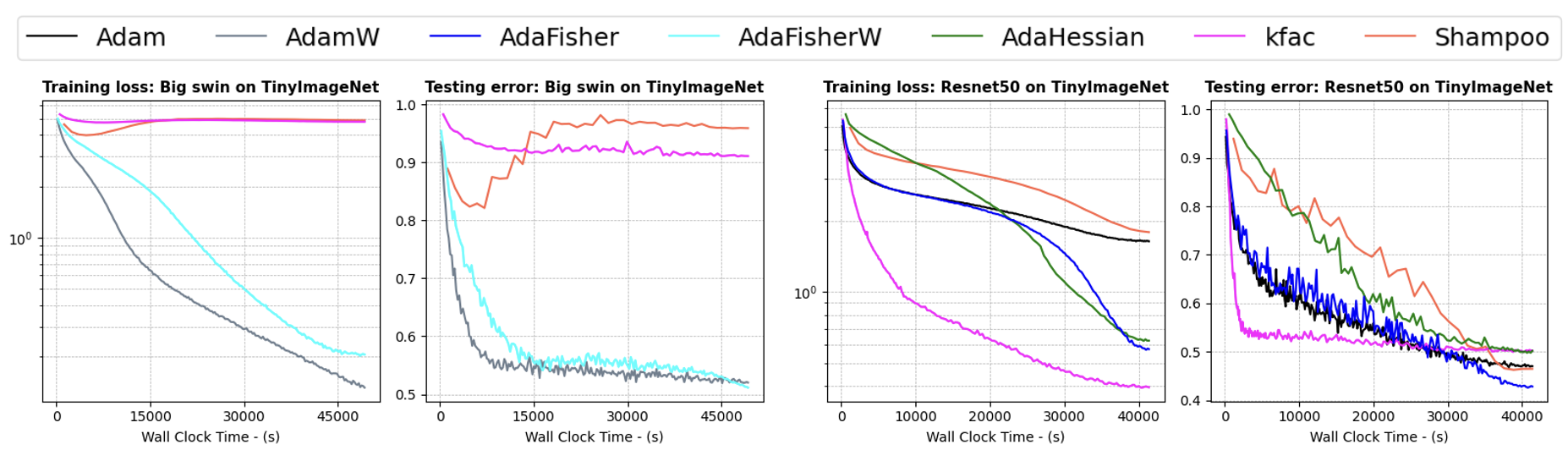}
    \caption{WCT training loss and testing error curves of several optimizers on Tiny ImageNet dataset, ResNet-50 and Big Swin with a batch size of 256. AdaFisher consistently achieves lower test error as compared to Adam, AdaHessian, K-FAC and Shampoo. The final accuracy results are reported in Table~\ref{TinyImageNet}.}
    \label{fig:results_TinyImageNet}
\end{figure}
\subsection{ImageNet-1k Training} 
Training on ImageNet-1k typically requires multiple GPUs and large batch sizes. Our study showcases that AdaFisher achieves superior validation accuracy on a single GPU than its counterparts in scenarios marked by the light blue region. This performance outstrips traditional approaches like SGD, LAMB \citep{You2020Large}, and LARS \citep{you2017large}, which typically utilize batch sizes of 16K. While AdaFisher attains SOTA results on a single GPU, it further excels when scaled up in a distributed setting with larger batch sizes. The results, benchmarked using 256 batch size and a WCT of 90 Adam training epochs, are detailed in Table~\ref{tab:imagenetresults} and illustrated in Figure~\ref{fig:result_ImageNet_TinyImageNet}. Distributed AdaFisher curves are illustrated in Figure~\ref{fig:results_imageNet_dist}. The light blue highlights in the table represent our experiments with a batch size of 256 on a single GPU. The light green indicates results from a distributed version of AdaFisher employing larger batch sizes, whereas the orange reflects results from SOTA methods using a higher batch size of 16K, SGD with a batch size of 256 and  AdamW with a batch size of 1024. It is important to note, however, that the training setups and augmentation techniques for the results highlighted in orange, taken from the literature, may differ from those in our study. These results are included to provide a broader context and intuition regarding AdaFisher’s performance compared to other experiments. Overall, by balancing curvature-aware updates with parameter efficiency, AdaFisher maintains strong generalization even under constrained computational budgets, a critical advantage over methods like K-FAC that struggle with over-parameterized models. Moreover, Figure~\ref{fig:results_imageNet_dist} illustrates the training and validation error of the distributed version of AdaFisher on ImageNet-1k across various batch sizes. AdaFisher not only outperforms its counterparts with smaller batch sizes (256), but it also continues to achieve superior generalization as batch sizes increase. Furthermore, these results reinforce the stability analysis concerning batch sizes presented in Section~\ref{sec:stabilityanalysis}, extending it to a more challenging dataset.
\begin{figure}[h]
    \begin{minipage}{0.5\textwidth}
        \captionof{table}{Validation of ImageNet-1k / ResNet50 by different optimizers reported on Top-1 and Top-5 accuracy.}
        \label{tab:imagenetresults}
        \scriptsize
        \centering
        \setlength{\tabcolsep}{4pt}
        \begin{tabular}{c|c|c|c}
        \toprule
        Optimizers & Batch size & Top-1 & Top-5 \\
        \midrule
        \rowcolor{LightCyan}
        Adam   &  $256$ & $67.78$ & $88.37$ \\
        \hline
        \rowcolor{LightCyan}
        K-FAC  & $256$ & $70.96$ & $89.44$  \\
        \hline
        \rowcolor{LightCyan}
        Shampoo & $256$ & $72.82$ & $91.42$ \\
        \hline
        \rowcolor{LightCyan}
        AdaFisher & $256$ & $\mathbf{76.95}$& $\mathbf{93.39}$\\
        \midrule
        \rowcolor{LightGreen}
        AdaFisher & $512$ & $\mathbf{77.01}$& $\mathbf{93.45}$\\
        \hline
        \rowcolor{LightGreen}
        AdaFisher & $1024$ & $\mathbf{77.09}$& $\mathbf{93.56}$\\
        \midrule
        \rowcolor{Orangelight}
        SGD \cite{goyal2017accurate}& $256$ & $76.40$ & -  \\
        \hline
        \rowcolor{Orangelight}
        AdamW \cite{chen2024symbolic} &$1024$&$76.34$ & -\\
        \hline
        \rowcolor{Orangelight}
        LAMB \cite{You2020Large}& $16K$ & $76.66$ & $93.22$ \\
        \hline
        \rowcolor{Orangelight}
        SGD \cite{You2020Large}& $16K$ & $75.20$ & -  \\
        \hline
        \rowcolor{Orangelight}
        LARS \cite{huo2021large}& $16K$ & $75.1$ & - \\
        \bottomrule
        \end{tabular}
    \end{minipage}
    \hfill
    \begin{minipage}{0.45\textwidth}
        \centering
        \includegraphics[width=\textwidth]{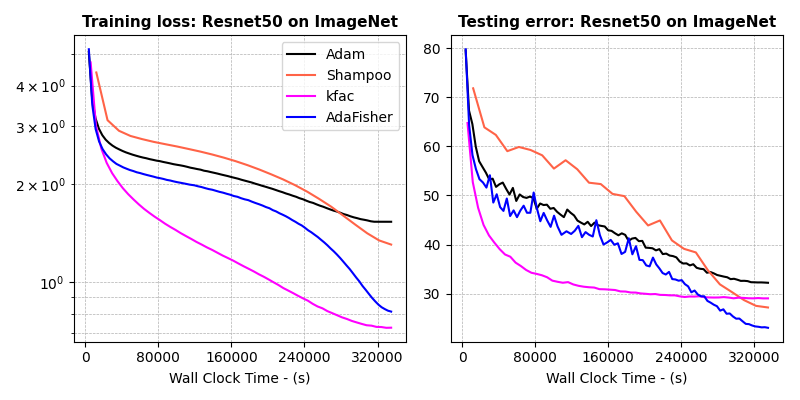}
        \captionof{figure}{Training loss and validation error of ResNet-50 on ImageNet-1k. AdaFisher consistently achieves lower test error as compared to its counterparts.}
        \label{fig:result_ImageNet_TinyImageNet}
    \end{minipage}
\end{figure}
\begin{figure}[!h]
    \centering
    \includegraphics[width=\textwidth]{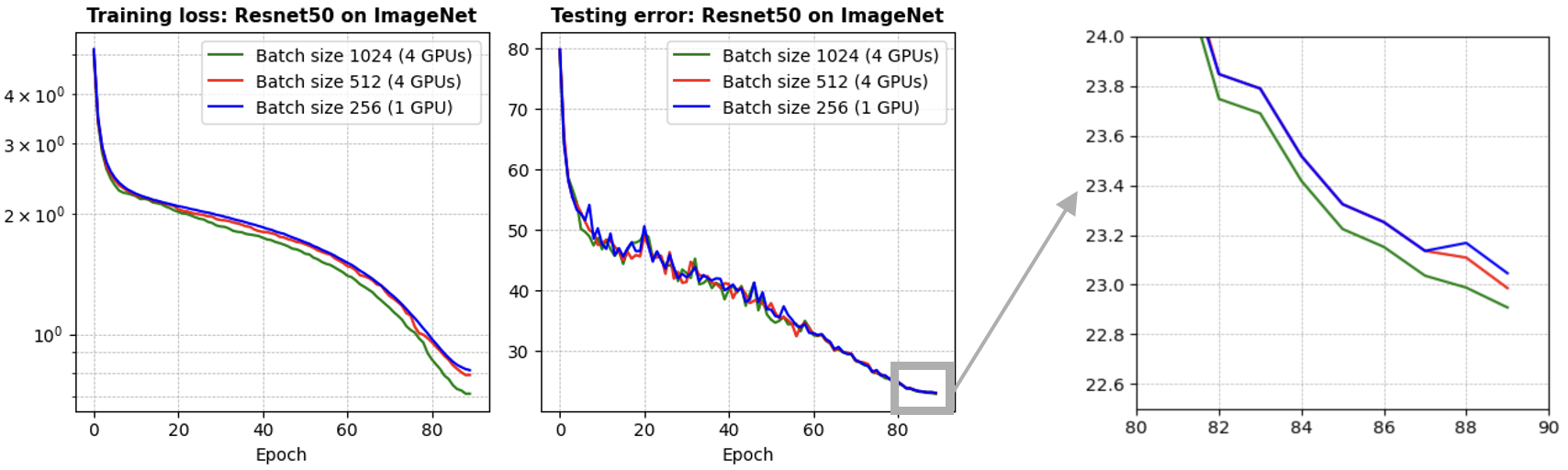}
    \caption{Performance of distributed AdaFisher using ResNet50 on ImageNet-1k with different batch sizes for 90 epochs. The final accuracy results are reported in Table~\ref{tab:imagenetresults}.}
    \label{fig:results_imageNet_dist}
\end{figure}
\subsection{Comparison with Other Relevant Methods}
In this section, we compare AdaFisher with six baseline optimizers for image classification: SGD, Adam/AdamW, AdaHessian, KFAC, and Shampoo. These baselines were selected because they either represent the current state of the art or utilize second-order gradients, making them suitable comparisons for evaluating second-order optimizers. However, other optimizers, such as AdaFactor \citet{shazeer2018adafactor} and EVA \citet{zhang2023eva}, are also relevant in this context. AdaFactor is an enhanced Adam memory-efficient optimizer that approximates second-order moments using row and column factorizations, reducing memory consumption for large-scale models. EVA is a second-order optimizer designed to leverage the FIM with efficient matrix inversion techniques. Therefore, we experimentally compare AdaFisher against the optimizer baselines, including Eva and AdaFactor. Regarding the HPs for EVA, we used the optimal values reported in its original paper and trained the model for 119 epochs using the WCT technique. For AdaFactor, we fine-tuned the learning rate as described in Section~\ref{sec:HPdetailed}, identifying $0.001$ as the optimal value, and trained the model for 216 epochs. Figure~\ref{fig:eva_adafactor} illustrates the performance comparison on two distinct models: ResNet-18 with CIFAR-100 and MobileNetV3 with CIFAR10. The same data augmentation techniques were applied across all experiments, as detailed in Section~\ref{sec:HPdetailed}. The best test accuracies achieved are summarized in Table~\ref{tab:ada_eva_comp}. AdaFisher demonstrates superior performance compared to the new optimizer baselines, outperforming both EVA and AdaFactor.
\begin{table*}[!ht]
    \noindent
    \centering
    \setlength{\tabcolsep}{8pt}
    \caption{Performance comparison of AdaFisher and other optimizers using ResNet-18 (CIFAR100) and MobileNet-V3 (CIFAR10). Reported using WCT of 200 AdaFisher training epochs as the cutoff. }
    \scriptsize
    \centering
    \begin{tabular}{c|c|c|c|c|c|c|c|c}
    
    Network|Optimizer&SGD&Adam&AdaFactor&AdaHessian&K-FAC&Eva&Shampoo&AdaFisher\\
    \midrule MobileNet-V3&$94.43$&$93.32$&$93.21$&$92.86$&$94.34$&$94.41$&$93.81$&$\mathbf{95.28}$\\
     ResNet-18& $76.56$&$75.74$&$69.45$&$71.79$&$76.03$&$76.69$&$76.78$&$\mathbf{77.28}$ 
     
    \label{tab:ada_eva_comp}
    \end{tabular}
\end{table*}
\begin{figure}[!h]
    \centering
    \includegraphics[width=\textwidth]{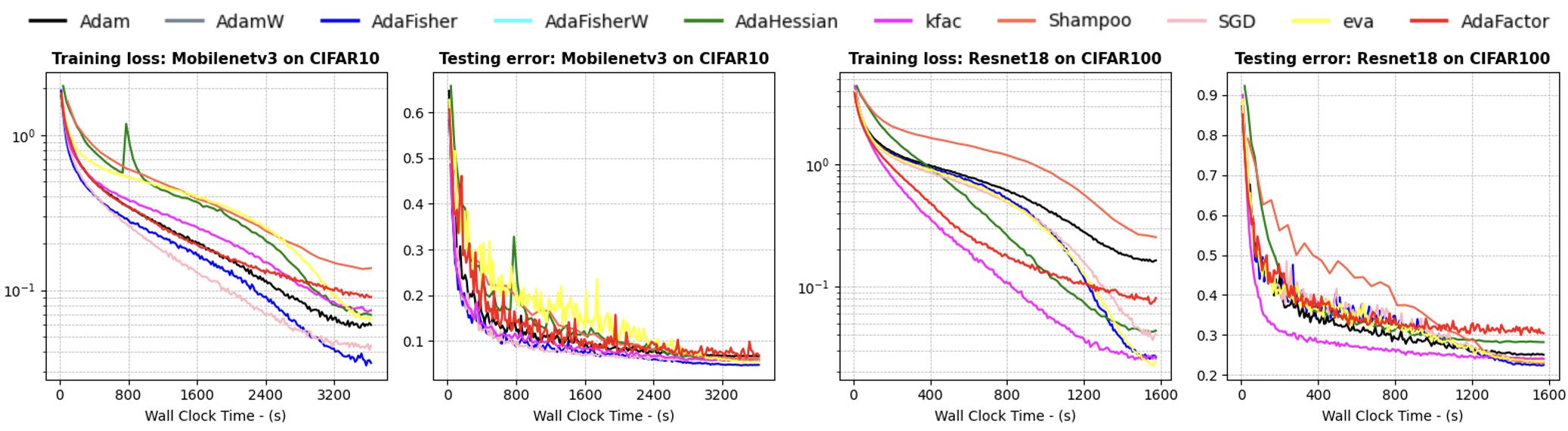}
    \caption{WCT training loss, test error, for ResNet-18 on CIFAR100 and MobileNet-V3 on CIFAR10. A batch size of 256 was used. The final accuracy and training time results are summarized in Table~\ref{tab:ada_eva_comp}.}
    \label{fig:eva_adafactor}
\end{figure}
\begin{table*}[!ht]
    
    \caption{Performance comparison of AdaFisher and other optimizers using (a) ResNet-18 and (b) MobileNet-V3 on CIFAR100 for 200 epochs.}
    \label{comp_epochs}
    \begin{subtable}{\textwidth}
    \noindent
    \setlength{\tabcolsep}{7pt}
    \caption{ResNet-18}
    \scriptsize
    \centering
    \begin{tabular}{c|c|c|c|c|c|c|c|c}
    Optimizer&SGD&Adam&AdaFactor&AdaHessian&K-FAC&Eva&Shampoo&AdaFisher\\
    \midrule Test Acc&$76.52$&$75.71$&$69.78$&$76.86$&$76.96$&$77.08$&$\mathbf{77.35}$&$77.28$\\
    Training Time (min) & $20.03$&$23.33$&$21.67$&$96.67$&$46.46$&$43.18$&$216.67$&$26.58$ 
    \end{tabular}
    \end{subtable}
    \bigskip 
    \begin{subtable}{\textwidth}
    \noindent
    \centering
    \setlength{\tabcolsep}{7pt}
    \caption{MobileNet-V3}
    \scriptsize
    \centering
    \begin{tabular}{c|c|c|c|c|c|c|c|c}
    Optimizer&SGD&Adam&AdaFactor&AdaHessian&K-FAC&Eva&Shampoo&AdaFisher\\
    \midrule Test Acc&$73.42$&$70.53$&$71.08$&$62.36$&$75.16$&$75.48$&$70.65$&$\mathbf{77.56}$\\
    Training Time (min) & $50.03$&$56.63$&$54.22$&$206.28$&$116.86$&$96.78$&$487.21$&$60.12$ 
    \end{tabular}
    \end{subtable}
    \bigskip 
    \begin{subtable}{\textwidth}
    \noindent
    \centering
    \setlength{\tabcolsep}{7pt}
    \caption{ResNet-50}
    \scriptsize
    \centering
    \begin{tabular}{c|c|c|c|c|c|c|c|c}
    Optimizer&SGD&Adam&AdaFactor&AdaHessian&K-FAC&Eva&Shampoo&AdaFisher\\
    \midrule Test Acc&$76.12$&$73.03$&$70.78$&$76.18$&$77.66$&$78.01$&$78.89$&$\mathbf{78.91}$\\
    Training Time (min) & $70.13$&$76.67$&$73.32$&$502.28$&$149.36$&$138.58$&$583.11$&$83.02$
    \end{tabular}
    \end{subtable}
\end{table*}
\subsection{Comparison with Consistent Epoch Counts}
We evaluated AdaFisher and its counterparts, including two prominent optimizers, Eva and Adafactor, over 200 epochs on ResNet-18, ResNet-50 and MobileNet-V3 using the CIFAR100 dataset. Figure~\ref{fig:comp_epochs} illustrates the training loss and test error trends over epochs, along with the best test error achieved as a function of training time per epoch for all optimizers across both models. Table~\ref{comp_epochs} summarizes the highest test accuracy and total training time for each method on both network architectures. Notably, while Shampoo achieved marginally better test accuracy than AdaFisher on ResNet-18, it required approximately eight times longer training time. Conversely, AdaFisher outperformed all baseline optimizers, including Shampoo, in the MobileNet-V3 and ResNet-50 experiments, achieving superior test accuracy while maintaining high efficiency comparable to first-order optimizers.
\begin{figure}[!h]
    \centering
    \includegraphics[width=\textwidth]{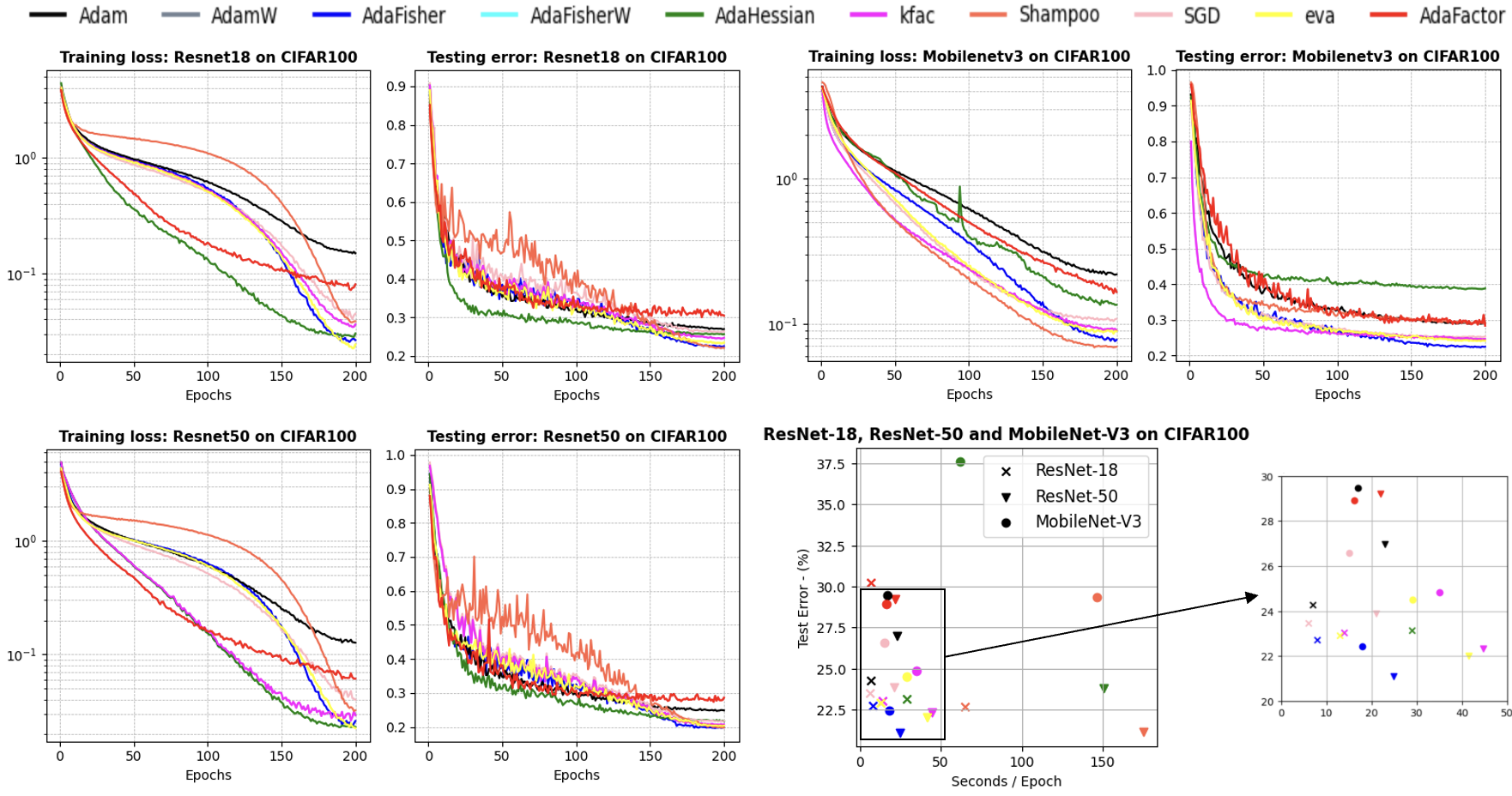}
    \caption{Performance comparison of AdaFisher and other well-finetuned optimizers at their best performances using ResNet-18 and MobileNet-V3 on CIFAR-100 for 200 epochs. A batch size of 256 was used. The final accuracy and training time results are summarized in Table~\ref{comp_epochs}.}
    \label{fig:comp_epochs}
\end{figure}
\section{AdaFisher for Transfer Learning}
Following a more sustainable practice of training DNNs, we employ pretrained models from ImageNet-1k in PyTorch on datasets like CIFAR10 and CIFAR100 to showcase AdaFisher's generalization capability for transfer learning. We applied these pretrained weights across various CNN architectures to train on these datasets. The results, presented in Table~\ref{pretrained_cifar} and Figure~\ref{fig:results_pretrained}, highlight the significant advantages of using AdaFisher, consistently achieving top accuracy across both datasets. 
\begin{table*}[!ht]
    \centering
    \caption{Performance comparison of different networks and optimizers on CIFAR10 and CIFAR100 using ImageNet-1k pretrained weights. Evaluation is based on wall clock time of 50 training epochs with AdaFisher.}
    \scriptsize
    \setlength{\tabcolsep}{0.001pt}
    \begin{tabular}{c|c|c|c|c|c|c||c|c|c|c|c|c}
    &\multicolumn{6}{c||}{CIFAR10}&\multicolumn{6}{c}{CIFAR100}\\ \cline{2-13} Network&SGD&Adam&AdaHessian&K-FAC&Shampoo&AdaFisher&SGD&Adam&AdaHessian&K-FAC&Shampoo&AdaFisher\\
    \midrule
    ResNet50 &$96.50_{0.2}$& $96.45_{0.2}$ & $96.35_{0.3}$ & $96.45_{0.1}$ & $96.03_{0.4}$ & $\mathbf{97.13_{0.2}}$ &$82.12_{0.1}$& $82.01_{0.4}$ & $80.64_{0.9}$ & $80.55_{0.4}$ & $81.70_{0.2}$ & $\mathbf{82.23_{0.2}}$ \\
    ResNet101 &$97.07_{0.2}$& $96.70_{0.1}$ & $96.65_{0.2}$ & $96.84_{0.1}$ & $96.63_{0.1}$ & $\mathbf{97.22_{0.1}}$ &$84.01_{0.1}$& $82.43_{0.2}$ & $81.36_{0.8}$ & $82.26_{0.3}$ & $82.65_{0.2}$ & $\mathbf{84.47_{0.2}}$ \\
    DenseNet121 & $94.80_{0.1}$&$94.77_{0.1}$ & $93.08_{0.1}$ & $94.41_{0.2}$ & $94.76_{0.1}$ & $\mathbf{95.03_{0.1}}$  & $75.98_{0.2}$& $75.65_{0.3}$ & $71.06_{0.9}$ & $76.10_{0.3}$ & $76.08_{0.2}$ & $\mathbf{76.92_{0.3}}$ \\
    MobileNetV3 & $91.76_{0.3}$& $90.92_{0.3}$ & $86.45_{2.5}$ & $91.72_{0.2}$ & $91.39_{0.3}$ & $\mathbf{92.78_{0.2}}$  & $71.86_{0.4}$& $66.11_{0.8}$ & $59.69_{2.3}$ & $69.85_{0.4}$ & $68.87_{0.3}$ & $\mathbf{72.38_{0.4}}$ 
    \label{pretrained_cifar}
    \end{tabular}
\end{table*}
\section{AdaFisher for Natural Language Modeling}
\begin{wraptable}[8]{r}{0.3\textwidth}
\vskip -0.16in
  \caption{Language Modeling performance (PPL) on Wikitext-2 and PTB test dataset (lower is better).}
  \label{tab:wikitext2}
  \centering
  \scriptsize
    \setlength{\tabcolsep}{1pt}
      \begin{tabular*}{\linewidth}{@{\extracolsep{\fill}}lcc}
    \toprule
    Optimizer & \multicolumn{2}{c}{Test PPL} \\
    \cmidrule{2-3}
    & WikiText-2 & PTB \\
    \midrule
    AdamW     & $175.06$ & $44.70$ \\  
    AdaHessian & $407.69$ & $59.43$ \\ 
    Shampoo   & $1727.75$ & $-$ \\ 
    \midrule
    AdaFisherW  & $\mathbf{152.72}$ & $\mathbf{41.15}$ \\ 
    \bottomrule
\end{tabular*}
\end{wraptable}
We employ the WikiText-2 dataset, which encompasses approximately 100 million tokens derived from over 2 million words extracted from a curated set of ``Good'' and ``Featured'' articles on Wikipedia. Additionally, we utilize the PTB dataset, renowned for its extensive collection of English words with part-of-speech tags, which has been widely used in NLP tasks for training and benchmarking language models. Our experiments utilize a scaled-down version of GPT-1 \citep{radford2019language}, featuring four self-attention layers with masking capabilities with more than 28 million learnable parameters. More details about tuning HPs and models can be found in Section~\ref{sec:HPdetailed}. The perplexity (PPL) on the test set, corresponding to the best-performing model during validation, is documented in Table~\ref{tab:wikitext2} and Figure~\ref{fig:stability_cifar100}. Similar to approaches in image classification, we apply the WCT method with 50 epochs training time of AdaFisher as the cutoff period. Notice that Shampoo did not achieve convergence despite using optimal HPs, and the K-FAC was unable to train with ASDL library \citep{osawa2023asdl}.
\begin{figure}[!h]
    \centering
    \includegraphics[width=\textwidth]{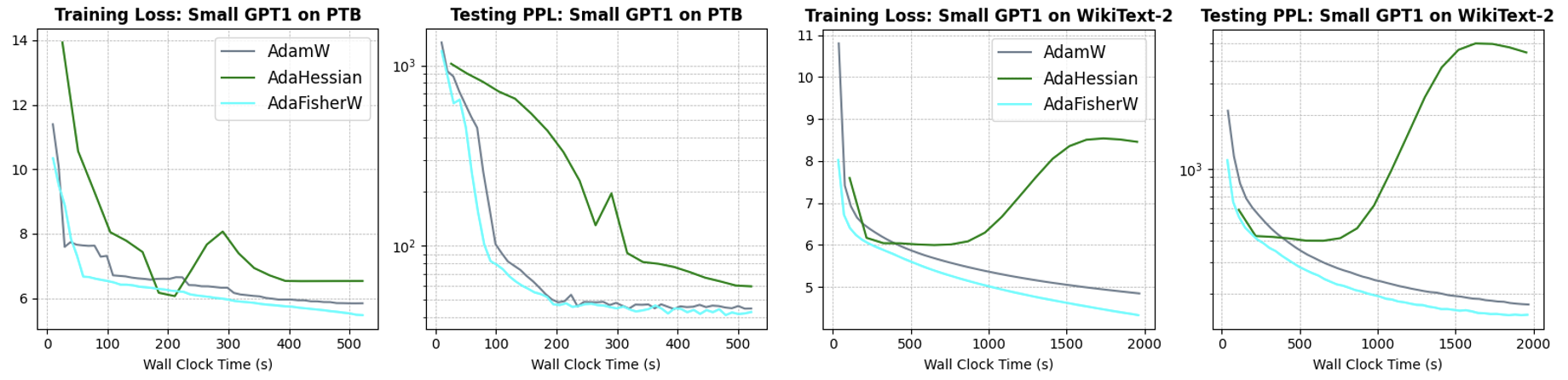}
    \caption{Training Loss and Test Perplexity of Small GPT-1 Model on WikiText-2 and PTB Datasets. Experiments were conducted using a batch size of 32 and optimal settings for all optimizers.}
    \label{fig:lm_results}
\end{figure}
\section{Concluding Remarks}
\label{sec:concluding_experiments}

In this chapter, we have presented a comprehensive experimental evaluation of AdaFisher across a wide range of tasks in both computer vision and natural language processing, covering image classification (CIFAR-10, CIFAR-100, Tiny ImageNet, and ImageNet-1k) and language modeling (WikiText-2 and PTB). Our results consistently highlight AdaFisher’s effectiveness in comparison to several baseline optimizers. One important observation is that AdaFisher achieves high accuracy across diverse model architectures and datasets. When trained on CIFAR-10 and CIFAR-100 using convolutional neural networks (ResNet-18, ResNet-50, ResNet-101, DenseNet121, MobileNetV3) and vision transformers (Tiny Swin, FocalNet, CCT-2/3$\times$2), AdaFisher and its weight-decay variant outperform both classical approaches like SGD and Adam, as well as second-order methods such as K-FAC, Shampoo, and AdaHessian. This pattern persists on Tiny ImageNet and ImageNet-1k, where AdaFisher demonstrates not only higher final accuracy but also robust convergence. Despite being a curvature-aware method, AdaFisher exhibits training times that remain comparable to first-order optimizers, especially when contrasted with other second-order techniques.

On practical tasks where methods like K-FAC or Shampoo can become prohibitively expensive, AdaFisher balances curvature estimation with computational efficiency, allowing it to scale well on multi-GPU setups and handle large batch sizes while retaining strong generalization. Another notable feature of AdaFisher is its tendency to locate flatter minima, as evidenced by superior test accuracy and smaller generalization gaps. 

The transfer learning experiments, which fine-tuned ImageNet-1k pretrained networks (ResNet-50, ResNet-101, DenseNet121, MobileNetV3) on smaller datasets (CIFAR-10, CIFAR-100), further highlight this characteristic by effectively leveraging partially pretrained weights and refined Fisher Information Matrix updates. 

In the language modeling domain, our scaled-down GPT-1 experiments on WikiText-2 and PTB show that AdaFisherW achieves significantly reduced perplexity, surpassing second-order baselines that struggled to converge. Extending our comparisons, we also evaluated AdaFisher against recent optimizers such as EVA and AdaFactor, and AdaFisher maintains superior or on-par performance, confirming that approximate curvature updates can be integrated without severely impacting memory or computational overhead. Altogether, our experiments confirm that AdaFisher outperforms established first- and second-order optimizers while remaining practical. 

By incorporating a refined form of the Fisher Information Matrix into an adaptive framework, AdaFisher stands out as a strong choice for training large-scale neural networks in both vision and language tasks. Potential future directions include deeper exploration of curvature approximation techniques in relation to flat minima, expanded distributed implementations for extremely large models, and integration with advanced data augmentation or regularization strategies. These results highlight AdaFisher’s capacity to leverage curvature information while preserving computational viability, establishing it as a strong contender for next-generation second-order optimization in deep learning.
\chapter{Dissecting Performance: Ablative and Stability Analysis of AdaFisher} \label{chap:ablativestudies}
In this chapter, we systematically examine the inner workings and robustness of the AdaFisher optimizer. Here, we investigate the impact of several key components that distinguish AdaFisher from its predecessors. Our ablation studies evaluate the effect of different learning rate schedulers, convergence efficiency, the role of removing the square root transformation in the update rule, the use of an EMA for KFs, and the integration of FIM computation in normalization layers. By scrutinizing these elements individually, we aim to elucidate how each contributes to AdaFisher’s superior convergence behavior and generalization performance, all while maintaining efficiency across diverse training environments.
\section{Ablation Studies} \label{sec:Ablationstudies}
This section explores the ablation study of AdaFisher to investigate the impact of various learning rate schedulers and convergence efficiency, as discussed in Section~\ref{sec:lrschedulersCE}. Additionally, we conduct an in-depth examination of the key components of AdaFisher. This includes analyzing the effects of the EMA, the use of square root transformations, our novel approximation of the FIM, and the critical role of computing the FIM for normalization layers, all of which are detailed in Section~\ref{sec:componentsAdaFisher}.
\subsection{Evaluating Stability Across Learning Rate Schedulers, and Assessing Convergence Efficiency} \label{sec:lrschedulersCE}
\paragraph{Learning rate schedulers.} This analysis evaluates the impact of different learning rate schedulers--Cosine Annealing, StepLR, and no scheduler—on the performance of AdaFisher, as depicted in Figure~\ref{fig:stability_schedulers_squareroot}. AdaFisher exhibits remarkable robustness across these scheduling strategies. Notably, its performance remains stable and efficient, whether it is paired with the gradual adjustments of Cosine Annealing, the abrupt changes of StepLR, or even in the absence of any scheduler. This underscores AdaFisher’s adaptability and effectiveness in diverse training environments.
\paragraph{Convergence Efficiency.} As training progresses, AdaFisher optimizer demonstrates a significant enhancement in performance compared to its counterparts, especially evident towards the end of the training period (see Appendix~\ref{chap:results}). This rapid convergence is attributed to AdaFisher's approach by incorporating the FIM. Early and mid-training, the FIM serves as an approximation to the Hessian matrix, equivalent to the Generalized Gauss Newton Matrix~\citep{eschenhagen2024kronecker}. However, as the model approaches a local minimum, the FIM increasingly aligns precisely with the Hessian~\citep{martens2020new}. This precise alignment accelerates convergence, markedly improving the optimizer's efficiency in the final phases of training. Additionally, AdaFisher’s tendency to converge to flat local minima leads to more stable generalization when transitioning from training to testing distributions~\citep{cha2021swad}, contrasting sharply with other optimizers. To support these points, we analyze the training distribution of our diagonal block-Kronecker FIM during the training of ResNet18 on CIFAR10. Specifically, we examine the FIM distribution for the first (Panel A), middle (Panel B) convolutional layers and the last linear layer (Panel C), as shown in Figure~\ref{fig:hist_fisher} In Section~\ref{sec:distributedImplementation}. It is evident that for each layer, the FIM distribution with AdaFisher narrows to smaller values with fewer variations compared to that with Adam. This pattern demonstrates AdaFisher’s convergence toward flatter local minima, as the Fisher Information, an approximation of the Hessian, contains crucial curvature information. 
\begin{figure}[!h]
    \centering
    \includegraphics[width=\textwidth]{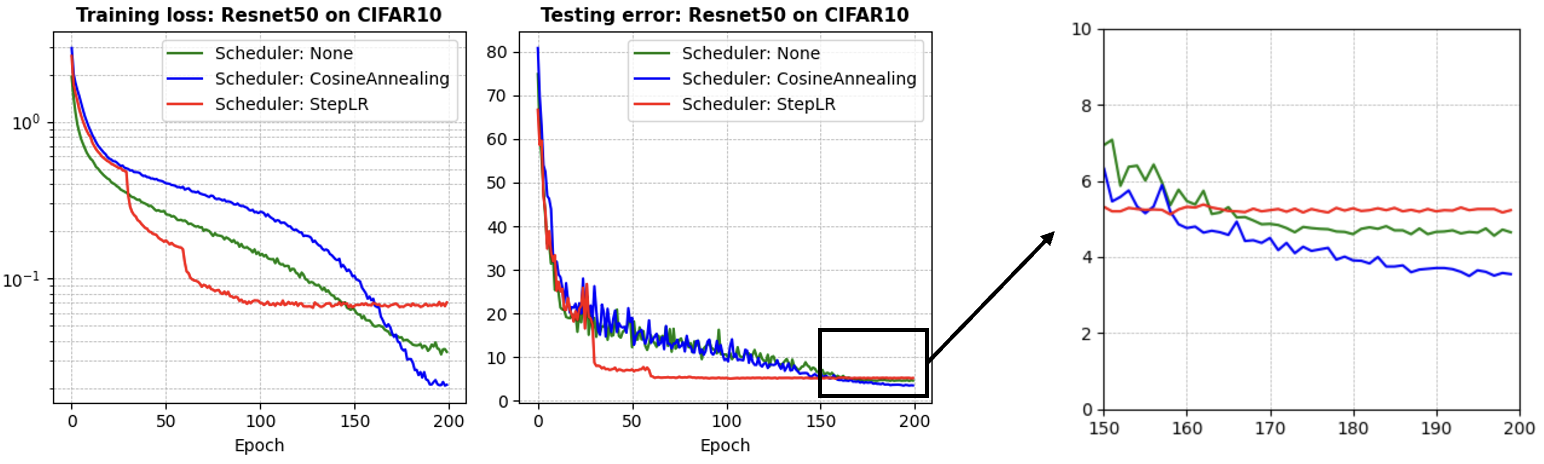}
    \caption{Performance comparison of AdaFisher using the ResNet50 on the CIFAR10 with a batch size of 256 with different learning rate schedulers.}
    \label{fig:stability_schedulers_squareroot}
\end{figure}
\begin{figure}[!h]
    \centering
    \includegraphics[width=\textwidth]{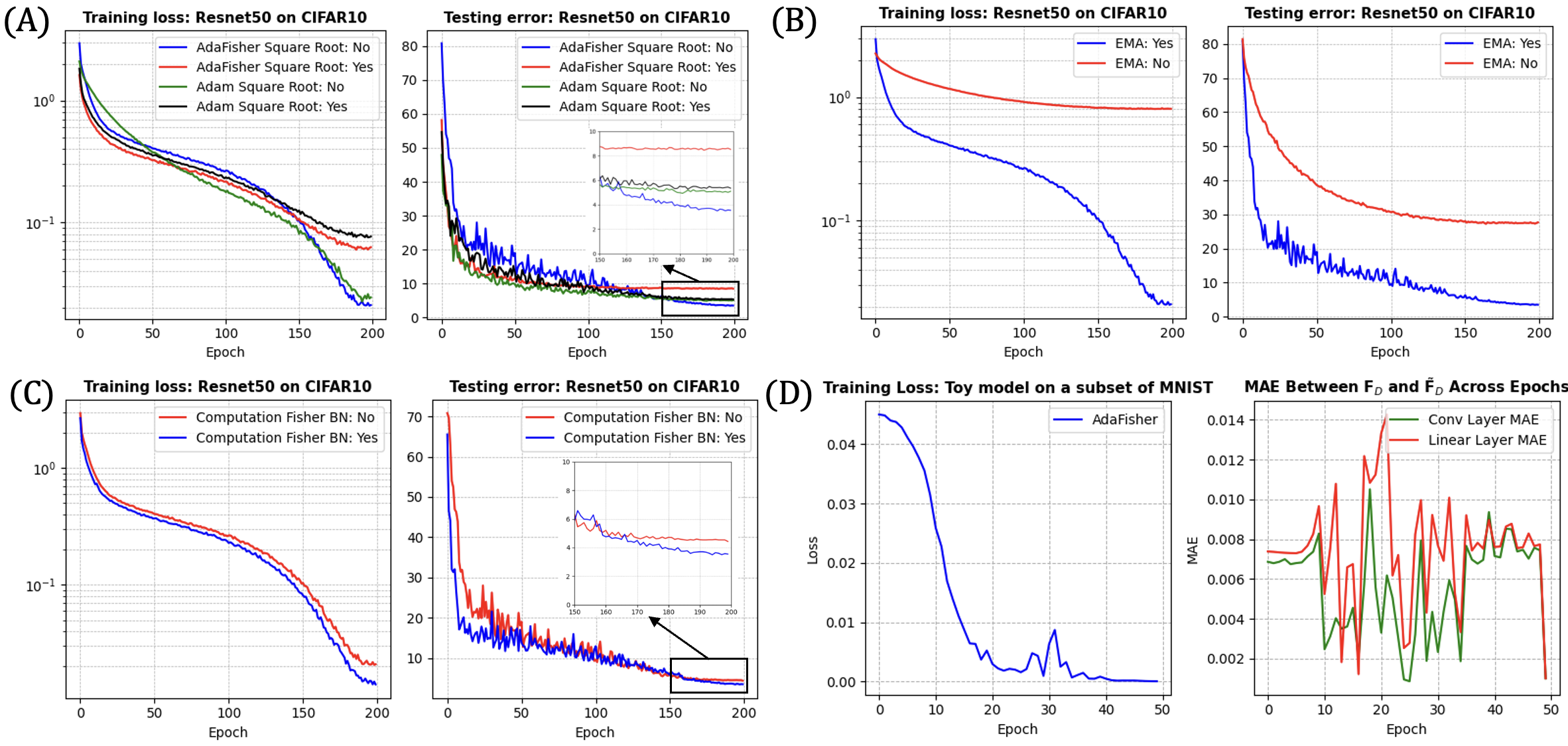}
    \caption{AdaFisher Component Analysis. (A)  Comparison of MAE between the true FIM $F_D$ and our approximation $\tilde{F}_D$ across convolutional and dense layers. (B) Performance comparison of AdaFisher with and without the EMA of KFs. (C) Assessment of AdaFisher's performance with and without the computation of EFIM for Batch Normalization (BN) layers.}
    \label{fig:AdaFisher_elements}
\end{figure}
\subsection{Component Analysis: Evaluating the Significance of AdaFisher's Elements} \label{sec:componentsAdaFisher}
AdaFisher incorporates several key components, including a novel approximation of the FIM, the EMA of the KFs, the omission of the square root in the update rule, and a new EFIM formula for normalization layers. In this part, we elucidate each component and its significance within the AdaFisher optimizer.
\paragraph{Square Root Utilization.} Recent studies, such as \citep{lin2024can}, have reevaluated the necessity of the square root operation in the Adam family's update rules. These studies suggest that eliminating the square root does not affect convergence and may even narrow the generalization gap compared to SGD in CNN models. Our analysis, shown in panel (A) of Figure~\ref{fig:AdaFisher_elements}, investigates this aspect by comparing the performance of AdaFisher and Adam, both with and without the square root operation. The findings reveal that removing the square root not only boosts the performance and stability of both optimizers but also significantly enhances computational efficiency. Specifically, AdaFisher without the square root not only outperforms the version with the square root but also surpasses Adam without the square root. However, Adam without the square root typically requires an additional \textbf{scaling factor} proportional to the batch size, denoted as $f \propto \text{batch size}$, to function correctly. Without this factor, Adam, without the square root, fails to learn effectively, making direct comparisons with AdaFisher invalid.
\paragraph{EMA of KFs.} As elucidated in Section~\ref{sec:effcomputFIM}, employing an EMA over the KFs facilitates a more sophisticated curvature estimation. This technique leverages data across multiple mini-batches, enabling continuous updates to the Fisher information rather than relying solely on the data from a single batch. Panel (B) of Figure~\ref{fig:AdaFisher_elements} underscores, using ResNet-50 on CIFAR10 over 200 epochs, the benefits of using EMA on KFs, a strategy particularly advantageous in methods that utilize diagonal or block-diagonal approximations of the curvature matrix.
\paragraph{Importance of Fisher Computation for Normalization Layers.} The integration of the EFIM in normalization layers, as detailed in Proposition~\ref{prop:proposition_normalization}, significantly enhances the generalization process. Panel (C) of Figure~\ref{fig:AdaFisher_elements} illustrates the impact of incorporating Fisher computation in these layers during the training of AdaFisher with ResNet-50 on CIFAR10 over 200 epochs. In contrast, the identity matrix is employed when Fisher's computation is omitted. The superior performance of AdaFisher when incorporating Fisher computation can be attributed to the critical role normalization layers play in adjusting the input distribution for each mini-batch. This adjustment substantially enhances the NN's learning stability \citep{jiang2024pre}. By quantifying the information each output \( \mathbf{y} \) carries about the parameters \( \boldsymbol{\theta} \) under the model distribution \( p(\mathbf{y}|\mathbf{x}; \boldsymbol{\theta}) \), the computation of the FIM in these layers provides valuable insights into parameter sensitivity and gradient variability. This insight is crucial for optimizing training dynamics and enhancing model convergence—areas that are often inadequately addressed by existing optimizers.
\paragraph{New Approximation of the FIM.}
In Proposition~\ref{prop:proposition1}, we introduce a new methodology for approximating the FIM that diverges from the K-FAC optimizer. Unlike K-FAC, which utilizes the full Kronecker product, our approach focuses solely on the diagonal elements of the FIM, where, as demonstrated in Section~\ref{sec:kroneckerdetail}, the energy of the KFs is predominantly concentrated. This method enables a more efficient computation of the FIM without sacrificing critical information. To validate our approach, we compare the true FIM diagonal with our approximation in convolutional and dense layers using a toy model composed of 2 convolutional layers and two linear layers on a subset of the MNIST dataset \citep{6296535} over 50 epochs. Specifically, we calculate the true Fisher using the NNgeometry Python package \citep{george_nngeometry}, which facilitates the computation of the FIM, Gauss-Newton Matrix, or NTK applied to neural networks. We estimate $p(\mathbf{y}|\mathbf{x})$ through Monte-Carlo sampling. During each epoch, we collected both the empirical and true Fisher information and calculated the Mean Absolute Error (MAE) between these two measures. Panel (D) of Figure~\ref{fig:AdaFisher_elements} showcases the close approximation of AdaFisher's empirical diagonal to the true Fisher, thus validating the efficacy of our approximation method.
\begin{figure}[!h]
    \centering
    \includegraphics[width=\textwidth]{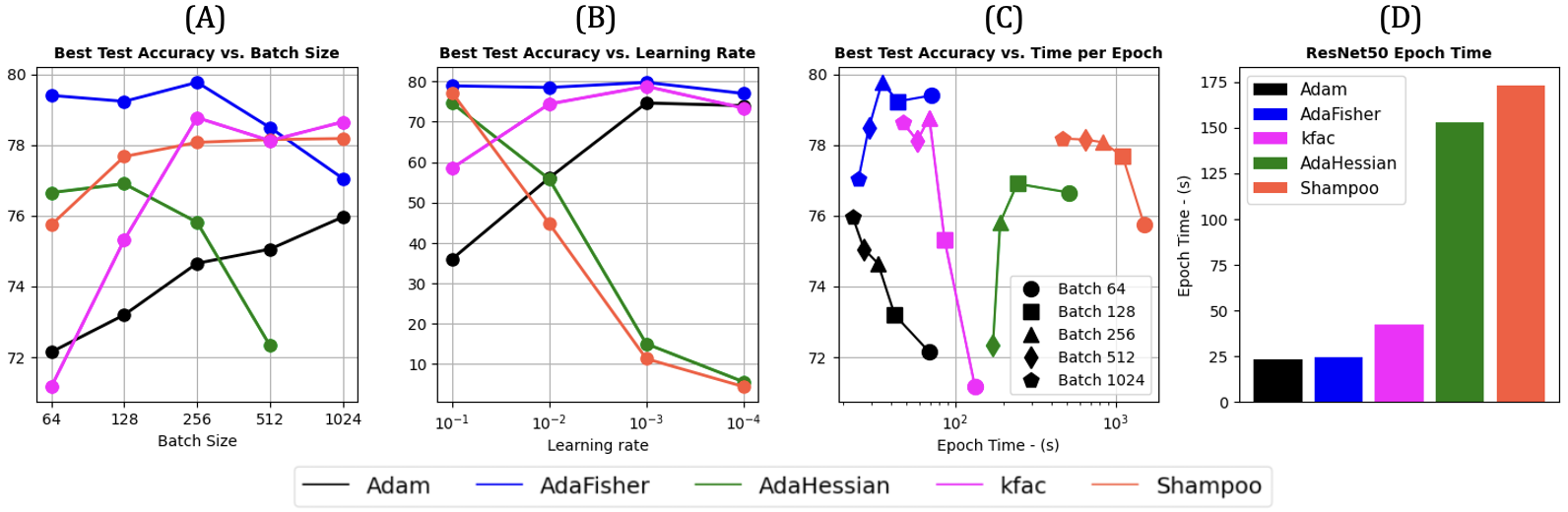}
    \caption{Performance comparison of AdaFisher and other optimizers using the ResNet50 network on the CIFAR100 dataset. (A) Test accuracy by batch size. (B) Accuracy vs. learning rates. (C) Accuracy related to epoch time across batch sizes. (D) Epoch time for different optimizers with a batch size of 256.}
    \label{fig:stability_cifar100}
\end{figure}
\section{Stability Analysis} \label{sec:stabilityanalysis}
\subsection{HP Sensitivity Analysis }
In this section, we assess AdaFisher's stability under varying learning rates and batch sizes using ResNet50 on CIFAR100 and compare its performance to other optimizers. Improved stability indicates a reduced need for HP tuning while maintaining high performance. To ensure a fair comparison, all methods were evaluated using a consistent experimental setup, with parameters tailored to each optimizer's strengths. However, we exclude AdaHessian results for a batch size of 1024 due to its significant computation cost.

\textbf{Batch Size Analysis.} We examine the impact of batch size on AdaFisher’s performance, as shown in Panels (A) and (C) of Figure~\ref{fig:stability_cifar100}. AdaFisher maintains high test accuracy across various batch sizes, excelling particularly at smaller sizes despite some sensitivity to larger ones. Panel (C) highlights AdaFisher’s efficiency, achieving high accuracy with shorter epoch times compared to Adam, detailed further in Panel (D), where AdaFisher shows competitive epoch durations against other optimizers. These results, discussed in Section~\ref{sec:complexitycost}, underscore AdaFisher’s effective performance across batch size variations without adjusting other HPs.

\textbf{Learning Rate Stability.} This analysis evaluates the impact of learning rate variations on AdaFisher's performance, as depicted in Panel (B) of Figure~\ref{fig:stability_cifar100}. AdaFisher demonstrates superior stability, particularly at lower learning rates, maintaining consistent performance across a broad spectrum. This stability alleviates the need for meticulous learning rate adjustments, thereby streamlining model training in various computational environments. Additionally, AdaFisher's stability across various learning rates can be attributed to its effective approximation of the curvature matrix.

\subsection{Comparison of Training Speed and Memory Utilization}
\label{sec:complexitycost}
As discussed in Section~\ref{sec:stabilityanalysis}, AdaFisher emerges as a balanced trade-off between time complexity and performance. Similarly, its memory footprint is comparable to that of Adam, showcasing efficient VRAM utilization. We extend our stability analysis to the CIFAR10 dataset to provide a dataset-independent evaluation of performance metrics, as depicted in panel (A) of Figure~\ref{fig:stability_cifar10}. Additionally, we analyze the memory usage for different batch sizes using the ResNet-50 model on the CIFAR-10/100, presented in panel (B) of Figure~\ref{fig:stability_cifar10}. The analysis reveals that AdaFisher while maintaining high accuracy levels, uses memory comparably to Adam, especially evident in larger batch sizes. This suggests that AdaFisher can achieve competitive performance without excessive VRAM consumption, making it an optimal choice for scenarios with memory constraints.
\begin{figure}[!h]
    \centering
    \includegraphics[width=\textwidth]{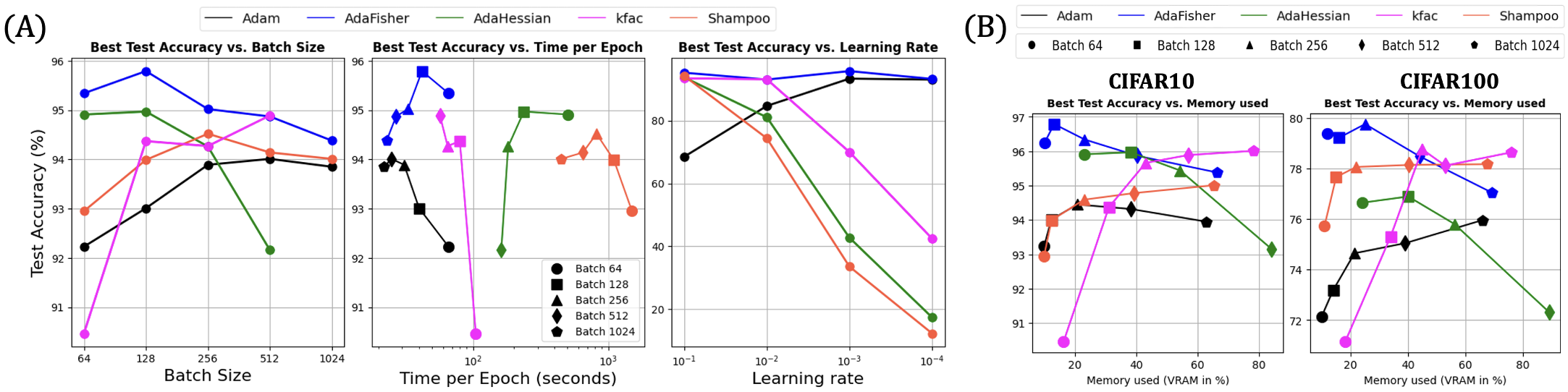}
    \caption{(A) Performance comparison of AdaFisher and other optimizers across various batch sizes, epoch times and learning rates (with a batch size of 256), evaluated using the ResNet50 on the CIFAR-10. (B) Performance comparison of AdaFisher and other optimizers regarding the memory used, assessed using ResNet50 and CIFAR10/100 across different batch sizes. This figure highlights how AdaFisher competes closely with Adam in terms of memory efficiency and performance.}
    \label{fig:stability_cifar10}
\end{figure}
\begin{figure}[!h]
    \centering
    \includegraphics[width=\textwidth]{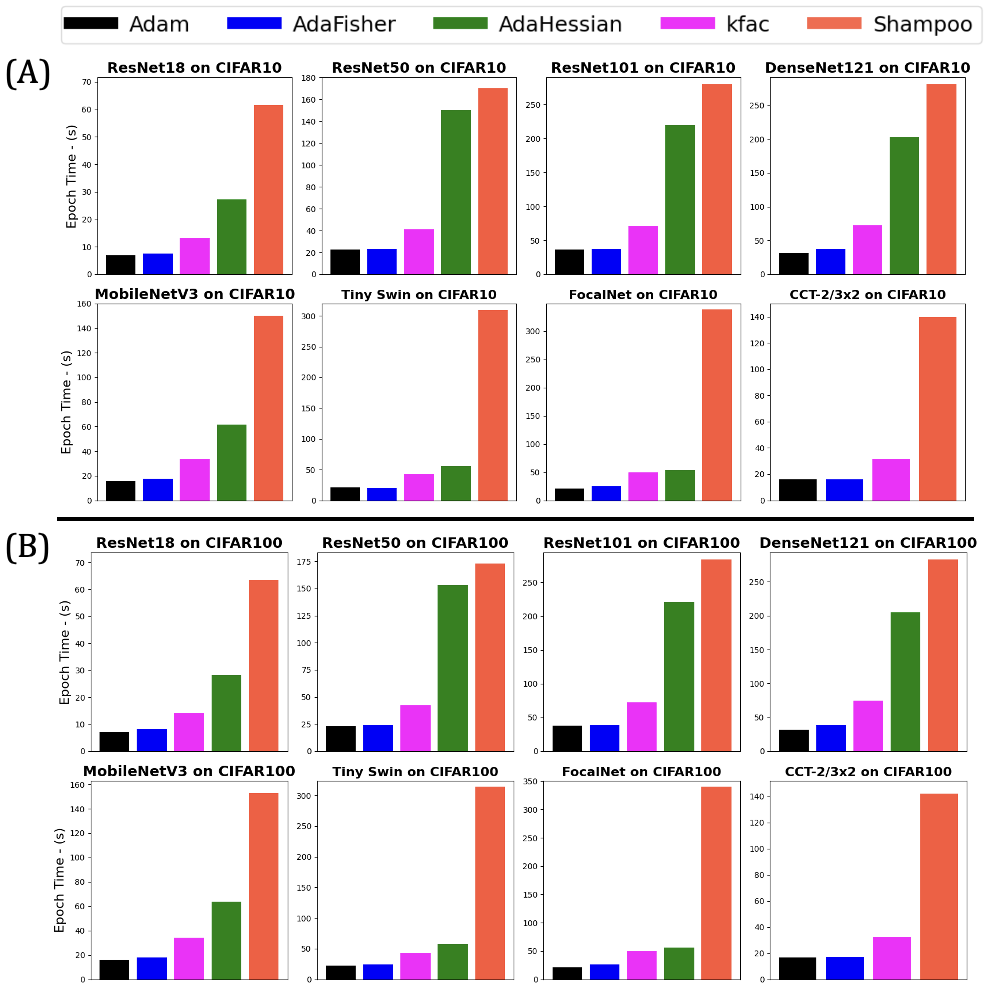}
    \caption{Epoch times for various networks on CIFAR10 (A) and CIFAR100 (B) using Adam, AdaFisher, K-FAC, AdaHessian and Shampoo.}
    \label{fig:bar_plot_time}
\end{figure}

\textbf{Epoch Times.} Continuing our analysis of the time complexity for each optimizer, we present in Figure~\ref{fig:bar_plot_time}  the epoch times for various network architectures and datasets. Specifically, we compare the epoch times of Adam, AdaFisher, K-FAC, AdaHessian, and Shampoo optimizers on CIFAR10 and CIFAR100 datasets. As depicted in Figure~\ref{fig:bar_plot_time} panel (A), AdaFisher demonstrates a comparable training time to Adam across multiple network architectures on the CIFAR10 dataset. This indicates that AdaFisher achieves efficient optimization without incurring significant additional computational costs. Similarly, in Figure~\ref{fig:bar_plot_time} panel (B), we observe that the epoch times for AdaFisher remain close to those of Adam on the CIFAR100 dataset. While K-FAC and AdaHessian exhibit increased training times, Shampoo shows the highest epoch times across all tested networks. This further highlights the efficiency of AdaFisher as an optimizer, combining the advantages of advanced optimization techniques with practical training times.
\section{Concluding Remarks}
In summary, our comprehensive ablative and stability analyses demonstrate that AdaFisher achieves robust and efficient performance through a judicious integration of adaptive and curvature-aware techniques. The experiments reveal that AdaFisher is resilient across various learning rate schedulers, and its refined diagonal block-Kronecker approximation of the FIM significantly enhances convergence efficiency, particularly in the later training stages. Notably, eliminating the square root in the update rule and incorporating an EMA over the KFs yield substantial improvements in both performance and computational efficiency. Additionally, the inclusion of FIM computation in normalization layers further stabilizes training and fosters better generalization by steering the optimizer towards flatter minima. Overall, these findings not only validate the design choices underlying AdaFisher but also underscore its potential as a scalable and practical optimizer for modern deep learning applications. The insights provided here offer a promising roadmap for future research in adaptive second-order optimization.
\chapter{Concluding Remarks} \label{chap:conclusion}
\section{Overview of Chapters}
This thesis has explored the challenge of making second-order optimization practical for deep learning, focusing on the FIM as a source of curvature information. 
Chapter~\ref{chap:introduction} introduced the motivation and problem statement, highlighting the gap between the superior convergence properties of second-order methods and their impracticality in modern DNN training. Chapter~\ref{chap:literaturereview} provided a literature review of optimization techniques, from classical first-order methods to advanced second-order approaches, identifying the central trade-off: first-order methods are efficient but capture limited curvature information, whereas second-order methods use richer curvature information for faster convergence but are often infeasible at scale.
Chapter~\ref{chap:efficientfisherapprox} developed a new, efficient approximation of the FIM suitable for large-scale DNNs. By analyzing the FIM's eigen-spectrum and sparsity patterns, we revealed that much of the curvature information can be retained through structured approximations. The chapter introduced a block-diagonal (Kronecker-factored) approximation to the FIM, naturally arising from layer-wise factorization, and demonstrated that a diagonal-centric structure in the FIM emerges during training, justifying using primarily the diagonal of each factor as an efficient proxy for curvature.
Chapter~\ref{chap:adafisher} introduced AdaFisher, a novel adaptive second-order optimizer leveraging our FIM approximation. AdaFisher employs a refined diagonal block-Kronecker approximation of the FIM as its preconditioner, effectively integrating curvature information into each update step without incurring the full cost of second-order methods. The chapter detailed AdaFisher's theoretical underpinnings, algorithmic structure, and distributed implementation. We provided theoretical analysis showing that AdaFisher's adaptive nature inherits stability similar to first-order methods while its FIM-based scaling guides the model toward minima that generalize better.
Chapter~\ref{chap:experiments} presented extensive experiments validating AdaFisher across computer vision tasks (image classification with CNNs and Vision Transformers) and natural language tasks (language modeling with transformers). Results demonstrated that AdaFisher consistently converged faster and reached lower testing error than baselines including Adam, AdaHessian, K-FAC, and Shampoo. AdaFisher exhibited stability and robustness to hyper-parameter tuning, performing consistently well over a range of settings while using standard learning rate schedules.
Chapter~\ref{chap:experiments} provided an in-depth ablation and stability analysis of AdaFisher's design. We confirmed that each key design choice is justified: the adaptive square-root scaling, exponential moving average for Kronecker factors, and FIM-based preconditioning in all layers all contributed to better outcomes. AdaFisher maintained strong performance under different batch sizes and learning rate schedules. Crucially, models trained with AdaFisher had significantly lower effective curvature than those trained with conventional optimizers, indicating that AdaFisher steers optimization toward broader, flatter minima known to correlate with improved generalization.

\section{Solving Deep Learning Optimization Problem with AdaFisher}
The core problem addressed by this thesis is how can we leverage the second order information in an efficient way to both achieve better generalization while maintaining bounded computation overhead. AdaFisher tackles this by using the Fisher Information Matrix as a guide for optimization. Theoretically, AdaFisher can be viewed as an adaptive natural gradient method: it preconditions the gradient with an estimate of the inverse Fisher matrix, effectively normalizing gradient directions by local curvature. This means AdaFisher \textbf{takes smaller steps where loss curvature is steep and larger steps where the surface is flat}. This curvature-aware scaling is similar to second-order optimizers but with far less computational overhead. AdaFisher integrates the adaptive moment estimation machinery of Adaptive framework, marrying the efficiency of first-order updates with the precision of second-order curvature information.

From a theoretical standpoint, AdaFisher's use of the FIM grants several advantages over purely first-order methods. The Fisher matrix is positive semi-definite and approximates the Hessian in many cases, helping AdaFisher avoid issues of negative curvature and ill-conditioning that plague Hessian-based methods. AdaFisher's update rule performs gradient descent in a whitened space where coordinate scaling is determined by curvature, yielding more balanced progress across all dimensions of parameter space. By steering updates towards regions favored by the FIM's structure, AdaFisher introduces implicit regularization towards flatter regions of the loss landscape, explaining the improved generalization observed empirically.
Experimentally, AdaFisher outperforms other optimizers across various benchmarks. It not only speeds up convergence but also finds solutions with lower test error. On ImageNet and CIFAR benchmarks, AdaFisher consistently achieved lower error rates than Adam and outperformed specialized second-order methods. Unlike other second-order methods which often require large batch sizes and careful tuning, AdaFisher remains robust with standard batch sizes and learning rate schedules.

Compared to Adam, AdaFisher provides more informed preconditioning based on expected curvature rather than gradient magnitudes. Compared to AdaHessian, AdaFisher's use of the Fisher matrix is better suited for classification problems where the Hessian may not be positive-definite. Versus K-FAC and Shampoo, AdaFisher is considerably more lightweight, storing only per-parameter statistics similar to Adam in memory cost, enabling scaling to larger models with far less overhead. AdaFisher also demonstrated efficient distributed training, making it practical for real-world large-scale applications.

\section{Future Directions}
While AdaFisher makes significant progress towards practical second-order optimization, several avenues for improvement and further research remain:
Richer Fisher Approximation (\textbf{Band-Diagonal FIM}): One immediate enhancement would be to improve the approximation of the Fisher Information Matrix. Currently, AdaFisher relies on the diagonal of each Kronecker factor to construct its preconditioner. A natural extension would be to use a band-diagonal approximation instead, capturing a band of off-diagonal entries around the main diagonal to account for limited correlations between parameters. This approach would bridge the gap between the sparse diagonal approximation and the full dense matrix: a small bandwidth could significantly increase accuracy of curvature representation with only a modest increase in computation. Future work can analyze the trade-off between bandwidth and performance gains, and develop algorithms to efficiently update and invert these banded approximations.
Informative Neurons for Parameter-Efficient Fine-Tuning: AdaFisher's insights could be applied to parameter-efficient transfer learning. The FIM can identify the most "informative" neurons or weights – those to which the loss is most sensitive – and fine-tuning could be restricted to those parameters for new tasks. This would create a systematic approach to parameter-efficient fine-tuning rather than guessing which layers to train. Research could involve computing per-neuron Fisher diagonals after training on a base task, then fine-tuning only neurons with top-tier Fisher values for target tasks. This could be combined with lightweight adaptation modules like LoRA or adapter methods, making large model deployment more efficient.

Efficiency through Low-Level Optimization (CUDA Kernels): To further improve runtime performance, future work could focus on low-level optimization. Developing custom CUDA kernels for AdaFisher's core operations could greatly speed up the optimizer on GPU hardware. By optimizing memory access patterns and parallelizing FIM computations, we could reduce per-iteration overhead to be almost on par with Adam. Integration with deep learning compiler frameworks and exploring mixed-precision implementations might further reduce memory and computation costs without affecting efficacy. These engineering improvements could make AdaFisher practically as fast as the best first-order optimizers.
Application to LLMs: An important next step is testing and refining AdaFisher on truly large-scale models, such as Transformer-based LLMs with hundreds of billions of parameters. This will likely require combining memory-efficient approximations and low-level optimizations, since storing even diagonal Fisher information for billions of parameters is non-trivial. It may also require algorithmic adaptations, such as segmenting the model into modules optimized with local curvature information. The payoff could be substantial: AdaFisher might help stabilize the fine-tuning of LLMs on downstream tasks by virtue of its curvature-guided updates. Evaluating AdaFisher on RLHF or long-horizon training could reveal whether curvature information helps navigate complex non-convex training dynamics.
Broader Research Opportunities: Additional research directions include theoretical analysis of generalization for curvature-aware optimizers, combining AdaFisher with techniques like sharpness-aware minimization, developing quantized or memory-compressed variants for enormous models, applying AdaFisher's methodology to continual learning to mitigate catastrophic forgetting, and exploring its use with advanced regularization techniques or in federated learning scenarios.

In summary, the development of AdaFisher opens many future research avenues in both algorithmic innovation and practical applications. By improving the approximation quality, efficiency, and applicability of second-order methods, we move closer to a paradigm where curvature-informed optimization becomes a standard component of deep learning training, leading to faster convergence, better generalization, and more efficient use of computational resources.

\appendix
\setcounter{table}{0}        
\setcounter{figure}{0}        
\renewcommand{\thefigure}{\Alph{chapter}.\arabic{figure}}     
\renewcommand{\thetable}{\Alph{chapter}.\arabic{table}}        

\begin{appendices}
    \chapter{Proofs} \label{chap:ApdxA}
\begin{proposition}[FIM for normalization layer]
Let $(\boldsymbol{\nu}_i, \boldsymbol{\beta}_i) \in \mathbb{R}^{C_i}$ be the scale and shift parameters of a normalization layer $i$. The empirical KFs for the FIM approximation are
\begin{align}
\mathcal{H}_{i-1}\Bigr|_{\boldsymbol{\nu}_i} = \frac{1}{|\mathcal{T}_i|} \sum_{x \in \mathcal{T}_i} \mathbf{h}_{i-1,x}\mathbf{h}_{i-1,x}^\top, \, \mathcal{H}_{i-1}\Bigr|_{\boldsymbol{\beta}_i} = \mathbf{1}\mathbf{1}^\top, \quad   
\mathcal{S}_i = \frac{1}{|\mathcal{T}_i|} \sum_{x \in \mathcal{T}_i} \mathbf{s}_{i,x}\mathbf{s}_{i,x}^\top \notag
\end{align}
where $\mathbf{h}_{i-1}, \mathbf{s}_i \in \mathbb{R}^{C_i \times |\mathcal{T}_i|}$ represent the pre-normalized activations and gradients, respectively. Here, $\mathcal{T}_i$ is the set of dimensions over which normalization statistics are computed, and $C_i$ is the channels/features size.
\end{proposition}
\begin{proof}
Let $(\boldsymbol{\nu}_i, \boldsymbol{\beta}_i) \in \mathbb{R}^{C_i}$ be the scale and shift parameters of a normalization layer $i$, with transformation 
\begin{align*}
    \mathbf{h}_i = \mathbf{a}_i = \boldsymbol{\nu}_i \odot \mathbf{h}_{i-1} + \boldsymbol{\beta}_i
\end{align*}
where $\mathbf{h}_{i-1} \in \mathbb{R}^{C_i \times |\mathcal{T}_i|}$ contains normalized activations and $\odot$ denotes element-wise multiplication. Let $\nabla_{\boldsymbol{\nu}_i}J(\boldsymbol{\theta}) = \sum_{x} \mathbf{h}_{i-1,x} \odot \mathbf{s}_{i,x}$ and $\nabla_{\boldsymbol{\beta}_i}J(\boldsymbol{\theta}) = \sum_{x} \mathbf{s}_{i,x}$ where $\mathbf{s}_i = \nabla_{\mathbf{h}_i}J(\boldsymbol{\theta})$. 

\textbf{For $\boldsymbol{\nu}_i$ parameters:}
\begin{align*}
\mathbb{E}[\nabla_{\boldsymbol{\nu}_i}\mathcal{L}\nabla_{\boldsymbol{\nu}_i}\mathcal{L}^\top] 
&= \mathbb{E}\left[\left(\sum_{x} \mathbf{h}_{i-1,x} \odot \mathbf{s}_{i,x}\right)\left(\sum_{x'} \mathbf{h}_{i-1,x'} \odot \mathbf{s}_{i,x'}\right)^\top\right] \\
&\approx \mathbb{E}\left[\sum_{x} (\mathbf{h}_{i-1,x}\mathbf{h}_{i-1,x}^\top) \otimes (\mathbf{s}_{i,x}\mathbf{s}_{i,x}^\top)\right] \quad \text{(K-FAC independence assumption)}  \\
&= \left(\frac{1}{|\mathcal{T}_i|}\sum_{x} \mathbf{h}_{i-1,x}\mathbf{h}_{i-1,x}^\top\right) \otimes \left(\frac{1}{|\mathcal{T}_i|}\sum_{x} \mathbf{s}_{i,x}\mathbf{s}_{i,x}^\top\right) \\
&= \mathcal{H}_{i-1}\Bigr|_{\boldsymbol{\nu}_i} \otimes \mathcal{S}_i 
\end{align*}

\textbf{For $\boldsymbol{\beta}_i$ parameters:}
\begin{align*}
\mathbb{E}[\nabla_{\boldsymbol{\beta}_i}\mathcal{L}\nabla_{\boldsymbol{\beta}_i}\mathcal{L}^\top] 
&= \mathbb{E}\left[\left(\sum_{x} \mathbf{s}_{i,x}\right)\left(\sum_{x'} \mathbf{s}_{i,x'}\right)^\top\right] \\
&= \mathbf{1}\mathbf{1}^\top \otimes \left(\frac{1}{|\mathcal{T}_i|}\sum_{x} \mathbf{s}_{i,x}\mathbf{s}_{i,x}^\top\right) \quad \text{(Bias term factorization)} \\
&= \mathcal{H}_{i-1}\Bigr|_{\boldsymbol{\beta}_i} \otimes \mathcal{S}_i 
\end{align*}

Cross-terms between $\boldsymbol{\nu}_i$ and $\boldsymbol{\beta}_i$ are excluded under the diagonal block assumption. 
\end{proof}
\begin{proposition}[Efficient FIM]
Let $\mathcal{H}_{i-1}$ and $\mathcal{S}_{i}$ represent the KFs for a given layer index $i$ within a neural network, where these factors exhibit semi-diagonal characteristics indicating energy concentration predominantly along the diagonal, as elaborated in Section \ref{sec:kroneckerdetail}. Define $g_{i}$ as the gradient obtained through backpropagation at layer $i$. Assume that $\mathcal{H}_{i-1}$ and $\mathcal{S}_{i}$ can be closely approximated by diagonal matrices, denoted by $\mathcal{H}_{D_{i-1}}$ and $\mathcal{S}_{D_{i}}$ respectively at layer $i$, such that $\mathcal{H}_{D_{i-1}} = \text{Diag}(\mathcal{H}_{i-1})$, $\mathcal{S}_{D_{i}} = \text{Diag}(\mathcal{S}_{i})$ where $\text{Diag}(\mathcal{M})$ denote the diagonal approximation of a matrix $\mathcal{M}$, which retains only the main diagonal. Therefore, we define the Empirical FIM as
\begin{align}
\Tilde{F}_{D_{i}} \triangleq \mathcal{H}_{D_{i-1}}' \otimes \mathcal{S}_{D_{i}}' + \lambda \mathbf{I},\label{eq:FIMDiag2}
\end{align}
where $\mathcal{M}'$ denotes the Min-Max normalization technique \cite{patro2015normalization} for \(\mathcal{M} = \mathcal{H}_{D_{i-1}}\) or \(\mathcal{S}_{D_{i}}\).  The regularization parameter \(\lambda\) set to \(0.001\) serves as damping factors, in alignment with the principles of Tikhonov regularization, to enhance computational stability and improve the conditioning of the matrix. The foundational aspects of the K-FAC optimization approach are detailed in~\cite{pmlr-v37-martens15}. Then, the closed-form solution for the preconditioned gradient $\bar{\mathbf{g}}^{(t)}$, derived from the diagonal approximation of the FIM, is given by: $\bar{\mathbf{g}}^{(t)} = (\tilde{F}_{D}^{(t)})^{-1} \mathbf{g}^{(t)}$.
\end{proposition}
\begin{proof}
The justification of our approach comprises two principal components: the rationale for adopting a diagonal approximation of the KFs and the methodology for normalization and regularization of these factors.

\textbf{Part 1: Diagonalization of KFs}

The assumption of independent neuronal activity within layers is foundational to our approach. This assumption posits that the covariance matrices \( \mathcal{H} \) and \( \mathcal{S} \), encapsulating the second-order statistics of activations and sensitivities, respectively, are diagonal. This diagonal nature arises because independence among random variables implies a covariance of zero for any pair of distinct variables, thereby nullifying all off-diagonal elements of these covariance matrices.

Consider matrices \( A \) and \( B \), each being diagonal with elements \( a_{ii} \) and \( b_{jj} \), respectively. The Kronecker product \( A \otimes B \), by definition, generates elements \( a_{ii}b_{jj} \) at the corresponding \( (i,j) \) positions. For diagonal \( A \) and \( B \), this product maintains non-zero values exclusively at diagonal positions where \( i = j \), resulting in:
\begin{align*}
    A \otimes B = \text{diag}(a_{11}b_{11}, \ldots, a_{nn}b_{mm}),
\end{align*}
yielding a purely diagonal matrix. Moreover, we have empirically demonstrated that the energy of the KFs is concentrated along the diagonal, as detailed in Section~\ref{sec:kroneckerdetail}. These arguments support our initial premise.

\textbf{Part 2: Normalization and Regularization}

Let $\mathcal{M} \in \{\mathcal{H}_{D_i}, \mathcal{S}_{D_i}\}$ be a diagonal matrix with entries $m_k > 0$. The min-max normalized matrix $\mathcal{M}'$ satisfies
\begin{align*}
\mathcal{M}' = D^{-1}(\mathcal{M} - m_{\min}I)D^{-1}, \quad D = \mathrm{diag}(\sqrt{m_{\max} - m_{\min}})
\end{align*}
where $m_{\min} = \min_k m_k$, $m_{\max} = \max_k m_k$. This affine transformation ensures that $0 \preccurlyeq \mathcal{M}' \preccurlyeq I$ where $\preccurlyeq$ denotes Loewner ordering. Combined with Tikhonov regularization, the modified FIM, $\tilde{F}_{D_i} = \mathcal{H}'_{D_{i-1}} \otimes \mathcal{S}'_{D_i} + \lambda \mathbf{I}$ admits eigenvalue bounds
\begin{align*}
\lambda \leq \lambda_k(\tilde{F}_{D_i}) \leq 1 + \lambda \quad \forall k
\end{align*}
which guarantees numerical stability for inversion. This approach ensures that all elements are scaled uniformly, preserving their relative magnitudes and distances. Compared to other normalization methods, such as z-score normalization~\citep{patro2015normalization}, Min-Max normalization offers several advantages such as the normalization Stability, where for $\mathcal{M}' = (\mathcal{M} - m_{\min}I)/(m_{\max} - m_{\min})$ we have $\sigma(\mathcal{M}') \subseteq [0,1]$ where $\sigma(\cdot)$ denotes matrix spectrum. Moreover, the Kronecker product satisfies $\sigma(\mathcal{H}'_{D_{i-1}} \otimes \mathcal{S}'_{D_i}) \subseteq [0,1]$, thus $\lambda_{\min}(\tilde{F}_{D_i}) \geq \lambda > 0$, guaranteeing invertibility. And the relative error satisfies $\frac{\|\tilde{F}_{D_i}^{-1} - F_{D_i}^{-1}\|}{\|F_{D_i}^{-1}\|} \leq \mathcal{O}(\epsilon + \lambda)$ where $\epsilon$ measures diagonal approximation error. Therefore, the preconditioned gradient can be written has $\bar{\mathbf{g}}^{(t)} = (\mathcal{H}_{D_{i-1}}' \otimes \mathcal{S}_{D_{i}}' + \lambda \mathbf{I})^{-1} \mathbf{g}^{(t)} = (\tilde{F}_{D_{i}}^{(t)})^{-1} \mathbf{g}^{(t)}$.
\end{proof}

\begin{proposition}[Convergence in convex optimization]
For the FIM defined in Eq. (\ref{eq:FIMDiag2}), the updating scheme $\boldsymbol{\theta}^{(t+1)} = \boldsymbol{\theta}^{(t)} -\alpha  (\tilde{F}_{D}^{(t)})^{-1} \nabla J(\boldsymbol{\theta}^{(t)})$ converges. Moreover, if $\nabla J$ is Lipschitz, i.e., $||\nabla J(\boldsymbol{\theta}) - \nabla J(\boldsymbol{\theta}')||_2 \leq L ||\boldsymbol{\theta} - \boldsymbol{\theta}'||$ for any $\boldsymbol{\theta} $ and $ \boldsymbol{\theta}'$, then for the $k$-step iteration with a fixed step size $\alpha\leq 1/L$, then 
\begin{equation*}
J(\boldsymbol{\theta}^{(k)}) - J(\boldsymbol{\theta}^*) \leq \frac{||\boldsymbol{\theta}^{(0)} - \boldsymbol{\theta}^*||_2^2}{2\alpha k},
\end{equation*}
where $J(\boldsymbol{\theta}^*)$ is the optimal value.  
\end{proposition}
\begin{proof}
We follow the same proof as in \cite{yao2021adahessian}. Assume that $J(\boldsymbol{\theta})$ is a strongly convex and strictly smooth function in $\mathbb{R}^d$, such that there exist positive constants $\alpha$ and $\beta$ so that $\alpha I \leq  \nabla^2 J(\boldsymbol{\theta}) \leq \beta I$ for all $w$. We can show that the update formulation $\triangle \boldsymbol{\theta}^{(t)} = (\tilde{F}^{(t)})^{-1} \mathbf{g}^{(t)}$ converges by showing that with the proper learning rate:
$$\triangle \boldsymbol{\theta}^{(t)} \coloneq J(\boldsymbol{\theta}^{(t+1)}) - J(\boldsymbol{\theta}^{(t)}) \leq  - \frac{\alpha}{2\beta^{2}}||g^{(t)}||^2$$
Note that when $k = 0$ or 1, the convergence rate is the same as gradient descent or Newton method, respectively. Our proof is similar to \cite{Boyd2004} for the Newton method. We denote $\lambda(\boldsymbol{\theta}^{(t)}) = (g^{(t)})^\top (\tilde{F}^{(t)})^{-1}\mathbf{g}^{(t)})^{1/2}$. Since $J(\boldsymbol{\theta})$ is strongly convex, we have 
\begin{align*}
    J(\boldsymbol{\theta}^{(t)} - \eta \triangle \boldsymbol{\theta}^{(t)}) &\leq J(\boldsymbol{\theta}^{(t)}) - \eta (\mathbf{g}^{(t)})^\top\triangle \boldsymbol{\theta}^{(t)} + \frac{\eta^2 \beta ||\triangle \boldsymbol{\theta}^{(t)}||^2}{2} \\
    & \leq J(\boldsymbol{\theta}^{(t)}) -\eta \lambda(\boldsymbol{\theta}^{(t)})^2 + \frac{\beta}{2\alpha}\eta^2\lambda(\boldsymbol{\theta}^{(t)})^2.
\end{align*}
The second inequality comes from the fact that 
\begin{equation*}
    \lambda(\boldsymbol{\theta}^{(t)})^2 = \triangle (\boldsymbol{\theta}^{(t)})^\top \tilde{F}^{(t)}\triangle \boldsymbol{\theta}^{(t)} \geq \alpha ||\triangle \boldsymbol{\theta}^{(t)}||^2.
\end{equation*}
Therefore, the step size $\hat{\eta} = \alpha/\beta$ will make $f$ decrease as follows,
$$J(\boldsymbol{\theta}^{(t)} - \hat{\eta} \triangle \boldsymbol{\theta}^{(t)}) - J(\boldsymbol{\theta}^{(t)}) \leq -\frac{1}{2}\hat{\eta}\lambda(\boldsymbol{\theta}^{(t)})^2.$$
Since $\alpha I\preceq \tilde{F}^{(t)}\preceq\beta I$, we have 
$$\lambda(\boldsymbol{\theta}^{(t)})^2 = (\mathbf{g}^{(t)})^\top (\tilde{F}^{(t)})^{-1}\mathbf{g}^{(t)}\geq \frac{1}{\beta}||\mathbf{g}^{(t)}||^2.$$
Therefore, 
\begin{equation}\label{eq:upper_bound1}
    J(\boldsymbol{\theta}^{(t)} - \hat{\eta} \triangle \boldsymbol{\theta}^{(t)}) - J(\boldsymbol{\theta}^{(t)})\leq -\frac{1}{2\beta}\hat{\eta}||\mathbf{g}^{(t)}||^2 = -\frac{\alpha}{2\beta^{2}}||\mathbf{g}^{(t)}||^2
\end{equation}
Since $F_{D}^{(t)}$ is positive definite, hence Eq. (\ref{eq:upper_bound1}) holds true. For the bound on convergence rate, we refer to \cite{Ryu2016} for the details of the complete proof.
\end{proof}
\begin{proposition}[Convergence in non-convex stochastic optimization] Under the assumptions:\\
(i) $f$ is lower bounded and differentiable; $||\nabla J(\boldsymbol{\theta}) - \nabla J(\boldsymbol{\theta}')||_2\leq L ||\boldsymbol{\theta} - \boldsymbol{\theta}'||_2$, $||\tilde{F}_{D}^{(t)}||_\infty<L,\, \forall t, \boldsymbol{\theta}, \boldsymbol{\theta}'$.\\
(ii) Both the true and stochastic gradient are bounded, i.e. $||\nabla J(\boldsymbol{\theta}^{(t)})||_2 \leq \lambda$ and $||g_t||_2 \leq \lambda$, $\forall t$ for some $\lambda>0$.\\
(iii) Unbiased and independent noise in $\mathbf{g}^{(t)}$, i.e. $\mathbf{g}^{(t)} = \nabla J(\boldsymbol{\theta}^{(t)}) + \zeta^{(t)}$, $\mathbb{E}[\zeta^{(t)}] = 0$, and $\zeta^{(t)}\perp\zeta^{(t)}$, $\forall i \neq j$.

Assume $\eta^{(t)} = \frac{\eta}{\sqrt{t}}$, $\beta^{(t)}\leq \beta\leq 1$ is non-increasing, $\frac{\tilde{F}_{D}^{(t-1)}[j]}{\eta^{(t-1)}}\leq \frac{\tilde{F}_{D}^{(t)}[j]}{\eta^{(t)}}$, $\forall t\in [T], j\in [d]$, we then have 
\begin{equation}\label{eq:gradient_bound}
\min_{t\in [T]}\mathbb{E}[||\nabla J(\boldsymbol{\theta}^{(t)})||_2^2] \leq \frac{L}{\sqrt{T}}(C_1 \eta^2 \lambda^2(1+ \log T) + C_2 d\eta + C_3 d\eta^2 + C_4)
\end{equation}
where $C_1, C_2, C_3$ are constants independent of $d$ and $T$, $C_4$ is a constant independent of $T$, the expectation is taken w.r.t all the randomness corresponding to $\{g^{(t)}\}$.
\end{proposition}
\begin{proof}
Follow \cite{chen2018on}, as AdaFisher is an Adam-type method with the condition $||\eta^{(t)} m^{(t)}/\tilde{F}_{D}^{(t)}||_2 \leq G$ for some $G$ (which can be obtained by $\eta^{(t)}< \eta$, $||\mathbf{g}^{(t)}||_2\leq \lambda$ and $||\tilde{F}_{D}^{(t)}||_2\geq 1$), we have 
\begin{align}\label{eq:upper_bound2}
\mathbb{E}\Bigg[\sum_{t=1}^{T} \eta^{(t)}\langle\nabla J(\boldsymbol{\theta}^{(t)}), \nabla J(\boldsymbol{\theta}^{(t)})/\tilde{F}_{D}^{(t)}\rangle\Bigg] \leq &\mathbb{E}\Bigg[C_1\sum_{t=1}^T \left\Vert\frac{\eta^{(t)} \mathbf{g}^{(t)}}{\tilde{F}_{D}^{(t)}}\right\Vert_2^2 + C_2\sum_{t=1}^T\left\Vert \frac{\eta^{(t)}}{\tilde{F}_{D}^{(t)}} - \frac{\eta^{(t-1)}}{\tilde{F}_{D}^{(t-1)}}\right\Vert_1 \notag\\
& + C_3\sum_{t=1}^T\left\Vert \frac{\eta^{(t)}}{\tilde{F}_{D}^{(t)}} - \frac{\eta^{(t)}}{\tilde{F}_{D}^{(t)}}\right\Vert_2^2\Bigg] + C_4.
\end{align}
We first bound non-constant terms in RHS of Eq. (\ref{eq:upper_bound2}). For the term with $C_1$, since $||\tilde{F}_{D}^{(t)}||_2\geq 1$, we have
\begin{align*}
\mathbb{E}\Bigg[\sum_{t=1}^T \left\Vert\frac{\eta^{(t)} \mathbf{g}^{(t)}}{\tilde{F}_{D}^{(t)}}\right\Vert_2^2\Bigg] &\leq \mathbb{E}\Bigg[\sum_{t=1}^T ||\eta^{(t)} \mathbf{g}^{(t)}||_2^2\Bigg]\\
&= \mathbb{E}\Bigg[\sum_{t=1}^T \left\Vert\frac{\eta}{\sqrt{t}}\mathbf{g}^{(t)}\right\Vert_2^2\Bigg]\\
&\leq \eta^2\lambda^2\sum_{t=1}^T \frac{1}{t}\leq \eta^2\lambda^2 (1 + \log T).
\end{align*}
For the term with $C_2$, we have
\begin{align*}
\mathbb{E}\Bigg[\sum_{t=1}^T\left\Vert \frac{\eta^{(t)}}{\tilde{F}_{D}^{(t)}} - \frac{\eta^{(t-1)}}{\tilde{F}_{D}^{(t-1)}}\right\Vert_1 \Bigg] &=  \mathbb{E}\Bigg[\sum_{j=1}^d \sum_{t=2}^T \left( \frac{\eta^{(t-1)}}{\tilde{F}_{D}^{(t-1)}[j]} - \frac{\eta^{(t)}}{\tilde{F}_{D}^{(t)}[j]}\right) \Bigg]\\
&= \mathbb{E}\Bigg[\sum_{j=1}^d \frac{\eta^{(1)}}{\tilde{F}_{D}^{(1)}[j]} - \frac{\eta^{(T)}}{\tilde{F}_{D}^{(T)}[j]}\Bigg]\\
& \leq \mathbb{E}\Bigg[\sum_{j=1}^d \frac{\eta^{(1)}}{\tilde{F}_{D}^{(1)}[j]}\Bigg] \leq d\eta
\end{align*}
where the first equality is due to  $\frac{\tilde{F}_{D}^{(t-1)}[j]}{\eta^{(t-1)}}\leq \frac{\tilde{F}_{D}^{(t)}[j]}{\eta^{(t)}}$, $\forall t\in [T], j\in [d]$.

For the term with $C_3$, we have
\begin{align*}
\mathbb{E}\Bigg[\sum_{t=1}^T\left\Vert \frac{\eta^{(t)}}{\tilde{F}_{D}^{(t)}} - \frac{\eta^{(t-1)}}{\tilde{F}_{D}^{(t-1)}}\right\Vert_2^2 \Bigg] &= \mathbb{E}\Bigg[\sum_{t=1}^T\sum_{j=1}^d\left( \frac{\eta^{(t)}}{\tilde{F}_{D}^{(t)}[j]} - \frac{\eta^{(t-1)}}{\tilde{F}_{D}^{(t)}[j]}\right)^2 \Bigg]\\
&= \mathbb{E}\Bigg[\sum_{t=1}^T\sum_{j=1}^d\left| \frac{\eta^{(t)}}{\tilde{F}_{D}^{(t)}[j]} - \frac{\eta^{(t-1)}}{\tilde{F}_{D^{(t-1)}}[j]}\right| \cdot \left| \frac{\eta^{(t)}}{\tilde{F}_{D}^{(t)}[j]} - \frac{\eta^{(t-1)}}{\tilde{F}_{D}^{(t-1)}[j]}\right| \Bigg]\\
&\leq \mathbb{E}\Bigg[\sum_{t=1}^T\sum_{j=1}^d\left| \frac{\eta^{(t)}}{\tilde{F}_{D}^{(t)}[j]} - \frac{\eta^{(t-1)}}{\tilde{F}_{D}^{(t-1)}[j]}\right| \cdot \left| \frac{\eta}{\sqrt{t}\tilde{F}_{D}^{(t)}[j]} - \frac{\eta}{\sqrt{t-1}\tilde{F}_{D}^{(t-1)}[j]}\right|\Bigg] \\
&\leq \mathbb{E}\Bigg[\eta\sum_{t=1}^T\sum_{j=1}^d\left| \frac{\eta_t}{\tilde{F}_{D}^{(t)}[j]} - \frac{\eta^{(t-1)}}{\tilde{F}_{D}^{(t-1)}[j]}\right|\Bigg]\\
&= \eta\mathbb{E}\Bigg[\sum_{t=1}^T\left\Vert \frac{\eta^{(t)}}{\tilde{F}_{D}^{(t)}} - \frac{\eta^{(t-1)}}{\tilde{F}_{D}^{(t-1)}}\right\Vert_1 \Bigg]\\
&\leq d\eta^2
\end{align*}
Hence 
\begin{align}\label{eq:bound_1}
\mathbb{E}\Bigg[C_1\sum_{t=1}^T \left\Vert\frac{\eta^{(t)} g^{(t)}}{\tilde{F}_{D}^{(t)}}\right\Vert_2^2 &+ C_2\sum_{t=1}^T\left\Vert \frac{\eta^{(t)}}{\tilde{F}_{D}^{(t)}} - \frac{\eta^{(t-1)}}{\tilde{F}_{D}^{(t-1)}}\right\Vert_1 \notag + C_3\sum_{t=1}^T\left\Vert \frac{\eta^{(t)}}{\tilde{F}_{D}^{(t)}} - \frac{\eta^{(t-1)}}{\tilde{F}_{D^{(t-1)}}}\right\Vert_2^2\Bigg] + C_4\\
&\leq C_1 \eta^2 \lambda^2(1+ \log T) + C_2 d\eta + C_3 d\eta^2 + C_4
\end{align}
Now we lower bound the LHS of Eq. (\ref{eq:gradient_bound}). With the assumption $||\tilde{F}_{D}^{(t)}||_\infty\leq L$, we have 
$$(\eta^{(t)}/\tilde{F}_{D}^{(t)})_j \geq \frac{\eta}{L\sqrt{t}}.$$
Thus 
\begin{equation}\label{eq:bound_2}
\mathbb{E}\Bigg[\sum_{t=1}^{T} \eta^{(t)}\langle\nabla J(\boldsymbol{\theta}^{(t)}), \nabla J(\boldsymbol{\theta}^{(t)})/\tilde{F}_{D}^{(t)}\rangle\Bigg]\geq \mathbb{E}\Bigg[\sum_{t=1}^{T}\frac{\eta}{L\sqrt{t}}||\nabla J(\boldsymbol{\theta}^{(t)})||_2^2\Bigg] \geq \frac{\sqrt{T}}{L}\min_{t\in [T]}\mathbb{E}[||\nabla J(\boldsymbol{\theta}^{(t)})||_2^2]
\end{equation}
Combining Eq. (\ref{eq:bound_1}) and (\ref{eq:bound_2}) gives the desired result.
\end{proof}
\chapter{Visualization}\label{chap:Appendixvisualization}
The convergence rate of an optimizer is crucial, serving as an indicator of its robustness against saddle points and its ability to generalize effectively. In this section, we introduce a novel methodology for visualizing the convergence behavior of optimizers through a statistical model, as depicted in Figure~\ref{fig:heatloss}. Initially, our process employs Principal Component Analysis (PCA) for dimensionality reduction, reducing the dataset dimensions from $\mathcal{D} \in \mathbb{R}^{m \times n}$ to $\hat{\mathcal{D}} \in \mathbb{R}^{m \times 2}$, following the protocol established in \cite{doi:10.1080/14786440109462720}. We then apply this reduced dataset to a toy model composed of an $L$-layer multi-layer perceptron (MLP). Notably, we focus on the first weight matrix $W_{1}^{e}$ of this MLP, which resides in $\mathbb{R}^2$, where $e$ denotes the epoch number. For consistency and to ensure comparability, all layers’ weights are initialized identically across different optimizers.
\begin{figure}[!h]
    \centering
    \includegraphics[width=\textwidth]{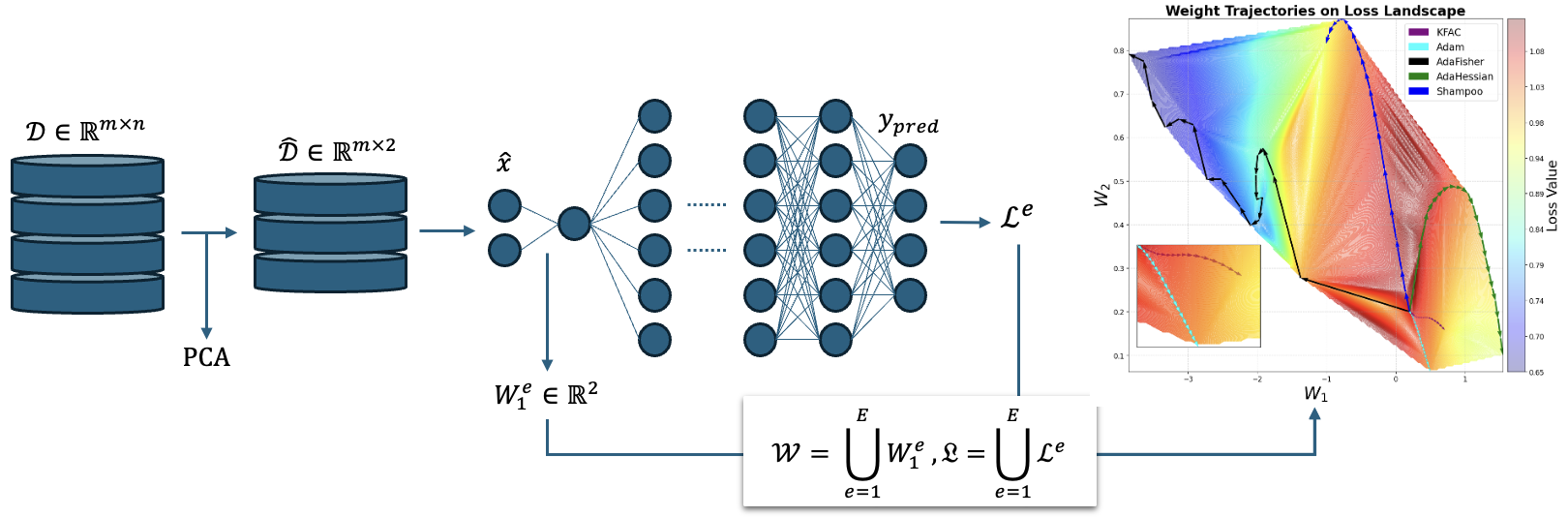}
    \caption{Pipeline for visualization of optimization paths for various algorithms on a loss surface, comparing their convergence efficiency.}
    \label{fig:visualization}
\end{figure}
Following the training phase with various optimizers where we denote a set of optimizer results \(\mathcal{O}\), we analyze both the collection of first-layer weights, $\mathcal{W}$, and the evolution of the loss function, $\mathfrak{L}$ defined as: 
\[
\mathcal{W} = \begin{bmatrix}
(W_1^1)^\top \\
(W_1^2)^\top \\
\vdots \\
(W_1^E)^\top 
\end{bmatrix}, \quad \mathfrak{L} = [\mathcal{L}_1^1, \mathcal{L}_1^2, \ldots, \mathcal{L}_1^E]^\top
\]
where \((W_1^e)^\top \) represents the weight vector at the \(e\)th epoch, and \(\mathcal{L}_1^e\) represents the loss at the \(e\)th epoch, extracted from the optimization results \(\mathcal{O}\). We construct a grid \((\mathbf{X}, \mathbf{Y})\) spanning the range of weight parameters, discretized into \(200\) linearly spaced points along each axis:
\[
\mathbf{X}, \mathbf{Y} = \text{meshgrid}\left( \min(\mathcal{W}_{:,1}), \max(\mathcal{W}_{:,1}), \min(\mathcal{W}_{:,2}), \max(\mathcal{W}_{:,2}), 200 \right)
\]

Finally, we interpolate the loss values \(\mathcal{L}\) over the grid using cubic interpolation to obtain a smooth loss surface \(\mathbf{Z}\):
\[
\mathbf{Z} = \text{griddata}(\mathcal{W}, \mathcal{L}, (\mathbf{X}, \mathbf{Y}), \text{method}=\texttt{cubic})
\]
These elements are integral to the visualization process, which elucidates the optimizer’s trajectory through the parameter space across training epochs. It is important to note that while we focus on the first layer's weight matrix for clarity, the methodology can be adapted to visualize the weights of any layer within the network. Figure~\ref{fig:visualization} summarizes the pipeline.

In the experiment depicted in Figure~\ref{fig:heatloss}, we selected the IRIS dataset~\citep{rz7n-kj20-18}, owing to its widespread recognition and compatibility with PCA application. Our model employs a 2-layer MLP architecture. We specifically attend to the weight matrix of the first layer, denoted by $W_{1} \in \mathbb{R}^2$. This particular focus is informed by the empirical observation that the parameters of the first layer tend to exhibit a faster convergence rate compared to those of subsequent layers in the network. Such a phenomenon can be attributed to the more direct influence of the input features on the first layer's weights, which often results in a more pronounced and expedited learning dynamic. Given the classification nature of the task, we employed the Cross-Entropy loss function~\citep{zhang2018generalized}. The network was trained over 20 epochs using a suite of optimizers: Adam, AdaHessian, K-FAC, Shampoo, and AdaFisher. We standardized the learning rate across all optimizers at $1 \times 10^{-3}$ to ensure comparability of results. Examination of Figure~\ref{fig:heatloss} reveals that AdaFisher's convergence is markedly superior to that of its counterparts, achieving rapid convergence to the local minimum of the loss landscape concerning the first weight parameter within a few iterations. Conversely, the alternative optimizers demonstrate convergence to less optimal local minima.  Note that while the results may vary due to the stochastic nature of parameter initialization, AdaFisher typically converges to a better local minimum compared to its counterparts.  
\chapter{Results} \label{chap:results}
\begin{figure}[!h]
    \centering
    \includegraphics[width=0.8\textwidth]{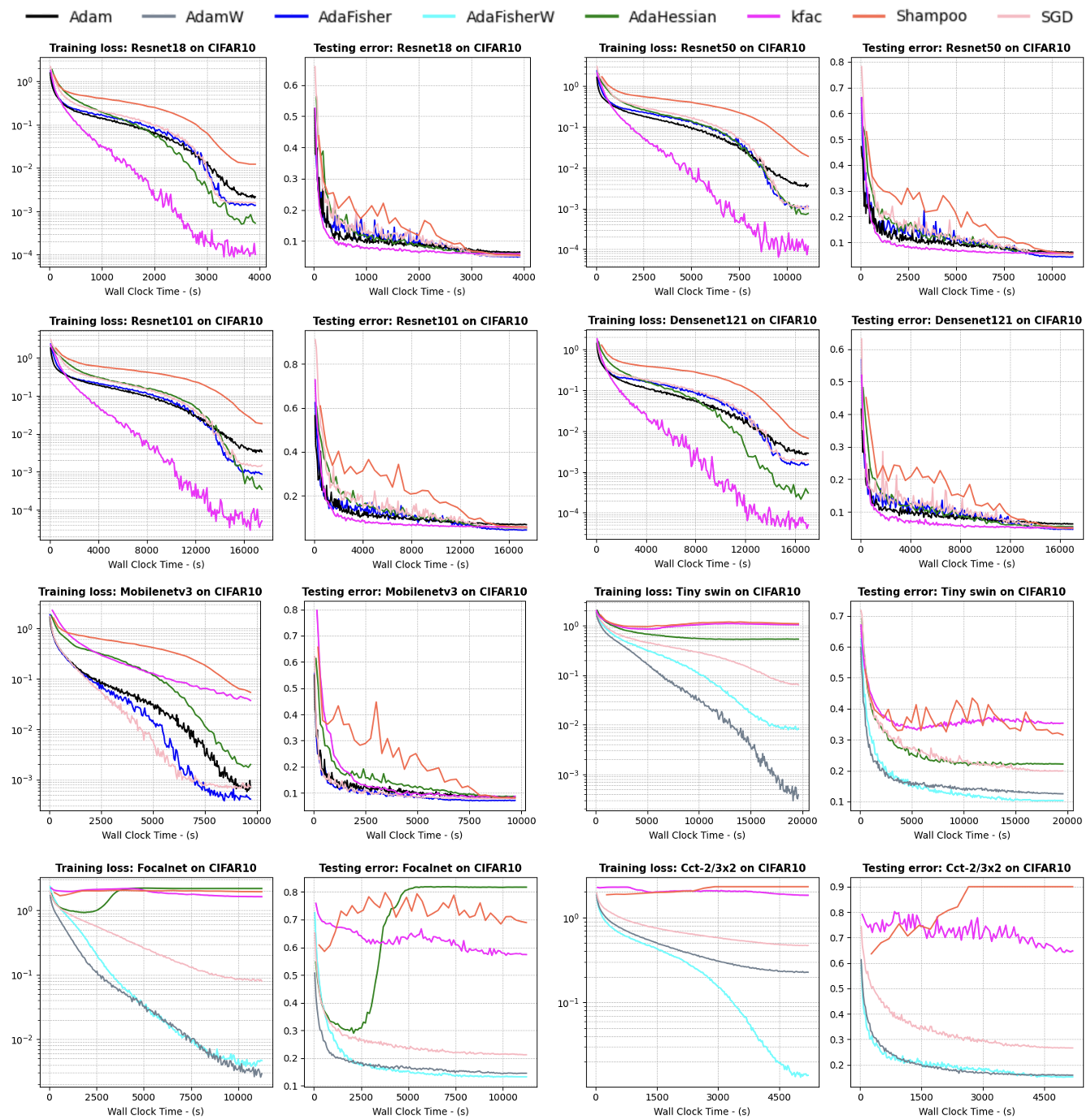}
    \caption{WCT training loss, test error, for CNNs and ViTs on CIFAR10 experiments, without Cutout. A batch size of 256 was used, and all networks were tuned using ResNet18 applied on CIFAR10. The final accuracy results are reported in Table~\ref{table_cutout_vs_nocutout} (a).}
    \label{fig:results_cifar10_no_cutout}
\end{figure}

\begin{figure}[!h]
    \centering
    \includegraphics[width=\textwidth]{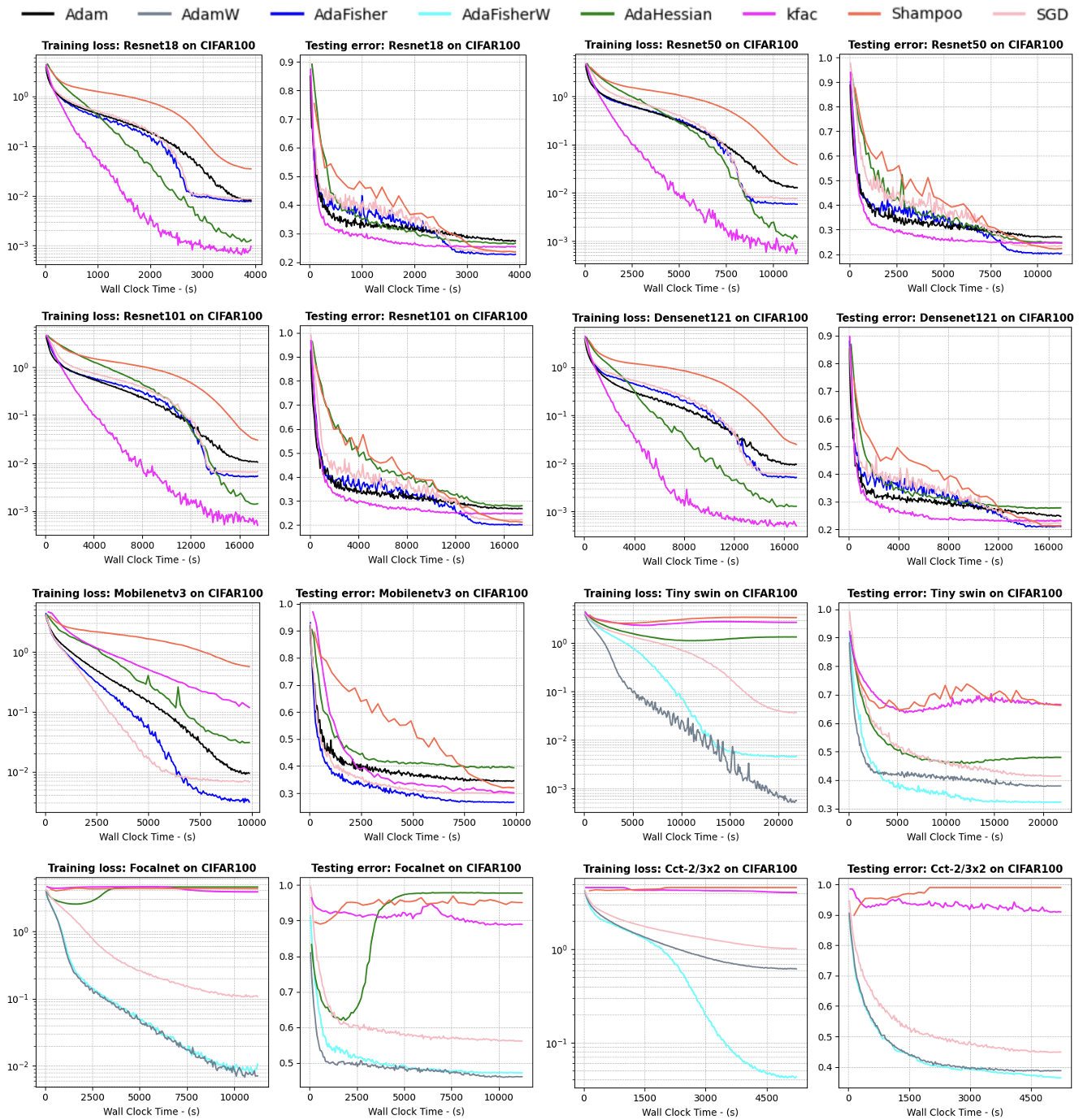}
    \caption{WCT training loss, test error, for CNNs and ViTs on CIFAR100 experiments, without Cutout. A batch size of 256 was used, and all networks were tuned using ResNet18 applied on CIFAR10. The final accuracy results are reported in Table~\ref{table_cutout_vs_nocutout} (a).}
    \label{fig:results_cifar100_no_cutout}
\end{figure}

\begin{figure}[!h]
    \centering
    \includegraphics[width=\textwidth]{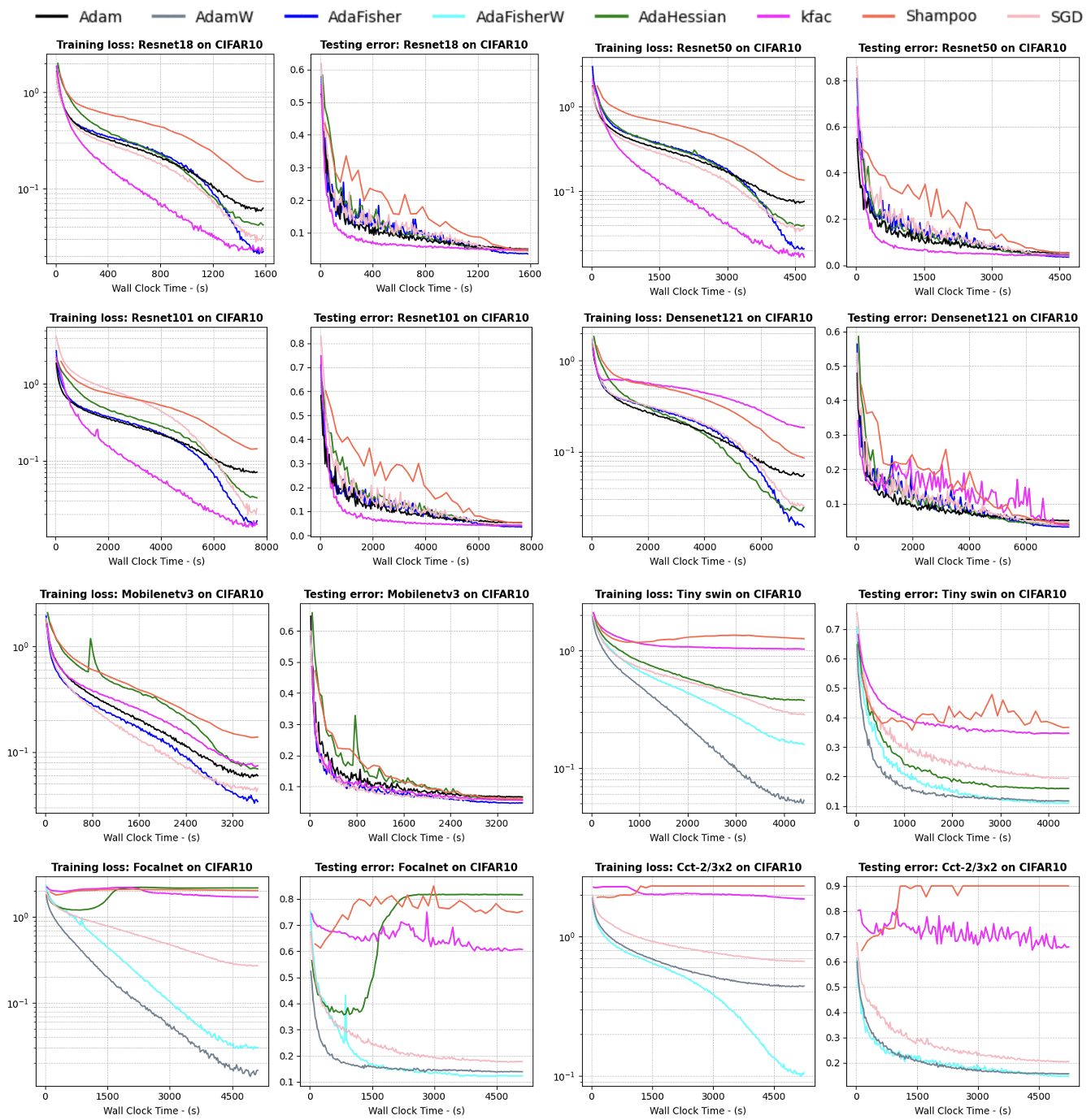}
    \caption{WCT training loss, test error, for CNNs and ViTs on CIFAR10 experiments, with Cutout. A batch size of 256 was used, and all networks were tuned using ResNet18 applied on CIFAR10. The final accuracy results are reported in Table~\ref{table_cutout_vs_nocutout} (b).}
    \label{fig:results_cifar10_cutout}
\end{figure}

\begin{figure}[!h]
    \centering
    \includegraphics[width=\textwidth]{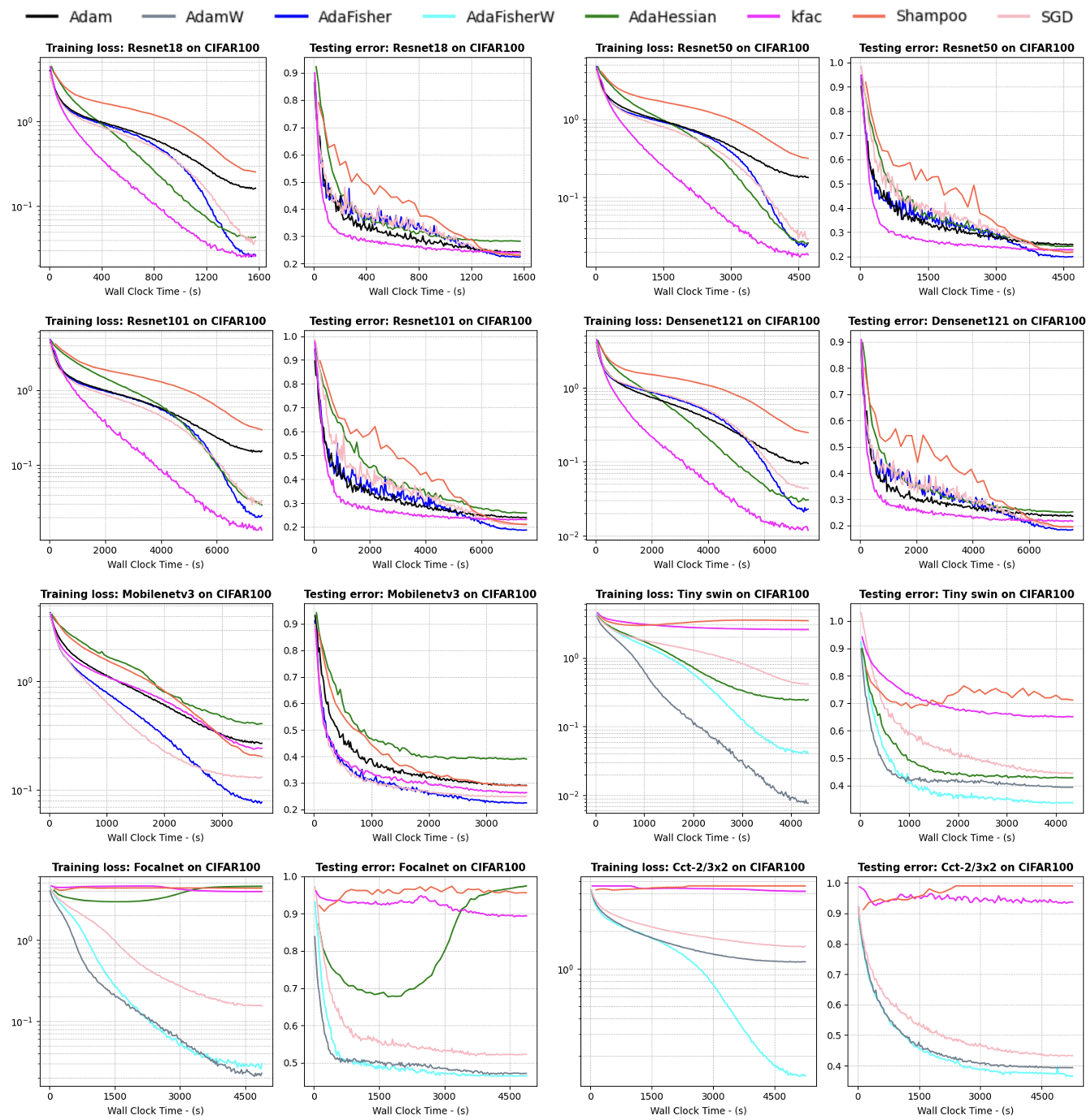}
    \caption{WCT training loss, test error, for CNNs and ViTs on CIFAR100 experiments, with Cutout. A batch size of 256 was used, and all networks were tuned using ResNet18 applied on CIFAR10. The final accuracy results are reported in Table~\ref{table_cutout_vs_nocutout} (b).}
    \label{fig:results_cifar100_cutout}
\end{figure}
\clearpage
\begin{center}
\begin{figure}[!h]
    \centering
    \includegraphics[width=\textwidth]{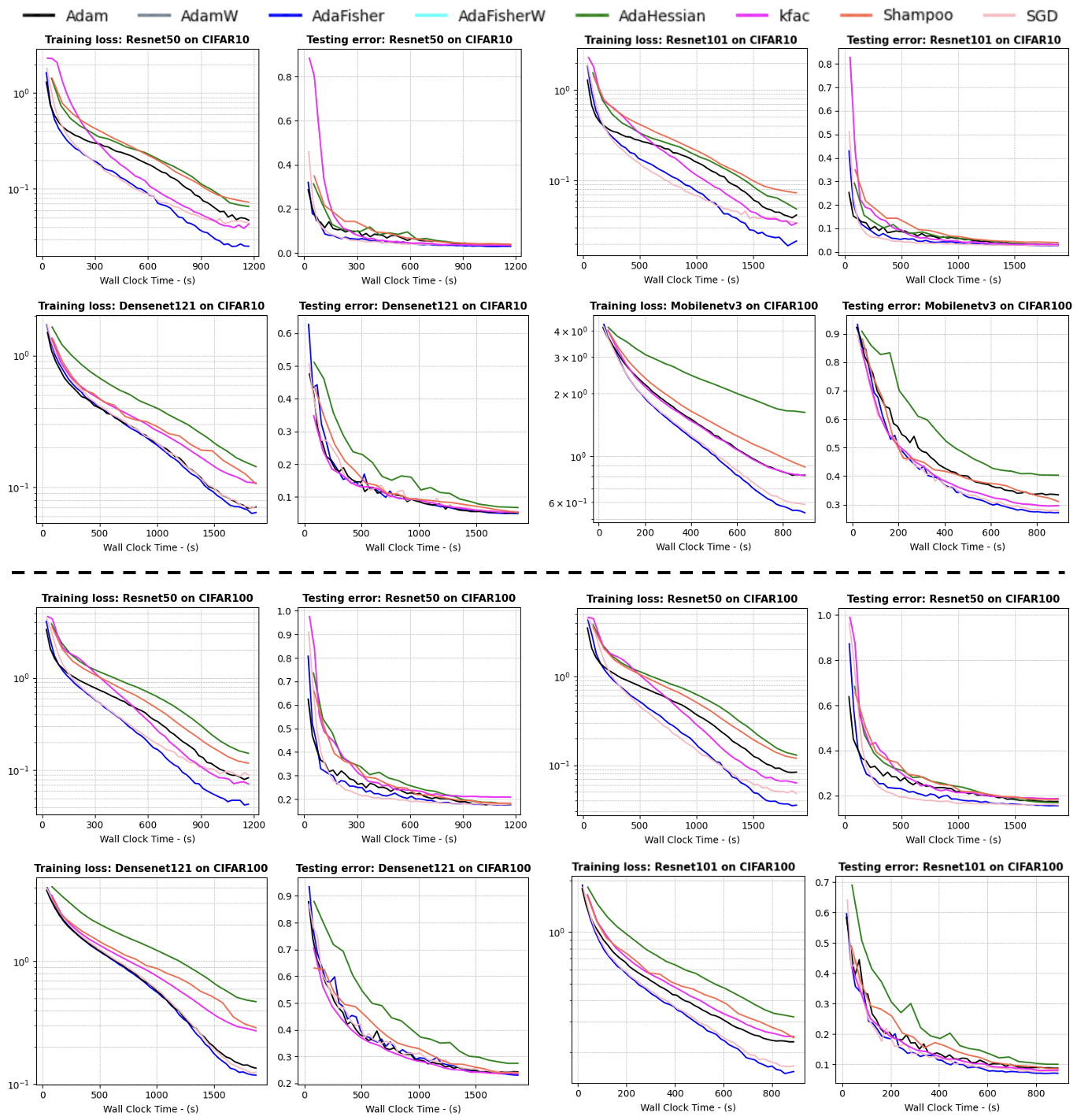}
    \caption{WCT training loss and test error for CNNs on CIFAR-10/100 experiments with pretrained weights on ImageNet-1k. A batch size of 256 was used, and all networks were tuned using ResNet50 applied on CIFAR10. The final accuracy results are reported in Table~\ref{pretrained_cifar}.}
    \label{fig:results_pretrained}
\end{figure}
\end{center}

\end{appendices}

\addcontentsline{toc}{chapter}{Bibliography}  
\bibliography{bibliography}        
\bibliographystyle{apacite}      

\end{document}